\documentclass{article} 
\usepackage{iclr2026_conference,times}

\usepackage{url}            
\usepackage[utf8]{inputenc} 
\usepackage[T1]{fontenc}    
\usepackage{booktabs}       
\usepackage{amsfonts}       
\usepackage{nicefrac}       
\usepackage{microtype}      
\usepackage[dvipsnames]{xcolor}         

\usepackage{colortbl}

\usepackage{titletoc}
\usepackage{amsmath}
\usepackage{amssymb}
\usepackage{mathtools}
\usepackage{amsthm}
\usepackage{algorithm}
\usepackage{algpseudocode}
\usepackage{setspace}
\usepackage{multirow}
\usepackage{makecell}
\usepackage{pifont}
\usepackage{comment}
\theoremstyle{plain}
\newtheorem{theorem}{Theorem}
\newtheorem{lemma}{Lemma}

\usepackage{hyperref}       
\usepackage[capitalise]{cleveref}

\usepackage{arydshln}

\usepackage{caption}
\usepackage[normalem]{ulem}
\usepackage{subcaption}

\usepackage{rotating}

\newcommand{\pend}[1]{#1}

\newcommand{\revised}[1]{#1}

\newcommand{\xmark}{\ding{55}}%

\def\ninput{{n_\text{in}}}
\def\nout{{n_\text{out}}}
\def\noutput{{n_\text{out}}}
\def\nprewnd{{n_{\text{slash}}}}
\def\cmax{{C_{\text{max}}}}
\def\nglobal{{n_{\text{vert}}}}
\def\nlookahead{{n_{\text{lookahead}}}}
\def\ncompwnd{{n_{\text{window}}}}
\def\poolk{k}
\def\nneigh{{n_{\text{neighbor}}}}
\def\lskip{{l_{\text{skip}}}}

\def\llamaoneb{Llama-3.2-1B}
\def\llamaeightb{Llama-3.1-8B}
\def\vllm{vLLM}

\def\speckv{SpecKV}
\def\adaspeckv{Ada-SpecKV}
\def\speckvlong{Speculative KV Dropping}
\def\specpc{SpecPC}
\def\specpclong{Speculative Prompt Compression}
\def\framework{Draft-based Approximate Inference}

\title{Draft-based Approximate Inference for LLMs}

\newcommand*\samethanks[1][\value{footnote}]{\footnotemark[#1]}

\author{%
  \textbf{Kevin Galim}$^{1}$\thanks{Equal contribution.} \quad
  \textbf{Ethan Ewer}$^{2}$\samethanks \quad
  \textbf{Wonjun Kang}$^{1,3}$ \quad
  \textbf{Minjae Lee}$^{1}$ \\
  \textbf{Hyung Il Koo}$^{1,4}$ \quad
  \textbf{Kangwook Lee}$^{2,5}$ \\[2ex]
  $^1$FuriosaAI \quad
  $^2$UW-Madison \quad
  $^3$Seoul National University \\
  $^4$Ajou University \quad
  $^5$KRAFTON \\[1ex]
  \texttt{\{kevin.galim, kangwj1995, minjae.lee, hikoo\}@furiosa.ai} \\
  \texttt{\{eewer, kangwook.lee\}@wisc.edu} \\[2ex]
  Code: \url{https://github.com/furiosa-ai/draft-based-approx-llm}
}

\iclrfinalcopy 
\begin{document}

\maketitle

\begin{abstract}

Optimizing inference for long-context large language models (LLMs) is increasingly important due to the quadratic compute and linear memory cost of Transformers. Existing approximate inference methods, including key-value (KV) cache dropping, sparse attention, and prompt compression, typically rely on coarse predictions of token or KV pair importance. We unify and extend recent work by introducing a framework for approximate LLM inference that leverages small draft models to more accurately predict token and KV pair importance. We provide novel theoretical and empirical analyses justifying lookahead-based importance estimation techniques. Within this framework, we present: (i) \textbf{\speckv{}}, the first method to use lookahead with a small draft model to enable precise KV cache dropping; (ii) \textbf{\specpc{}}, which leverages draft model attention activations to identify and discard less important prompt tokens; \revised{and (iii) \textbf{SpecKV-PC}, a cascaded compression strategy combining both techniques}. Extensive experiments on long-context benchmarks demonstrate that our methods consistently achieve higher accuracy than existing baselines while retaining the same efficiency gains in memory usage, latency, and throughput.

\end{abstract}

\section{Introduction}

The demand for longer context lengths in large language models (LLMs)~\citep{gpt4,gemini25pro} continues to grow~\citep{liu2024lost}, driven by applications such as dialogue systems~\citep{gpt4,gemini25pro}, document summarization~\citep{liu2020learning}, and code completion~\citep{code1}. Modern models like GPT-4 \citep{gpt4} and Gemini 2.5 Pro \citep{gemini25pro} have pushed context windows to over a million tokens. However, scaling Transformers~\citep{transformer} to these lengths remains difficult due to significant computational and memory constraints. Attention computation scales quadratically with context length, increasing inference latency, while key-value (KV) cache memory grows linearly, straining GPU resources. For example, caching the KV states for 128K tokens in Llama-3.1-8B \citep{grattafiori2024llama} can consume over 16GB of memory, limiting the practical scalability of LLMs.

To address scalability challenges, recent work introduces approximate LLM inference techniques that reduce latency and memory usage at inference time. Techniques include sparse attention for prefilling~\citep{minference} and decoding~\citep{quest}, which speed up inference by having each query attend to only a subset of keys. Sparse prefilling shortens time to the first token, while sparse decoding boosts generation throughput. KV cache dropping~\citep{pyramidkv,snapkv,streamingllm,h2o} reduces memory and increases throughput by shrinking the cache after prefilling or during decoding. Prompt compression~\citep{choi2024r2c,llmlingua,liskavets2024cpc} further improves efficiency by removing less important tokens before inputting the prompt, reducing both attention and MLP computation, as well as decreasing KV cache size. 

Orthogonally, speculative decoding~\citep{medusa,specdec2,hu2025speculativedecodingbeyondindepth,specdecode} accelerates LLM inference by using a small draft model to propose a sequence of multiple tokens, which the target model verifies in parallel. This improves throughput without altering the output distribution and is particularly effective for autoregressive models, where sequential generation is a bottleneck. However, unlike approximate inference, speculative decoding does not lower the total memory or computation requirements and struggles with increasing context length.

\pend{In contrast, approximate LLM inference improves efficiency by reducing the amount of computation the model performs. This is often done by estimating the importance of each token or KV pair for future generation and discarding less important ones from attention or feedforward computations. Existing methods~\citep{pyramidkv,adakv,minference,snapkv} use attention activations from input tokens to predict which tokens or KV pairs future tokens will attend to, as future tokens are not yet available. However, input attention activations alone do not reliably identify the tokens or KV pairs most relevant for future token generation. Recent methods~\citep{specprefill,laq} address this by leveraging an approximate output to improve importance estimates. In this work, we unify and extend these techniques by introducing \framework{}, a lookahead-based framework that uses an inexpensive draft model to approximate future outputs with minimal overhead (\cref{fig:framework}). Our main contributions are as follows:}

\begin{enumerate}
    \item We present \framework{}, a framework using draft model lookahead for enhanced approximate inference.
    \item We present theoretical and empirical analyses justifying lookahead-based KV cache dropping and the use of draft model token importance to approximate target model token importance. 
    \item Within the \framework{} framework, we develop two concrete algorithms targeting three LLM inference optimizations: \speckvlong{} (\speckv{}) for KV cache dropping with sparse prefill, and \specpclong{} (\specpc{}) for prompt compression. Notably, \speckv{} is the first to use draft model lookahead for KV cache optimization. \revised{Additionally, we introduce SpecKV-PC, a cascaded pipeline integrating both algorithms to achieve superior accuracy, latency, and memory efficiency.}
    \item We perform comprehensive experiments on long-context benchmarks, demonstrating that our methods attain state-of-the-art accuracy under fixed KV cache or prompt size constraints. Our results consistently outperform prior baselines, underscoring the potential of draft models for fast and accurate approximate inference in large language models.
\end{enumerate}


     

\begin{figure}
     \centering

     \begin{subfigure}[b]{0.55\textwidth}
         \centering
         \includegraphics[width=\textwidth]{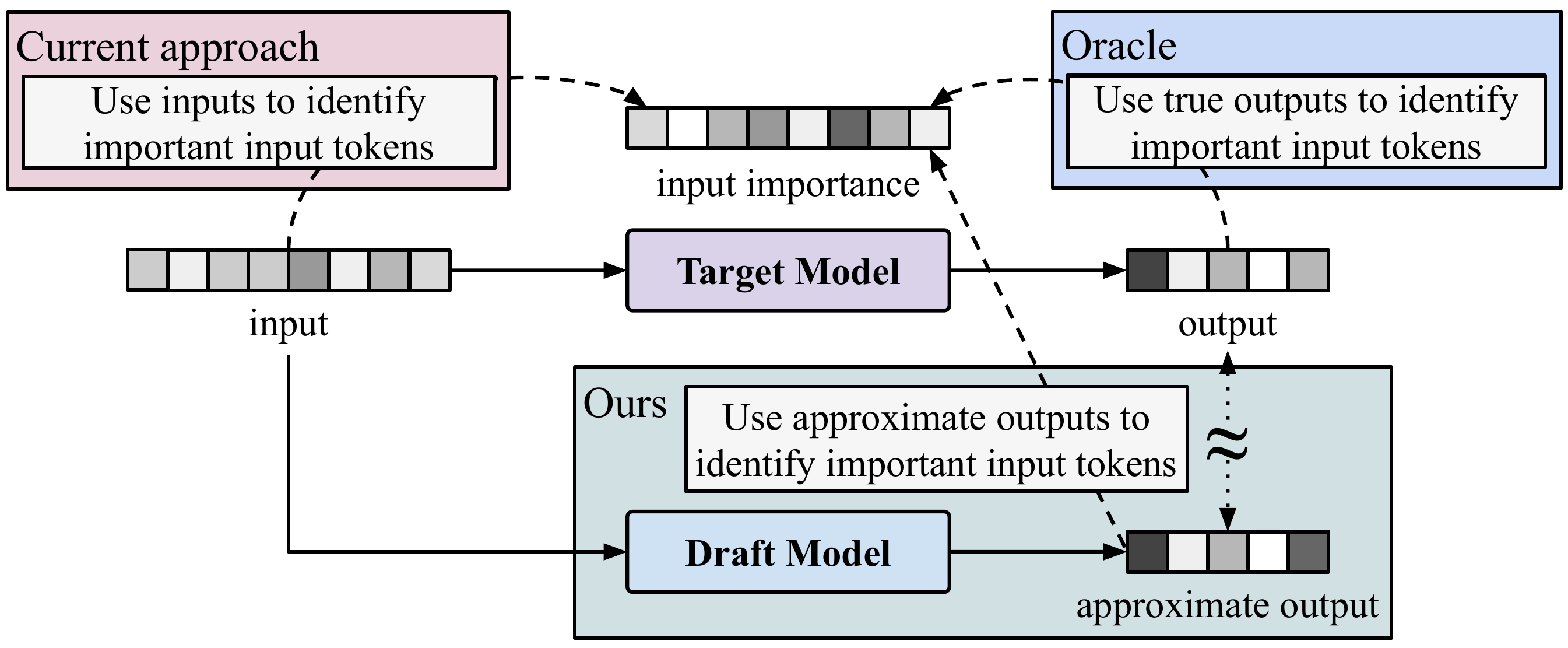}
         \caption{Proposed Framework.}\label{fig:framework}
     \end{subfigure}
     \hfill
     \begin{subfigure}[b]{0.43\textwidth}
         \centering
         \includegraphics[width=\textwidth]{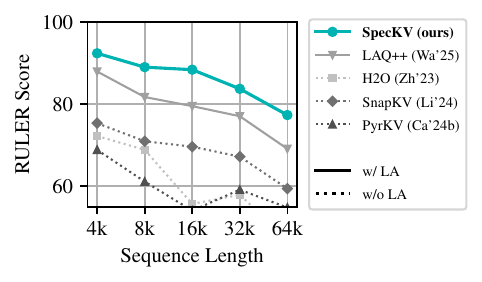}
         \caption{Impact of lookahead (LA) on accuracy.}
     \end{subfigure}

\caption{\textbf{(a)} Overview of our \framework{} framework for input token importance estimation in comparison with current approaches and the oracle approach. Prior methods use input tokens to estimate input token importance. Our approach incorporates draft model predictions of future output tokens, yielding more accurate importance estimates. This better aligns with the hypothetical oracle setting, where the true output is known and influential tokens can be precisely identified. \textbf{(b)} On RULER with Llama-3-70B~\citep{grattafiori2024llama}, lookahead-based methods (LAQ++~\citep{laq}, SpecKV) significantly outperform non-lookahead approaches (H2O~\citep{h2o}, SnapKV~\citep{snapkv}, PyramidKV~\citep{pyramidkv}), with our proposed SpecKV achieving the best overall downstream score.}
\end{figure}

\section{Related Work}

\colorlet{mygold}{yellow!30}
\newcommand{\best}{\cellcolor{mygold}}

\begin{table}[h!]
\centering
\caption{Summary of prior work. Complexity reported without auxiliary/draft model. $\ninput$ and $\noutput$ denote the number of input and output tokens, respectively. $\cmax$ represents the maximum KV cache or prompt capacity. $s_\text{prefill}$ and $s_\text{decode}$ indicate the number of keys each query attends to during the prefill and decoding phases, respectively. \pend{$L$ is the number of attention layers. Since all time complexities are linear with respect to $L$, we only include $L$ for space complexity.}
Highlighted cells indicate improved complexity.}
\label{tab:related_summary}
\renewcommand{\arraystretch}{1.2}
\resizebox{\linewidth}{!}{
\begin{tabular}{llcccccc}
\toprule
\textbf{Type} & \textbf{Method} & \textbf{Sparse Attn} & \textbf{KV Dropping} & \textbf{Prefill Time} & \textbf{Decoding Time} & \textbf{Prefill Space} & \textbf{Decoding Space} \\
\midrule
Dense & Dense & \xmark & \xmark & $O(\ninput ^2)$ & $O(\noutput(\ninput + \noutput))$ & $O(L\ninput)$ & $O(L(\ninput + \noutput))$\\
\midrule
 \multirow{3}{*}{Sparse attention}
 & MInference, FlexPrefill & Prefill & \xmark & \best $O(\ninput s_\text{prefill})$ & $O(\noutput(\ninput + \noutput))$ & $O(L\ninput)$ & $O(L(\ninput + \noutput))$\\
 \cmidrule{2-8}
 & Quest, RetrievalAttention & Decode & \xmark & $O(\ninput ^2)$ & \best $O(\noutput s_\text{decode})$ & $O(L\ninput)$ & $O(L(\ninput + \noutput))$\\
\midrule
\multirow{5}{*}{KV dropping} 
 & StreamingLLM & Prefill & Decode & \best $O(\ninput\cmax)$ & \best $O(\noutput\cmax)$ & \best $O(\max(\ninput, L\cmax))$ & \best $O(L\cmax)$\\
 & H2O & \xmark & Decode & $O(\ninput ^2)$ & \best $O(\noutput\cmax)$  & \best $O(\max(\ninput, L\cmax))$ & \best $O(L\cmax)$\\
 & SnapKV, PyramidKV, AdaKV & \xmark & After prefill & $O(\ninput ^2)$ & \best $O(\noutput(\cmax + \noutput))$ & \best $O(\max(\ninput, L\cmax))$ & \best $O(L(\cmax + \noutput))$\\
 & LAQ++ & \xmark & After prefill & $O(\ninput ^2)$ & \best $O(\noutput(\cmax + \noutput))$ & $O(L\ninput)$ & \best $O(L(\cmax + \noutput))$\\
 & \textbf{\speckv{}} (Ours) & Prefill & After prefill & \best $O(\ninput s_\text{prefill})$ & \best $O(\noutput(\cmax + \noutput))$ & \best $O(\max(\ninput, L\cmax))$ & \best $O(L(\cmax + \noutput))$\\
\midrule
Prompt compression & \makecell[l]{LLMLingua-2, CPC,\\ R2C, SpecPrefill, \textbf{\specpc{}} (Ours)}  & -- & -- & \best $O(\cmax^2)$ & \best $O(\noutput(\cmax + \noutput))$ & \best $O(L\cmax)$ & \best $O(L(\cmax + \noutput))$\\
\bottomrule
\end{tabular}}
\end{table}

\paragraph{Sparse Attention}
One way to improve inference efficiency is through sparse attention with static patterns. For example, sliding window attention~\citep{sliding_window}, used in models like Mistral 7B \citep{mistral7b}, Gemma~3~\citep{gemma_2025}, GPT-3~\citep{gpt3}, and gpt-oss~\citep{gpt-oss}, restricts each query to attend only a fixed-size window of recent keys, reducing computation and KV cache size during decoding. StreamingLLM~\citep{streamingllm} improves on sliding window by using initial tokens, called attention sinks, along with the sliding window.
MInference~\citep{minference}, adopted by Qwen2.5-1M \citep{qwen2.5-1m}, further boosts prefill efficiency by searching offline for adaptive sparse attention patterns (A-shape, Vertical-Slash, and Block-Sparse) assigned per head. FlexPrefill~\citep{flexprefill} extends this idea by determining sparsity rates for each input prompt.
In contrast, Quest~\citep{quest} and \mbox{RetrievalAttention~\citep{retrievalattention}} target the decoding stage by only retrieving the most important KV pairs from the cache, reducing both memory bandwidth and computational demands during generation.

\paragraph{KV Cache Dropping}
KV dropping reduces computation and memory during decoding. Sliding window attention~\citep{sliding_window} and StreamingLLM~\citep{streamingllm} are examples of KV dropping methods (as well as sparse attention) as they permanently evict KV pairs from cache. H2O~\citep{h2o} improves on this by dynamically selecting attention sinks, termed heavy-hitters, using attention scores at each decoding step, while also maintaining a sliding window. SnapKV~\citep{snapkv} compresses the KV cache at the end of the prefill stage by dropping unimportant KV pairs. Subsequent work extends this idea by allocating KV cache budgets non-uniformly across layers (PyramidKV~\citep{pyramidkv}) and attention heads (AdaKV~\citep{adakv}, HeadKV~\citep{headkv}). However, these approaches drop tokens based only on current information, making them less robust to changes in token importance over time~\citep{sparse_frontier}. \pend{Recently, \cite{laq} proposed Lookahead Q-Cache (LAQ++), which addresses this by generating draft queries with a sparse approximation of the target model, using them to compute more accurate importance scores, though this comes at the cost of no reduction in peak memory usage.} 


\paragraph{Prompt Compression}
Prompt compression removes tokens before reaching the model, reducing compute and memory usage during both prefill and decoding, unlike KV dropping, which speeds up only decoding. It also surpasses sparse attention by saving both attention and MLP computation. Prompt compression works seamlessly with all types of inference setups, such as APIs or inference engines like \vllm{}~\citep{vllm}, since it does not require any modifications to the model. However, KV dropping can achieve higher compression because it selects tokens per head, while prompt compression drops the same tokens across all layers and heads.

In a question-answer setup, prompt compression may be question-agnostic (compressing context without considering the question) or question-aware (factoring in the question). Selective context~\citep{li2023compressing} and LLMLingua~\citep{llmlingua} are training-free, question-agnostic approaches using a small LLM to keep only key tokens. LongLLMLingua~\citep{jiang2024longllmlingua} adapts this for longer contexts in a question-aware manner. LLMLingua-2~\citep{pan2024llmlingua2} trains a small model~\citep{conneau2020roberta} to score token importance without using the question. CPC~\citep{liskavets2024cpc} uses a trained encoder to compute sentence importance via cosine similarity with the question, while R2C~\citep{choi2024r2c} splits the prompt into chunks, processes each with the question using a fine-tuned encoder-decoder Transformer (FiD~\citep{izacard2021fid}), and ranks them via cross-attention. \pend{Similar to our proposed \specpc{}, SpecPrefill~\citep{specprefill} leverages attention scores from a smaller draft model to identify important tokens, using the draft model for lookahead to improve token importance estimates.}



\paragraph{Speculative Decoding}
Speculative decoding \citep{medusa,specdec2,hu2025speculativedecodingbeyondindepth,specdecode} accelerates LLM inference by using a small draft model to propose multiple tokens that the target model verifies in parallel. This increases decoding throughput without changing the output distribution, addressing the bottleneck of autoregressive generation. Previous work further accelerates speculative decoding by enabling approximate inference in the draft model, using techniques such as sparse attention~\citep{magicdec}, KV cache dropping~\citep{triforce}, or KV cache quantization~\citep{quantspec}, all while preserving exact inference. In contrast, our approach leverages draft models to enable fast, approximate inference directly in the target model.

\cref{tab:related_summary} summarizes prior work, highlighting their prefill, decoding time, and memory complexities.

\section{Proposed Framework: \framework{}}

Previous LLM approximation methods~\citep{pyramidkv,adakv,minference,snapkv} estimate the  importance of current tokens on future generation by analyzing current attention patterns. While this can be effective, it provides only a rough estimate of each token’s importance. 
In contrast, if future tokens were available, we could make substantially better importance estimates by directly identifying which input tokens contribute to generating those output tokens. However, these future tokens are inaccessible before generation.

Recent work has explored using approximate future information to improve token importance estimation. LAQ++~\citep{laq} extends SnapKV~\citep{snapkv} by generating draft queries and then using them to compute more accurate importance scores. SpecPrefill~\citep{specprefill} generates lookahead tokens with a small draft model and relies on the draft model’s attention to those tokens to estimate input token importance. We bring these ideas together under a unified framework, \framework{} that leverages approximate future information to improve token importance estimation. We further extend this framework with two new algorithms: \speckv{} for KV cache dropping and sparse prefilling (\cref{sec:speckv}) and \specpc{} for prompt compression (\cref{sec:specpc}). \revised{Finally, we introduce cascaded compression with SpecKV-PC, combining both approaches into a single pipeline for superior accuracy and efficiency (\cref{sec:speckv_pc}).}

While our methods use draft models, a technique also common in speculative decoding, our objective is fundamentally different. Speculative decoding improves hardware utilization by having a draft model propose tokens that the target model verifies, accelerating generation without changing the output distribution. However, it does not reduce total computation or memory usage. In contrast, our framework reduces computation and memory costs by approximating the target model.

\subsection{Justification for Lookahead-based KV Cache Dropping}

KV cache dropping requires estimating the importance of each input KV pair. We define importance as the average attention activation from output queries to each input key. Specifically, the vector of importance scores and its approximation are given by

\begin{equation}\label{eqn:importance}
s^T = \tfrac{1}{\noutput} \sum_{i=1}^\noutput \operatorname{Softmax}\left( \tfrac{x_i^{(o)T} W_q W_k^T X^T}{\sqrt d} \right)
\quad \text{and} \quad
\hat{s}^T = \tfrac{1}{\noutput}\sum_{i=1}^\noutput\operatorname{Softmax}\left( \tfrac{\hat{x}_i^{(o)T} W_q W_k^T X^T}{\sqrt d}\right),
\end{equation}

where $X = [x_1, \dots, x_\ninput]^T \in \mathbb{R}^{\ninput \times d}$ is the matrix of input embeddings, $x_i^{(o)} \in \mathbb{R}^d$ is the input embedding from the $i$th output token, and $\hat{x}_i^{(o)} \in \mathbb{R}^d$ is its approximation. $s_i$ and $\hat s_i$ denote the importance of the $i$th KV pair.

To understand when lookahead-based KV cache dropping algorithms provide reliable importance estimates, we analyze how error in approximate input embeddings impacts error in importance score estimates for a single attention head.

\begin{theorem}\label{theorem:speckv}
If  $\|x_i^{(o)} - \hat x_i^{(o)}\|_2 \le \epsilon$ for all $i$ and $\|x_j\|_2 \le \sqrt d$ for all $j$, then $\|s - \hat s\|_2 \le \epsilon \|W_q W_k^T\|_2$.
\end{theorem}

This result shows that for a single attention layer, the worst-case error in the approximate importance scores is proportional to the worst-case error in the approximate input embeddings, implying that lookahead-based KV cache dropping algorithms provide reliable importance estimates as long as the draft model remains reasonably accurate (see \cref{proof:speckv} for proof).

To evaluate the quality of different draft outputs, we compute $\epsilon$ for two algorithms (LAQ++ and SpecKV). While \cref{theorem:speckv} assumes a strict token-wise alignment (where each draft token $\hat{x}_i^{(o)}$ corresponds to a target token $x_i^{(o)}$), in practice, the draft and target models often generate sequences of different lengths ($\hat \noutput \neq \noutput$).
To address this mismatch, we relax the strict point-wise condition and empirically evaluate draft quality using the global distance between the centroids of the output sequences.
We compute $\epsilon$ as:
\begin{equation}\label{eqn:eps}
\epsilon = \Bigg\|\frac{1}{\noutput}\sum_{i=1}^{\noutput}x_i^{(o)} - \frac{1}{\hat \noutput}\sum_{i=1}^{\hat \noutput}\hat x_i^{(o)}\Bigg\|_2.
\end{equation}
This metric serves as a practical proxy for the distributional mismatch between the target and draft outputs.
We report $\epsilon$ averaged across all attention layers. As shown in \cref{fig:eps_corr_pcc:eps}, increasing draft quality (decreasing $\epsilon$) leads to higher downstream accuracy. Furthermore, \cref{app:sec:speckv_corr} demonstrates how SpecKV's lookahead mechanism improves importance score correlation and accuracy, particularly for long outputs.

\begin{figure} 
    \centering
     \begin{subfigure}[b]{0.32\textwidth}
         \centering
         \includegraphics[width=\textwidth]{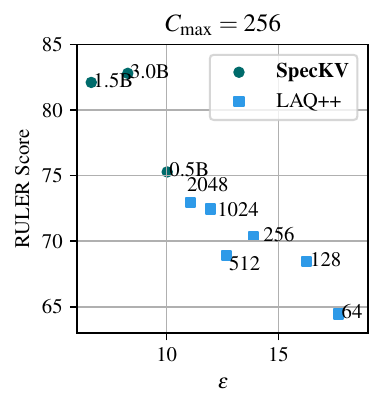}
         \caption{Score vs. $\epsilon$ (32B Target)}
         \label{fig:eps_corr_pcc:eps}
     \end{subfigure}
     \hfill
     \begin{subfigure}[b]{0.32\textwidth}
         \centering
         \includegraphics[width=\textwidth]{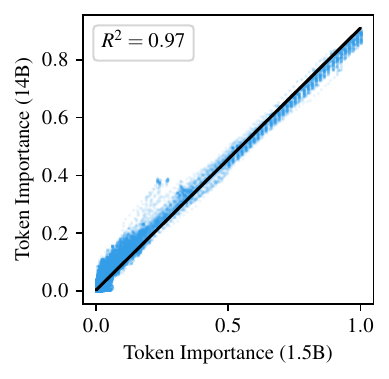}
         \caption{Importance Corr. (1.5B vs. 14B)}
         \label{fig:eps_corr_pcc:corr}
     \end{subfigure}
     \hfill
     \begin{subfigure}[b]{0.32\textwidth}
         \centering
         \includegraphics[width=\textwidth]{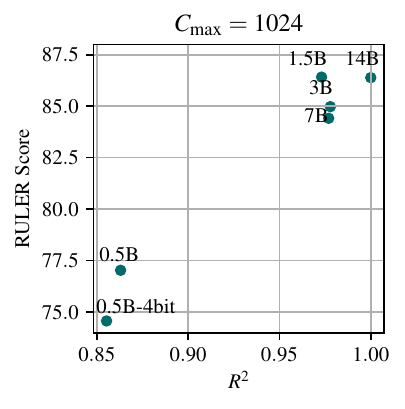}
         \caption{Score vs. $R^2$ (14B Target)}
         \label{fig:eps_corr_pcc:pcc}
     \end{subfigure}
    \caption{
    Experimental validation on RULER-32K tasks (5 samples each) using Qwen2.5 models. \textbf{(a)} Lower error $\epsilon$ (\cref{eqn:eps}) yields higher downstream scores. Increasing the draft model size (SpecKV) or initial cache size (LAQ++)\protect\footnotemark{} reduces $\epsilon$, with SpecKV outperforming LAQ++. \textbf{(b)} Importance scores (as used in \specpc{}) of the draft and target models are highly correlated. \textbf{(c)} For \specpc{}, a larger draft model improves both the token importance correlation ($R^2$) and the final task performance.}
    \label{fig:eps_corr_pcc}
\end{figure}
\footnotetext{For LAQ++, initial cache size refers to the value of $\cmax$ used during the sparse draft generation.}

\subsection{Justification for Draft-based Prompt Compression}
\label{sec:pc-justification} 

Most KV cache dropping algorithms estimate token importance using attention scores from the target model itself. While effective, this approach is too computationally expensive for tasks like prompt compression, where we need to estimate token importance without forwarding the entire input through the target model. A more efficient alternative is to use a smaller draft model for this estimation. To justify this approach, we study how the similarity between the draft and target models’ outputs correlates with the similarity of their attention activations in a single attention layer.



The target model attention layer uses weights $W_q$, $W_k$, and $W_v$, and the draft model attention layer uses $\hat{W}_q$, $\hat{W}_k$, and $\hat{W}_v$. Let the input prompt be $X = [x_1, \ldots, x_n]^T \in \mathbb{R}^{n \times d}$. The outputs of the target attention layer and its approximation are

\begin{equation}
Y = \operatorname{Softmax}\!\left(\tfrac{X W_q W_k^T X^T}{\sqrt{d}}\right) X W_v = A X W_v,
\quad
\hat{Y} = \operatorname{Softmax}\!\left(\tfrac{X \hat{W}_q \hat{W}_k^T X^T}{\sqrt{d}}\right) X \hat{W}_v = \hat{A} X \hat{W}_v,
\end{equation}

where $X = [x_1, \ldots, x_n]^T \in \mathbb{R}^{n \times d}$,
$A = [a_1, \ldots, a_n]^T$ is the attention matrix, and $\hat{A} = [\hat{a}_1, \ldots, \hat{a}_n]^T$ is the approximate attention matrix.

If the scaled inputs satisfy the Restricted Isometry Property (RIP)\footnote{The input embedding matrix may satisfy the RIP if its entries are approximately uniformly or normally distributed. RIP can also hold with positional embeddings constructed from a Fourier basis.}~\citep{rip}, a condition widely studied in compressed sensing to ensure the stable recovery of sparse signals, we can establish the following bound:
\begin{theorem}\label{theorem:specpc_rip2}
    If there exists a constant $c$ such that $cX^T$ satisfies the Restricted Isometry Property with parameters $(2k, \delta)$, where $\delta$ is the restricted isometry constant and  $k$ is the approximate sparsity of $a_i$ and $\hat a_i$, and the output error satisfies $\|y_i - \hat y_i\|_2 \le \epsilon\|X\|_{\infty, 2}$, then the attention error satisfies $\|a_i - \hat a_i\|_2 \le \tfrac{2c\epsilon\|X\|_{\infty, 2}}{\sigma_{\min}(W_v)(1-\delta)}$.\footnote{$\|X\|_{\infty, 2}$ denotes the maximum $\ell_2$ norm of $X$’s rows; $\sigma_{\min}(W_v)$ is the smallest singular value of $W_v$.}
\end{theorem}
This result offers a surprising and elegant connection: it reveals that mathematical tools developed for compressed sensing can also bound the error in attention approximations. Specifically, it shows that the worst-case error in the approximate attention activations is proportional to the worst-case error in the approximate outputs, with the constant depending on the conditioning of the weight matrices and the maximum input embedding norm. This implies that if the draft model provides a reasonable approximation of the output, it also gives a reasonable approximation of the attention activations (see \cref{proof:rip} for proof). Furthermore, even if the scaled inputs do not satisfy the RIP, we can still bound the attention approximation error by applying \Cref{theorem:specpc} (see \Cref{proof:specpc} for proof).

In addition to the theoretical analysis, we examine the correlation between the importance scores (as computed by \specpc{}) for Qwen2.5-Instruct (1.5B as draft, 14B as target). In \cref{fig:eps_corr_pcc:corr}, we plot the draft importance scores against the corresponding target importance scores. The results reveal a strong correlation, supporting the use of draft attention activations to approximate target token importance. Furthermore, this correlation strengthens as the draft model size increases (\cref{fig:eps_corr_pcc:pcc}).


\section{\framework{} Methods}
\subsection{\speckv{}: Robust Importance Estimation for KV Cache Dropping}
\label{sec:speckv}


\pend{Existing sparse attention and KV cache dropping methods, such as SnapKV, estimate token importance by analyzing recent attention activations. This approach can be inaccurate when the set of important KV pairs shifts during generation, as past patterns do not always predict future ones. We argue that a more robust estimate can be derived from the attention activations of draft queries for future tokens.}

\pend{LAQ++ attempts this by generating draft queries from the target model using an initially compressed cache. These queries are then used to compute more accurate importance scores for a second, more informed compression pass. However, this two-pass method provides no reduction in peak memory because it must store the entire original KV cache of the target model to avoid recomputing the expensive prefill stage.}

\pend{To overcome this limitation, we propose SpecKV. Our method employs a lightweight draft model to generate the draft output, which substantially reduces the cost of the lookahead step. This enables an accurate compression of the KV cache without sacrificing peak memory reduction.
}

\begin{figure}
     \centering
    \includegraphics[width=0.7\textwidth]{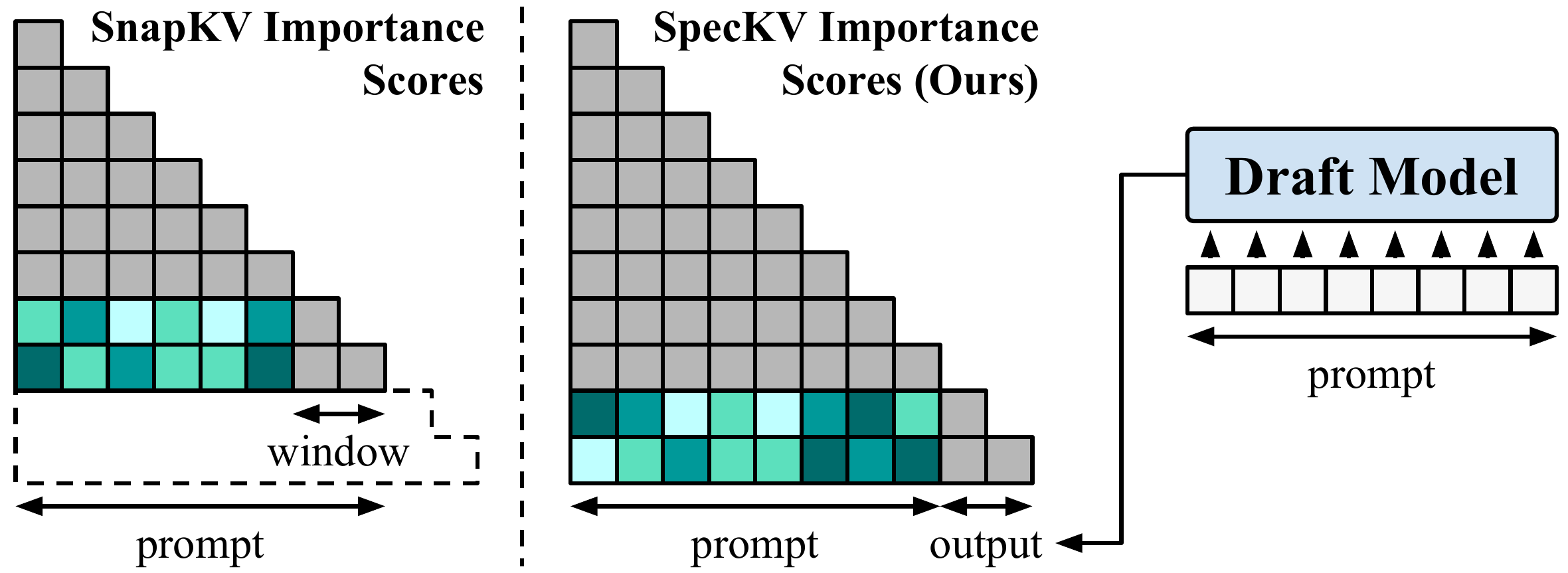}

     
    \caption{Overview of \speckv{}: Instead of using only the last prompt tokens like SnapKV, \speckv{} employs a lightweight draft model to generate lookahead tokens, providing richer context for more accurate KV importance estimation. Tokens in window are always retained.}
    \label{fig:speckv}
\end{figure}

\speckv{} (\cref{alg:speckv}) begins by generating a draft output of length $\nlookahead$ using a small draft model, which acts as a proxy for the target model’s future outputs. During prefilling, both the input tokens and the draft tokens are passed through the target model. For each attention head, we compute token importance scores by measuring the cross-attention between the queries from the last $\ncompwnd$ input tokens and the draft output tokens to the remaining input keys (\cref{fig:speckv}). We apply local pooling with kernel size $k$ to the attention scores to maintain continuity. These scores guide two optimizations: sparse prefilling and KV cache dropping. For sparse prefilling, we use a variation of the Vertical-Slash kernel pattern introduced in \citep{minference}. For KV cache dropping, we retain the top $\cmax - \ncompwnd$ KV pairs with the highest importance scores, along with the final $\ncompwnd$ KV pairs from the most recent tokens.



\subsection{\specpc{}: Leveraging Draft Models for Efficient Prompt Compression}
\label{sec:specpc}

\pend{\speckv{} leverages the draft model outputs to enable more effective KV cache dropping. However, greater efficiency gains are possible by leveraging more information from the draft model. As demonstrated in \cref{sec:pc-justification}, draft model attention scores serve as reliable estimates of target token importance. Building on this insight, we introduce \specpc{}, an extension of SpecPrefill~\citep{specprefill}, which compresses the prompt to reduce latency and memory usage during both prefilling and decoding, surpassing the efficiency benefits of traditional KV cache dropping.}

\specpc{} (\cref{alg:spc}) feeds an input prompt (length $\ninput$) to the draft model and directly extracts its attention activations $A \in \mathbb{R}^{n_\text{layer} \times n_\text{head} \times \left(\ninput+\nlookahead-1\right) \times \ninput}$, where $n_\text{layer}$ and $n_\text{head}$ denote the number of layers and heads. These activations indicate token importance and are used to drop less relevant tokens from the prompt.

SpecPrefill~\citep{specprefill} uses a window size of $\ncompwnd = 1$, meaning it relies on attention scores from queries associated with the last input token and $\nlookahead$ draft output tokens to compute token importance. However, our experiments show that the optimal choice of $\ncompwnd$ (the number of input queries used for importance estimation) is task-dependent, with some tasks benefiting from additional input queries. To address this, we adopt a large window with non-uniform token weights, placing greater emphasis on tokens near the prompt’s end to achieve robust, task-wide performance. Window tokens are reweighted so that the $j$th token from the end receives weight $\frac{\ncompwnd - (j - 1)}{\ncompwnd}$. Max aggregation is performed across layers, heads, and queries to produce a single importance score per token (excluding the always-kept last $\ncompwnd$ tokens). Additionally, as PyramidKV~\citep{pyramidkv} showed that attention is more focused in deeper layers, we exclude the first $\lskip$ layers during aggregation.

We apply average, then max pooling, so selected tokens also include nearby context, avoiding static chunking that could split related tokens (e.g., key-value pairs). This maintains the local context LLMs require. Unlike other methods that select entire sentences, we avoid sentence-level pre-processing to support non-text inputs, such as images. We then select the top-$\cmax$ tokens with the highest scores, always including window tokens, to form the compressed prompt. We pass the compressed prompt to the target model with reassigned contiguous position IDs, enabling SpecPC to surpass the model’s maximum context length.







\revised{\subsection{SpecKV-PC: Cascaded Compression with \speckv{} and \specpc{}}}
\label{sec:speckv_pc}

\begin{figure}
     \centering
    \includegraphics[width=0.7\textwidth]{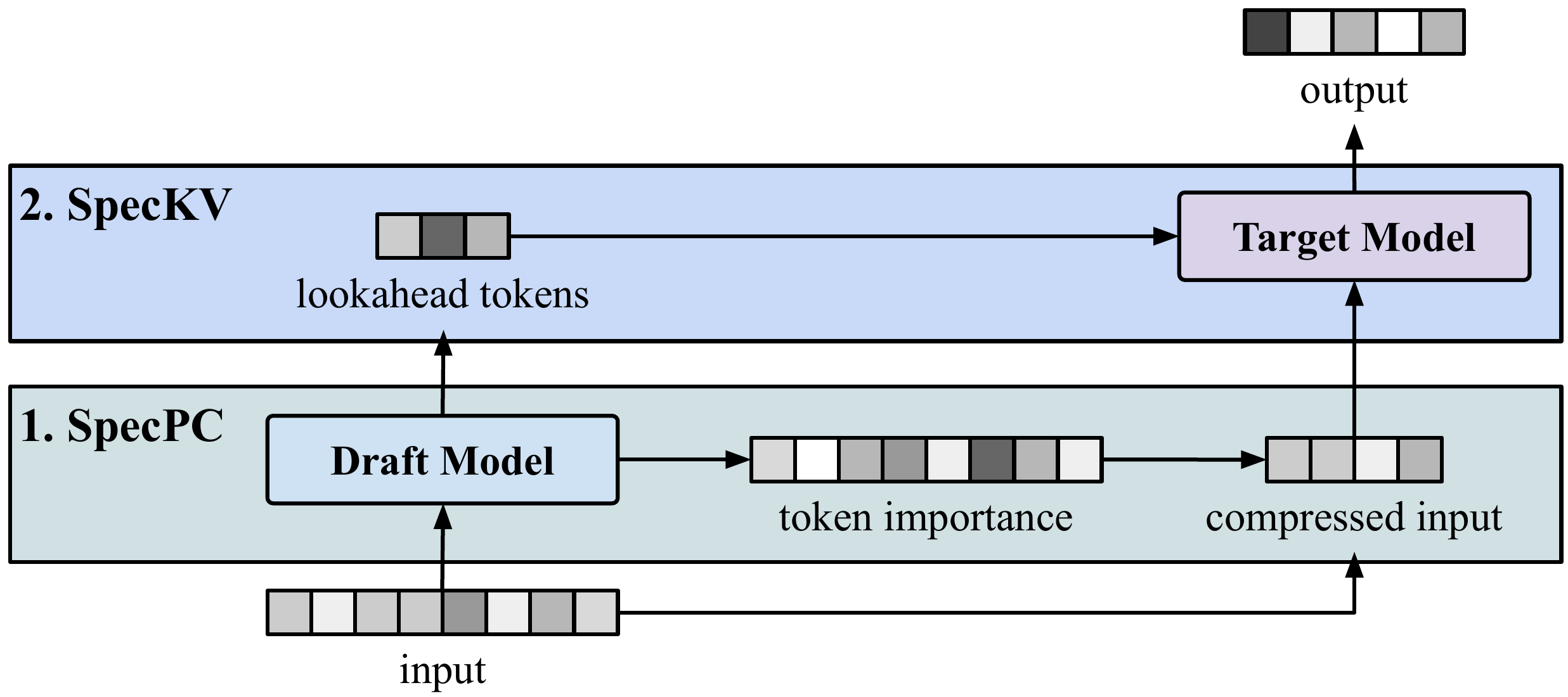}
    \caption{\revised{Overview of cascaded compression with SpecKV-PC: First, the draft model produces token importance scores and lookahead tokens. Next, \specpc{} uses these scores to compress the initial input prompt. Finally, the target model is prefilled using both the compressed prompt and the lookahead tokens, while \speckv{} compresses its KV cache}}
    \label{fig:speckv_pc}
\end{figure}
\revised{
SpecKV-PC integrates \speckv{} and \specpc{} into a highly efficient, two-stage compression pipeline. The core strategy leverages a cascaded approach, which first compresses the prompt with SpecPC and then further compresses the KV cache with SpecKV (\cref{fig:speckv_pc}). Because of fewer target model activations, SpecKV-PC achieves substantially lower latency and a smaller memory footprint than \speckv{} alone, as the computationally intensive target model processes only a fraction of the original prompt.}



\section{Experiments}
\label{sec:exps}

\subsection{Setup}

We benchmark \speckv{} and \specpc{} against baselines on RULER~\citep{ruler} and LongBench~\citep{longbench} using two model pairs: Qwen2.5-Instruct~\citep{qwen2.5} (14B target with 1.5B/0.5B drafts for KV dropping/prompt compression) and Llama-3-Instruct~\citep{grattafiori2024llama} (3.1-70B 4-bit target with 3.2-3B/1B drafts, respectively). 

RULER is a synthetic benchmark with 13 tasks of varying complexity,  including tasks such as key-value retrieval (NIAH), multi-hop tracing, and aggregation. It can be generated at any sequence length to assess a model’s effective context window. We evaluate at 4K, 8K, 16K, 32K, and 64K (Qwen is excluded at 64K due to its 32K sequence limit). LongBench contains 12 English, five Chinese, two code, and three synthetic tasks across five categories. We exclude the Chinese tasks (unsupported by Llama) and synthetic tasks (already covered by RULER). \pend{We select 50 examples from each task, yielding 700 examples from LongBench and 650 examples at each context length for RULER.}

For \speckv{}, we compare against KV dropping methods H2O~\citep{h2o}, SnapKV~\citep{snapkv}, PyramidKV~\citep{pyramidkv}, and LAQ++~\citep{laq}, using $\cmax = 256$. \revised{Additionally, we assess SpecKV-PC-2048 that first compresses to 2048 tokens via \specpc{} before applying \speckv{} ($\cmax = 256$).}
For \specpc{}, we benchmark against LLMLingua-v2~\citep{pan2024llmlingua2}, CPC~\citep{liskavets2024cpc}, R2C~\citep{choi2024r2c}, and SpecPrefill~\citep{specprefill} with $\cmax = 1024$. 
Based on our ablation studies (\cref{app:fig:nlookahead}), we set $\nlookahead$ to the maximum token limit for \speckv{} and to one for \specpc{}. For SpecPrefill and LAQ++, we use $\nlookahead = 8$ as in the official paper.
\cref{app:datasets,app:exp_details,app:add_res} contain datasets, details, and additional results, including cross-family evaluations, different $\cmax$ settings, and multimodal experiments on MileBench~\citep{milebench} with Qwen2.5-VL~\citep{qwen2.5-vl}.

\begin{figure} 
    \centering
     \begin{subfigure}[b]{0.4925\textwidth}
         \centering
         \includegraphics[width=\textwidth]{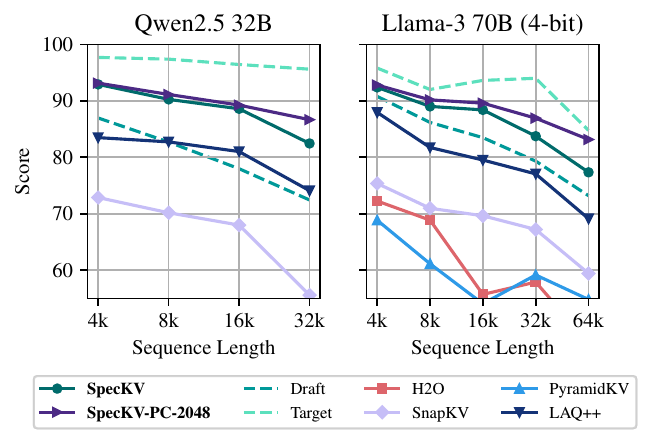}
         \caption{\revised{KV dropping ($\cmax = 256$)}}
     \end{subfigure}
     \hfill
     \begin{subfigure}[b]{0.49\textwidth}
         \centering
         \includegraphics[width=\textwidth]{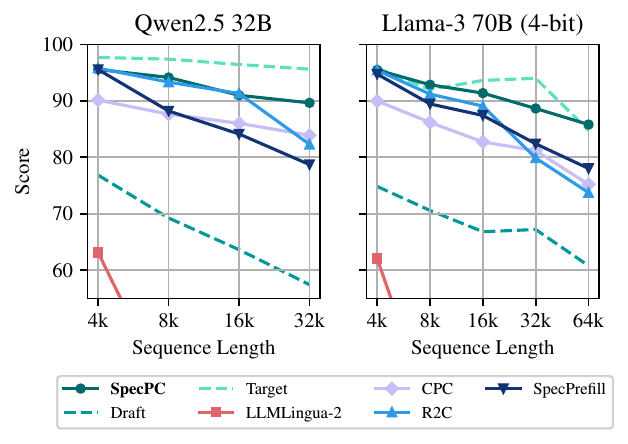}
         \caption{Prompt compression ($\cmax = 1024$)}
     \end{subfigure}
    \caption{Performance of \speckv{} and \specpc{}. Both methods consistently outperform all baselines across sequence lengths, maintaining strong results at longer contexts. \revised{SpecKV-PC further improves upon \speckv{} to achieve state-of-the-art results for KV dropping.} Note that H2O and PyramidKV are not plotted for Qwen2.5 32B as their performance falls outside the visible range.}
    \label{fig_ruler_llama_qwen}
\end{figure}

\subsection{Results}
\label{sec:results}

\cref{fig_ruler_llama_qwen} (RULER) and  \cref{tab:longbench_sidebyside} (LongBench) compare our methods with baselines. All our methods consistently outperform other baselines, demonstrating superior KV cache and prompt compression. Their performance far exceeds the draft model, highlighting robustness even with weaker drafts. Performance improves further with better drafts (\cref{app:fig:better_draft}). On RULER and LongBench, \speckv{} and \specpc{} consistently surpass existing baselines across multiple model families and context lengths. \revised{Interestingly, SpecKV-PC outperforms \speckv{} alone, suggesting that the initial SpecPC-based prompt compression acts as an effective pre-filter by removing easy-to-identify, unimportant tokens.}
For larger $\cmax$, our methods remain superior (\cref{app:fig:extra_ruler_kv,app:fig:extra_ruler_pc}; \cref{app:tab:extra_longbench_qwen25_32b,app:tab:extra_longbench_llama31_70b,app:tab:extra_longbench_qwen25_72b}). Finally, we provide additional results regarding AdaKV (\cref{app:sec:adakv}), multimodal tasks (\cref{app:add_res_milebench}), cross-family settings (\cref{app:sec:cross_model}), and SpecKV-PC prompt compression ratios (\cref{app:sec:combined_speckv_specpc}).

\begin{table}[h!]
\centering
\caption{LongBench performance with Qwen2.5 and Llama-3.}
\resizebox{\linewidth}{!}{
\begin{tabular}{llcccccc|cccccc}
\toprule
 & & \multicolumn{6}{c|}{\textbf{Qwen2.5 32B}} & \multicolumn{6}{c}{\textbf{Llama-3 70B (4-bit)}} \\
\cmidrule(lr){3-8}\cmidrule(lr){9-14}
\textbf{Group} & \textbf{Method} 
 & \rotatebox{45}{SingleQA} & \rotatebox{45}{MultiQA} & \rotatebox{45}{Summ.} & \rotatebox{45}{Few-shot} & \rotatebox{45}{Code} & \rotatebox{45}{All}
 & \rotatebox{45}{SingleQA} & \rotatebox{45}{MultiQA} & \rotatebox{45}{Summ.} & \rotatebox{45}{Few-shot} & \rotatebox{45}{Code} & \rotatebox{45}{All} \\
\midrule
\multirow{1}{*}{Dense}

& Target  & 56.01 & 43.99 & 25.90 & 64.06 & 44.74 & 47.78 & 55.02 & 47.06 & 28.61 & 70.47 & 48.19 & 49.99 \\
\midrule
\multirow{6}{*}{KV}
 &   H2O    & 46.63 & 30.81 & 19.88 & 56.03 & 39.27 & 39.32 & 54.07 & 41.30 & 22.55 & 49.10 & 54.14 & 43.52 \\
  &   SnapKV    & 52.54 & 40.21 & 19.89 & 61.18 & 40.12 & 42.98 & \underline{55.88} & 45.30 & 22.49 & 62.15 & 55.49 & 47.75 \\
  &   PyramidKV    & 50.92 & 37.26 & 18.90 & 63.24 & 40.20 & 43.19 & 55.41 & 45.59 & 22.50 & 59.06 & 49.90 & 46.25 \\
  &   LAQ++    & \textbf{55.15} & \underline{44.14} & 22.24 & 63.25 & 41.19 & 45.79 & 54.90 & 46.48 & 22.83 & \underline{64.31} & 55.10 & 48.43 \\
  &   \textbf{SpecKV}    & \underline{53.48} & 43.77 & \underline{24.02} & \textbf{63.79} & \underline{44.80} & \underline{46.06} & 51.80 & \textbf{47.23} & \underline{25.53} & 64.02 & \textbf{58.75} & \underline{48.80} \\
 &    \revised{\textbf{SpecKV-PC-2048}}   & 52.60 & \textbf{44.52} & \textbf{24.11} & \underline{63.38} & \textbf{48.45} & \textbf{46.48} & \textbf{61.42} & \underline{47.15} & \textbf{26.51} & \textbf{66.94} & \underline{58.19} & \textbf{51.60} \\
\midrule
\multirow{5}{*}{PC}
& LLMLingua-2  & 33.83 & 26.39 & 22.85 & 32.46 & \underline{43.01} & 30.90 & 37.95 & 28.20 & 23.35 & 42.37 & 37.63 & 33.63 \\
& CPC  & 45.60 & 40.62 & 23.09 & 60.08 & 32.31 & 40.91 & 45.14 & 39.41 & 24.86 & 61.40 & 37.58 & 41.97 \\
& R2C  & 50.49 & 40.37 & \underline{23.26} & 53.45 & 34.11 & 39.88 & 48.93 & 42.01 & 25.38 & 58.91 & 40.19 & 43.29 \\
& SpecPrefill  & 45.94 & 39.32 & 23.16 & \underline{62.04} & \textbf{43.17} & \underline{42.70} & \underline{54.62} & \textbf{46.43} & \underline{25.63} & \underline{64.80} & \underline{44.92} & \underline{48.37} \\
& \textbf{SpecPC}  & 51.23 & 41.40 & \textbf{23.37} & \textbf{62.26} & 38.23 & \textbf{43.66} & \textbf{56.84} & \underline{44.48} & \textbf{25.91} & \textbf{67.37} & \textbf{47.15} & \textbf{48.44} \\
\bottomrule
\end{tabular}}
\label{tab:longbench_sidebyside}
\end{table}

\subsection{Efficiency}

We evaluate latency and memory on a single NVIDIA H200 (141GB) GPU. The target model is Qwen2.5-32B, with draft models of Qwen2.5-1.5B for KV dropping and Qwen2.5-0.5B for prompt compression. Latency is measured as the time to generate 64 tokens including all draft stages, with $\nlookahead$ set to 64 for \speckv{} and SpecKV-PC\footnote{While some tasks can generate up to 128 tokens, benchmark outputs are typically under 64 tokens.}, 8 for SpecPrefill and LAQ++, and 1 for \specpc{}, with SpecKV-PC prompts compressed to 2048 tokens. For KV dropping methods, we report peak system memory, while for prompt compression, we report memory for the compression stage only, because the target model uses the same amount of memory for all prompt compression algorithms. A detailed breakdown of latency and memory is included in \cref{app:sec:ext_lat_and_mem}.

\speckv{} outperforms both SnapKV and LAQ++ in speed, driven by efficient sparse prefilling and a low-latency draft model (\cref{fig:mem_latency}). \revised{Notably, SpecKV-PC achieves a 75\% reduction in latency compared to LAQ++ at 64k context, as it significantly reduces the target prefill bottleneck through prompt compression.} Regarding memory, while \speckv{} requires more memory than SnapKV to store draft weights, this overhead is constant and can be further reduced by offloading to a CPU. Crucially, \speckv{} is far more memory-efficient than LAQ++, which needs the entire target KV cache to function, matching the memory footprint of the dense target model. Thus, \speckv{} offers a superior combination of accuracy and memory efficiency. \revised{Furthermore, SpecKV-PC yields substantial peak memory savings (around 25GB compared to LAQ++ at 64k context) by feeding a shorter prompt to the target model. Other KV dropping methods are omitted as their performance is similar to SnapKV.}

\specpc{} achieves the lowest latency among all baselines. It avoids the CPU preprocessing overhead of CPC and R2C and is faster than SpecPrefill due to a shorter lookahead. \specpc{} is also the most memory-efficient, using substantially less memory than R2C. Overall, prompt compression is faster than KV dropping because only the compressed prompt tokens ($\cmax$) are passed to the target model.

\begin{figure}[h]
\centering
\begin{subfigure}[b]{0.495\textwidth}
\centering
\includegraphics[width=\textwidth]{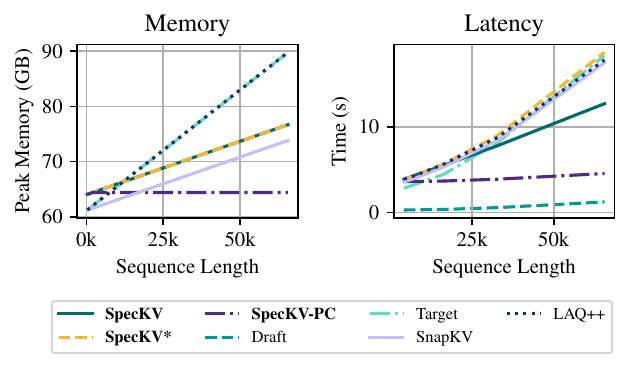}
\caption{\revised{SpecKV memory and latency.}}
\label{fig:mem_latency:speckv}
\end{subfigure}
\hfill
\begin{subfigure}[b]{0.495\textwidth}
\centering
\includegraphics[width=\textwidth]{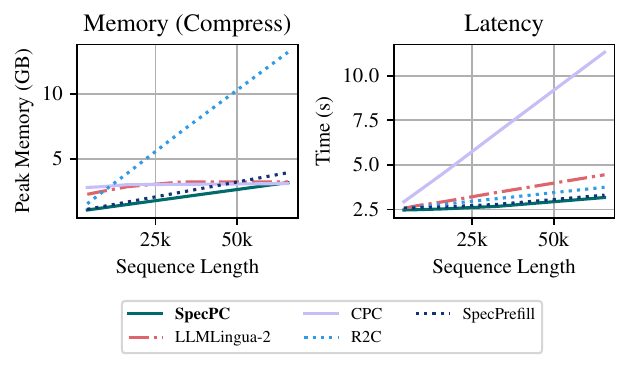}
\caption{SpecPC memory and latency.}
\label{fig:mem_latency:specpc}
\end{subfigure}
\caption{Peak memory usage and latency. \speckv{}* denotes \speckv{} without sparse prefill.}
\label{fig:mem_latency}
\end{figure}

\section{Discussion}
\label{sec:discussion}

In this paper, we present \framework{}, a framework that leverages draft model lookahead for approximate inference. Within this framework, we introduce two concrete algorithms: \speckvlong{} (\speckv{}) for KV cache dropping with sparse prefill, and \specpclong{} (\specpc{}) for prompt compression. \revised{We further propose SpecKV-PC, a cascaded pipeline that synergizes these approaches for improved efficiency and accuracy.} Our approach is grounded in theoretical and empirical analyses that justify lookahead-based KV cache dropping and the use of draft model token importance to approximate target model importance. Through comprehensive experiments on long-context benchmarks, we show that our methods consistently achieve state-of-the-art accuracy under fixed KV cache and prompt size constraints, surpassing prior baselines. These contributions establish draft model lookahead as a powerful tool for efficient long-context inference, extending the role of draft models beyond speculative decoding.

\paragraph{Limitations and Future Work}

For \speckv{}, draft generation causes minimal latency even for larger $\nlookahead$ (\cref{app:sec:ext_lat_and_mem}). However, very long outputs or large $\nlookahead$ values may reduce performance. In these cases, lowering $\nlookahead$ could maintain speed with little loss in accuracy.
For \specpc{}, increasing $\nlookahead$ to generate more tokens led to only minor accuracy gains; better leveraging longer drafts remains future work. 
\revised{Additionally, while our methods are robust using very small (\cref{app:fig:better_draft}) or cross-family (\cref{app:sec:cross_model}) draft models, a reasonably accurate draft model is still required.}
Currently, \framework{} supports sparse prefill, KV dropping, and prompt compression. Extensions such as lookahead-based sparse decoding or iterative KV cache dropping, where KV entries are periodically removed using draft lookahead, could further improve support for reasoning models with long outputs.



\section*{Reproducibility statement} 
We provide a link to our public code. All algorithms, datasets, and experimental details, including hyperparameter settings, can be found in \cref{sec:exps} and the appendix (\cref{app:algo,app:datasets,app:exp_details}).

\bibliography{iclr2026_conference}

\begin{thebibliography}{57}
\providecommand{\natexlab}[1]{#1}
\providecommand{\url}[1]{\texttt{#1}}
\expandafter\ifx\csname urlstyle\endcsname\relax
  \providecommand{\doi}[1]{doi: #1}\else
  \providecommand{\doi}{doi: \begingroup \urlstyle{rm}\Url}\fi

\bibitem[Achiam et~al.(2023)Achiam, Adler, Agarwal, Ahmad, Akkaya, Aleman, Almeida, Altenschmidt, Altman, Anadkat, et~al.]{gpt4}
Josh Achiam, Steven Adler, Sandhini Agarwal, Lama Ahmad, Ilge Akkaya, Florencia~Leoni Aleman, Diogo Almeida, Janko Altenschmidt, Sam Altman, Shyamal Anadkat, et~al.
\newblock {GPT}-4 technical report.
\newblock \emph{arXiv preprint}, abs/2303.08774, 2023.
\newblock URL \url{https://arxiv.org/abs/2303.08774}.

\bibitem[Agarwal et~al.(2025)Agarwal, Ahmad, Ai, Altman, Applebaum, Arbus, Arora, Bai, Baker, Bao, et~al.]{gpt-oss}
Sandhini Agarwal, Lama Ahmad, Jason Ai, Sam Altman, Andy Applebaum, Edwin Arbus, Rahul~K Arora, Yu~Bai, Bowen Baker, Haiming Bao, et~al.
\newblock {GPT-OSS-120B} \& {GPT-OSS-20B} model card.
\newblock \emph{arXiv preprint}, abs/2508.10925, 2025.
\newblock URL \url{https://arxiv.org/abs/2508.10925}.

\bibitem[Ainslie et~al.(2023)Ainslie, Lee-Thorp, de~Jong, Zemlyanskiy, Lebron, and Sanghai]{gqa}
Joshua Ainslie, James Lee-Thorp, Michiel de~Jong, Yury Zemlyanskiy, Federico Lebron, and Sumit Sanghai.
\newblock {GQA}: Training generalized multi-query {Transformer} models from multi-head checkpoints.
\newblock In \emph{Proceedings of the 2023 Conference on Empirical Methods in Natural Language Processing}, pp.\  4895--4901, 2023.

\bibitem[Bai et~al.(2025)Bai, Chen, Liu, Wang, Ge, Song, Dang, Wang, Wang, Tang, et~al.]{qwen2.5-vl}
Shuai Bai, Keqin Chen, Xuejing Liu, Jialin Wang, Wenbin Ge, Sibo Song, Kai Dang, Peng Wang, Shijie Wang, Jun Tang, et~al.
\newblock {Qwen2.5-VL} technical report.
\newblock \emph{arXiv preprint}, abs/2502.13923, 2025.
\newblock URL \url{https://arxiv.org/abs/2502.13923}.

\bibitem[Bai et~al.(2024)Bai, Lv, Zhang, Lyu, Tang, Huang, Du, Liu, Zeng, Hou, Dong, Tang, and Li]{longbench}
Yushi Bai, Xin Lv, Jiajie Zhang, Hongchang Lyu, Jiankai Tang, Zhidian Huang, Zhengxiao Du, Xiao Liu, Aohan Zeng, Lei Hou, Yuxiao Dong, Jie Tang, and Juanzi Li.
\newblock {L}ong{B}ench: A bilingual, multitask benchmark for long context understanding.
\newblock In \emph{Proceedings of the 62nd Annual Meeting of the Association for Computational Linguistics (Volume 1: Long Papers)}, pp.\  3119--3137, Bangkok, Thailand, 2024. Association for Computational Linguistics.
\newblock \doi{10.18653/v1/2024.acl-long.172}.
\newblock URL \url{https://aclanthology.org/2024.acl-long.172}.

\bibitem[Beltagy et~al.(2020)Beltagy, Peters, and Cohan]{sliding_window}
Iz~Beltagy, Matthew~E. Peters, and Arman Cohan.
\newblock {Longformer}: The long-document {Transformer}.
\newblock \emph{arXiv preprint}, abs/2004.05150, 2020.
\newblock URL \url{https://arxiv.org/abs/2004.05150}.

\bibitem[Brown et~al.(2020)Brown, Mann, Ryder, Subbiah, Kaplan, Dhariwal, Neelakantan, Shyam, Sastry, Askell, Agarwal, Herbert{-}Voss, Krueger, Henighan, Child, Ramesh, Ziegler, Wu, Winter, Hesse, Chen, Sigler, Litwin, Gray, Chess, Clark, Berner, McCandlish, Radford, Sutskever, and Amodei]{gpt3}
Tom~B. Brown, Benjamin Mann, Nick Ryder, Melanie Subbiah, Jared Kaplan, Prafulla Dhariwal, Arvind Neelakantan, Pranav Shyam, Girish Sastry, Amanda Askell, Sandhini Agarwal, Ariel Herbert{-}Voss, Gretchen Krueger, Tom Henighan, Rewon Child, Aditya Ramesh, Daniel~M. Ziegler, Jeffrey Wu, Clemens Winter, Christopher Hesse, Mark Chen, Eric Sigler, Mateusz Litwin, Scott Gray, Benjamin Chess, Jack Clark, Christopher Berner, Sam McCandlish, Alec Radford, Ilya Sutskever, and Dario Amodei.
\newblock Language models are few-shot learners.
\newblock In Hugo Larochelle, Marc'Aurelio Ranzato, Raia Hadsell, Maria{-}Florina Balcan, and Hsuan{-}Tien Lin (eds.), \emph{Advances in Neural Information Processing Systems 33: Annual Conference on Neural Information Processing Systems 2020, NeurIPS 2020, December 6-12, 2020, virtual}, 2020.
\newblock URL \url{https://proceedings.neurips.cc/paper/2020/hash/1457c0d6bfcb4967418bfb8ac142f64a-Abstract.html}.

\bibitem[Cai et~al.(2024{\natexlab{a}})Cai, Li, Geng, Peng, Lee, Chen, and Dao]{medusa}
Tianle Cai, Yuhong Li, Zhengyang Geng, Hongwu Peng, Jason~D Lee, Deming Chen, and Tri Dao.
\newblock {Medusa}: Simple {LLM} inference acceleration framework with multiple decoding heads.
\newblock In \emph{International Conference on Machine Learning}, pp.\  5209--5235. PMLR, 2024{\natexlab{a}}.

\bibitem[Cai et~al.(2024{\natexlab{b}})Cai, Zhang, Gao, Liu, Liu, Lu, Xiong, Dong, Chang, Hu, and Wen]{pyramidkv}
Zefan Cai, Yichi Zhang, Bofei Gao, Yuliang Liu, Tianyu Liu, Keming Lu, Wayne Xiong, Yue Dong, Baobao Chang, Junjie Hu, and Xiao Wen.
\newblock {PyramidKV}: Dynamic {KV} cache compression based on pyramidal information funneling.
\newblock \emph{arXiv preprint}, abs/2406.02069, 2024{\natexlab{b}}.
\newblock URL \url{https://arxiv.org/abs/2406.02069}.

\bibitem[Candes \& Tao(2005)Candes and Tao]{rip}
E.J. Candes and T.~Tao.
\newblock Decoding by linear programming.
\newblock \emph{IEEE Transactions on Information Theory}, 51\penalty0 (12):\penalty0 4203--4215, 2005.
\newblock \doi{10.1109/TIT.2005.858979}.

\bibitem[Chen et~al.(2023)Chen, Borgeaud, Irving, Lespiau, Sifre, and Jumper]{specdec2}
Charlie Chen, Sebastian Borgeaud, Geoffrey Irving, Jean-Baptiste Lespiau, Laurent Sifre, and John Jumper.
\newblock Accelerating large language model decoding with speculative sampling.
\newblock \emph{arXiv preprint}, abs/2302.01318, 2023.
\newblock URL \url{https://arxiv.org/abs/2302.01318}.

\bibitem[Chen et~al.(2024)Chen, Zhao, Liu, Bai, Lin, Zhou, and Chang]{fastv}
Liang Chen, Haozhe Zhao, Tianyu Liu, Shuai Bai, Junyang Lin, Chang Zhou, and Baobao Chang.
\newblock An image is worth 1/2 tokens after layer 2: Plug-and-play inference acceleration for large vision-language models.
\newblock In \emph{European Conference on Computer Vision}, pp.\  19--35. Springer, 2024.

\bibitem[Choi et~al.(2024)Choi, Lee, Choi, Park, and Lee]{choi2024r2c}
Eunseong Choi, Sunkyung Lee, Minjin Choi, Jun Park, and Jongwuk Lee.
\newblock From reading to compressing: Exploring the multi-document reader for prompt compression.
\newblock In \emph{Findings of the Association for Computational Linguistics: EMNLP 2024}, pp.\  14734--14754, 2024.

\bibitem[Conneau et~al.(2020)Conneau, Khandelwal, Goyal, Chaudhary, Wenzek, Guzm{\'a}n, Grave, Ott, Zettlemoyer, and Stoyanov]{conneau2020roberta}
Alexis Conneau, Kartikay Khandelwal, Naman Goyal, Vishrav Chaudhary, Guillaume Wenzek, Francisco Guzm{\'a}n, Edouard Grave, Myle Ott, Luke Zettlemoyer, and Veselin Stoyanov.
\newblock Unsupervised cross-lingual representation learning at scale.
\newblock In \emph{Proceedings of the 58th Annual Meeting of the Association for Computational Linguistics}, pp.\  8440--8451, Online, 2020. Association for Computational Linguistics.
\newblock \doi{10.18653/v1/2020.acl-main.747}.
\newblock URL \url{https://aclanthology.org/2020.acl-main.747}.

\bibitem[Dao(2024)]{flashattn2}
Tri Dao.
\newblock {FlashAttention-2}: Faster attention with better parallelism and work partitioning.
\newblock In \emph{The Twelfth International Conference on Learning Representations}, 2024.

\bibitem[Dingjie et~al.(2024)Dingjie, Chen, Chen, Yu, Wan, and Wang]{milebench}
Song Dingjie, Shunian Chen, Guiming~Hardy Chen, Fei Yu, Xiang Wan, and Benyou Wang.
\newblock {MileBench}: Benchmarking {MLLM}s in long context.
\newblock In \emph{First Conference on Language Modeling}, 2024.

\bibitem[Du et~al.(2024)Du, Liu, Wang, Wang, Liu, Chen, Feng, Sha, Peng, and Lou]{code1}
Xueying Du, Mingwei Liu, Kaixin Wang, Hanlin Wang, Junwei Liu, Yixuan Chen, Jiayi Feng, Chaofeng Sha, Xin Peng, and Yiling Lou.
\newblock Evaluating large language models in class-level code generation.
\newblock In \emph{Proceedings of the IEEE/ACM 46th International Conference on Software Engineering}, ICSE '24, New York, NY, USA, 2024. Association for Computing Machinery.
\newblock ISBN 9798400702174.
\newblock \doi{10.1145/3597503.3639219}.
\newblock URL \url{https://doi.org/10.1145/3597503.3639219}.

\bibitem[Feng et~al.(2025)Feng, Lv, Cao, Xie, and Zhou]{adakv}
Yuan Feng, Junlin Lv, Yukun Cao, Xike Xie, and S~Kevin Zhou.
\newblock Ada-{KV}: Optimizing {KV} cache eviction by adaptive budget allocation for efficient {LLM} inference.
\newblock In \emph{The Thirty-ninth Annual Conference on Neural Information Processing Systems}, 2025.
\newblock URL \url{https://openreview.net/forum?id=tcisuhGsQZ}.

\bibitem[Fu et~al.(2025)Fu, Cai, Asi, Xiong, Dong, and Xiao]{headkv}
Yu~Fu, Zefan Cai, Abedelkadir Asi, Wayne Xiong, Yue Dong, and Wen Xiao.
\newblock Not all heads matter: A head-level {KV} cache compression method with integrated retrieval and reasoning.
\newblock In \emph{The Thirteenth International Conference on Learning Representations}, 2025.
\newblock URL \url{https://openreview.net/forum?id=FJFVmeXusW}.

\bibitem[{Gemma Team}(2025)]{gemma_2025}
{Gemma Team}.
\newblock {Gemma} 3.
\newblock 2025.
\newblock URL \url{https://goo.gle/Gemma3Report}.

\bibitem[{Google DeepMind}(2025)]{gemini25pro}
{Google DeepMind}.
\newblock {Gemini} 2.5 {Pro}: Advanced reasoning {AI} model, 2025.
\newblock URL \url{https://deepmind.google/technologies/gemini/pro/}.

\bibitem[Grattafiori et~al.(2024)Grattafiori, Dubey, Jauhri, Pandey, Kadian, Al-Dahle, Letman, Mathur, Schelten, Vaughan, et~al.]{grattafiori2024llama}
Aaron Grattafiori, Abhimanyu Dubey, Abhinav Jauhri, Abhinav Pandey, Abhishek Kadian, Ahmad Al-Dahle, Aiesha Letman, Akhil Mathur, Alan Schelten, Alex Vaughan, et~al.
\newblock The {Llama} 3 herd of models.
\newblock \emph{arXiv preprint}, abs/2407.21783, 2024.
\newblock URL \url{https://arxiv.org/abs/2407.21783}.

\bibitem[Hsieh et~al.(2024)Hsieh, Sun, Kriman, Acharya, Rekesh, Jia, and Ginsburg]{ruler}
Cheng-Ping Hsieh, Simeng Sun, Samuel Kriman, Shantanu Acharya, Dima Rekesh, Fei Jia, and Boris Ginsburg.
\newblock {RULER}: What{\textquoteright}s the real context size of your long-context language models?
\newblock In \emph{First Conference on Language Modeling}, 2024.
\newblock URL \url{https://openreview.net/forum?id=kIoBbc76Sy}.

\bibitem[Hu et~al.(2025)Hu, Liu, Dong, Peng, McDanel, and Zhang]{hu2025speculativedecodingbeyondindepth}
Yunhai Hu, Zining Liu, Zhenyuan Dong, Tianfan Peng, Bradley McDanel, and Sai~Qian Zhang.
\newblock Speculative decoding and beyond: An in-depth survey of techniques.
\newblock \emph{arXiv preprint}, abs/2502.19732, 2025.
\newblock URL \url{https://arxiv.org/abs/2502.19732}.

\bibitem[Izacard \& Grave(2021)Izacard and Grave]{izacard2021fid}
Gautier Izacard and Edouard Grave.
\newblock Leveraging passage retrieval with generative models for open domain question answering.
\newblock In \emph{Proceedings of the 16th Conference of the European Chapter of the Association for Computational Linguistics: Main Volume}, pp.\  874--880, Online, 2021. Association for Computational Linguistics.
\newblock URL \url{https://aclanthology.org/2021.eacl-main.74}.

\bibitem[Jiang et~al.(2023{\natexlab{a}})Jiang, Sablayrolles, Mensch, Bamford, Chaplot, de~las Casas, Bressand, Lengyel, Lample, Saulnier, Lavaud, Lachaux, Stock, Scao, Lavril, Wang, Lacroix, and Sayed]{mistral7b}
Albert~Q. Jiang, Alexandre Sablayrolles, Arthur Mensch, Chris Bamford, Devendra~Singh Chaplot, Diego de~las Casas, Florian Bressand, Gianna Lengyel, Guillaume Lample, Lucile Saulnier, Lélio~Renard Lavaud, Marie-Anne Lachaux, Pierre Stock, Teven~Le Scao, Thibaut Lavril, Thomas Wang, Timothée Lacroix, and William~El Sayed.
\newblock {Mistral} {7B}.
\newblock \emph{arXiv preprint}, abs/2310.06825, 2023{\natexlab{a}}.
\newblock URL \url{https://arxiv.org/abs/2310.06825}.

\bibitem[Jiang et~al.(2023{\natexlab{b}})Jiang, Wu, Lin, Yang, and Qiu]{llmlingua}
Huiqiang Jiang, Qianhui Wu, Chin-Yew Lin, Yuqing Yang, and Lili Qiu.
\newblock {LLML}ingua: Compressing prompts for accelerated inference of large language models.
\newblock In \emph{The 2023 Conference on Empirical Methods in Natural Language Processing}, 2023{\natexlab{b}}.
\newblock URL \url{https://openreview.net/forum?id=ADsEdyI32n}.

\bibitem[Jiang et~al.(2024{\natexlab{a}})Jiang, LI, Zhang, Wu, Luo, Ahn, Han, Abdi, Li, Lin, Yang, and Qiu]{minference}
Huiqiang Jiang, YUCHENG LI, Chengruidong Zhang, Qianhui Wu, Xufang Luo, Surin Ahn, Zhenhua Han, Amir~H. Abdi, Dongsheng Li, Chin-Yew Lin, Yuqing Yang, and Lili Qiu.
\newblock {MI}nference 1.0: Accelerating pre-filling for long-context {LLM}s via dynamic sparse attention.
\newblock In \emph{The Thirty-eighth Annual Conference on Neural Information Processing Systems}, 2024{\natexlab{a}}.
\newblock URL \url{https://openreview.net/forum?id=fPBACAbqSN}.

\bibitem[Jiang et~al.(2024{\natexlab{b}})Jiang, Wu, Luo, Li, Lin, Yang, and Qiu]{jiang2024longllmlingua}
Huiqiang Jiang, Qianhui Wu, Xufang Luo, Dongsheng Li, Chin-Yew Lin, Yuqing Yang, and Lili Qiu.
\newblock Long{LLM}lingua: Accelerating and enhancing {LLM}s in long context scenarios via prompt compression.
\newblock In \emph{Proceedings of the 62nd Annual Meeting of the Association for Computational Linguistics (Volume 1: Long Papers)}, pp.\  1658--1677, 2024{\natexlab{b}}.

\bibitem[Kwon et~al.(2023)Kwon, Li, Zhuang, Sheng, Zheng, Yu, Gonzalez, Zhang, and Stoica]{vllm}
Woosuk Kwon, Zhuohan Li, Siyuan Zhuang, Ying Sheng, Lianmin Zheng, Cody~Hao Yu, Joseph~E. Gonzalez, Hao Zhang, and Ion Stoica.
\newblock Efficient memory management for large language model serving with {PagedAttention}.
\newblock In \emph{Proceedings of the ACM SIGOPS 29th Symposium on Operating Systems Principles}, 2023.

\bibitem[Lai et~al.(2025)Lai, Lu, Luo, Ma, and Zhou]{flexprefill}
Xunhao Lai, Jianqiao Lu, Yao Luo, Yiyuan Ma, and Xun Zhou.
\newblock {FlexPrefill}: A context-aware sparse attention mechanism for efficient long-sequence inference.
\newblock In \emph{The Thirteenth International Conference on Learning Representations}, 2025.

\bibitem[Leviathan et~al.(2023)Leviathan, Kalman, and Matias]{specdecode}
Yaniv Leviathan, Matan Kalman, and Yossi Matias.
\newblock Fast inference from {Transformers} via speculative decoding.
\newblock In \emph{International Conference on Machine Learning}, pp.\  19274--19286. PMLR, 2023.

\bibitem[Li et~al.(2023)Li, Dong, Guerin, and Lin]{li2023compressing}
Yucheng Li, Bo~Dong, Frank Guerin, and Chenghua Lin.
\newblock Compressing context to enhance inference efficiency of large language models.
\newblock In \emph{The 2023 Conference on Empirical Methods in Natural Language Processing}, 2023.
\newblock URL \url{https://openreview.net/forum?id=cjbdRN8Yxy}.

\bibitem[Li et~al.(2024)Li, Huang, Yang, Venkitesh, Locatelli, Ye, Cai, Lewis, and Chen]{snapkv}
Yuhong Li, Yingbing Huang, Bowen Yang, Bharat Venkitesh, Acyr Locatelli, Hanchen Ye, Tianle Cai, Patrick Lewis, and Deming Chen.
\newblock {SnapKV}: {LLM} knows what you are looking for before generation.
\newblock In \emph{The Thirty-eighth Annual Conference on Neural Information Processing Systems}, 2024.
\newblock URL \url{https://openreview.net/forum?id=poE54GOq2l}.

\bibitem[Liskavets et~al.(2025)Liskavets, Ushakov, Roy, Klibanov, Etemad, and Luke]{liskavets2024cpc}
Barys Liskavets, Maxim Ushakov, Shuvendu Roy, Mark Klibanov, Ali Etemad, and Shane~K Luke.
\newblock Prompt compression with context-aware sentence encoding for fast and improved {LLM} inference.
\newblock In \emph{Proceedings of the AAAI Conference on Artificial Intelligence}, volume~39, pp.\  24595--24604, 2025.

\bibitem[Liu et~al.(2024{\natexlab{a}})Liu, Chen, Lu, Jiang, Han, Zhang, Chen, Zhang, Ding, Zhang, Chen, Yang, Yang, and Qiu]{retrievalattention}
Di~Liu, Meng Chen, Baotong Lu, Huiqiang Jiang, Zhenhua Han, Qianxi Zhang, Qi~Chen, Chengruidong Zhang, Bailu Ding, Kai Zhang, Chen Chen, Fan Yang, Yuqing Yang, and Lili Qiu.
\newblock {RetrievalAttention}: Accelerating long-context {LLM} inference via vector retrieval.
\newblock \emph{arXiv preprint}, abs/2409.10516, 2024{\natexlab{a}}.
\newblock URL \url{https://arxiv.org/abs/2409.10516}.

\bibitem[Liu et~al.(2020)]{liu2020learning}
Fei Liu et~al.
\newblock Learning to summarize from human feedback.
\newblock In \emph{Proceedings of the 58th Annual Meeting of the Association for Computational Linguistics}, pp.\  583--592, 2020.

\bibitem[Liu et~al.(2025)Liu, Chen, and Zhang]{specprefill}
Jingyu Liu, Beidi Chen, and Ce~Zhang.
\newblock Speculative prefill: Turbocharging {TTFT} with lightweight and training-free token importance estimation.
\newblock \emph{arXiv preprint}, abs/2502.02789, 2025.
\newblock URL \url{https://arxiv.org/abs/2502.02789}.

\bibitem[Liu et~al.(2024{\natexlab{b}})Liu, Lin, Hewitt, Paranjape, Bevilacqua, Petroni, and Liang]{liu2024lost}
Nelson~F Liu, Kevin Lin, John Hewitt, Ashwin Paranjape, Michele Bevilacqua, Fabio Petroni, and Percy Liang.
\newblock Lost in the middle: How language models use long contexts.
\newblock \emph{Transactions of the Association for Computational Linguistics}, 12:\penalty0 157--173, 2024{\natexlab{b}}.

\bibitem[Nawrot et~al.(2025)Nawrot, Li, Huang, Ruder, Marchisio, and Ponti]{sparse_frontier}
Piotr Nawrot, Robert Li, Renjie Huang, Sebastian Ruder, Kelly Marchisio, and Edoardo~M. Ponti.
\newblock The sparse frontier: Sparse attention trade-offs in {Transformer} {LLM}s.
\newblock \emph{arXiv preprint}, abs/2504.17768, 2025.
\newblock URL \url{https://arxiv.org/abs/2504.17768}.

\bibitem[Pan et~al.(2024)Pan, Wu, Jiang, Xia, Luo, Zhang, Lin, R{\"u}hle, Yang, Lin, et~al.]{pan2024llmlingua2}
Zhuoshi Pan, Qianhui Wu, Huiqiang Jiang, Menglin Xia, Xufang Luo, Jue Zhang, Qingwei Lin, Victor R{\"u}hle, Yuqing Yang, Chin-Yew Lin, et~al.
\newblock {LLM}lingua-2: Data distillation for efficient and faithful task-agnostic prompt compression.
\newblock In \emph{Findings of the Association for Computational Linguistics ACL 2024}, pp.\  963--981, 2024.

\bibitem[Paszke et~al.(2017)Paszke, Gross, Chintala, Chanan, Yang, DeVito, Lin, Desmaison, Antiga, and Lerer]{pytorch}
Adam Paszke, Sam Gross, Soumith Chintala, Gregory Chanan, Edward Yang, Zachary DeVito, Zeming Lin, Alban Desmaison, Luca Antiga, and Adam Lerer.
\newblock Automatic differentiation in {PyTorch}.
\newblock In \emph{NIPS-W}, 2017.

\bibitem[Raffel et~al.(2020)Raffel, Shazeer, Roberts, Lee, Narang, Matena, Zhou, Li, and Liu]{t5base}
Colin Raffel, Noam Shazeer, Adam Roberts, Katherine Lee, Sharan Narang, Michael Matena, Yanqi Zhou, Wei Li, and Peter~J Liu.
\newblock Exploring the limits of transfer learning with a unified text-to-text {Transformer}.
\newblock \emph{Journal of machine learning research}, 21\penalty0 (140):\penalty0 1--67, 2020.

\bibitem[Rajpurkar et~al.(2016)Rajpurkar, Zhang, Lopyrev, and Liang]{rajpurkar2016squad}
Pranav Rajpurkar, Jian Zhang, Konstantin Lopyrev, and Percy Liang.
\newblock {SQuAD}: 100,000+ questions for machine comprehension of text.
\newblock In \emph{Proceedings of the 2016 Conference on Empirical Methods in Natural Language Processing}, pp.\  2383--2392, Austin, Texas, 2016. Association for Computational Linguistics.
\newblock \doi{10.18653/v1/D16-1264}.
\newblock URL \url{https://aclanthology.org/D16-1264}.

\bibitem[Sadhukhan et~al.(2025)Sadhukhan, Chen, Chen, Tiwari, Lai, Shi, Yen, May, Chen, and Chen]{magicdec}
Ranajoy Sadhukhan, Jian Chen, Zhuoming Chen, Vashisth Tiwari, Ruihang Lai, Jinyuan Shi, Ian En-Hsu Yen, Avner May, Tianqi Chen, and Beidi Chen.
\newblock {MagicDec}: Breaking the latency-throughput tradeoff for long context generation with speculative decoding.
\newblock In \emph{The Thirteenth International Conference on Learning Representations}, 2025.
\newblock URL \url{https://openreview.net/forum?id=CS2JWaziYr}.

\bibitem[Sun et~al.(2024)Sun, Chen, Yang, Tian, and Chen]{triforce}
Hanshi Sun, Zhuoming Chen, Xinyu Yang, Yuandong Tian, and Beidi Chen.
\newblock {TriForce}: Lossless acceleration of long sequence generation with hierarchical speculative decoding.
\newblock In \emph{First Conference on Language Modeling}, 2024.
\newblock URL \url{https://openreview.net/forum?id=HVK6nl3i97}.

\bibitem[Tang et~al.(2024)Tang, Zhao, Zhu, Xiao, Kasikci, and Han]{quest}
Jiaming Tang, Yilong Zhao, Kan Zhu, Guangxuan Xiao, Baris Kasikci, and Song Han.
\newblock {QUEST}: Query-aware sparsity for efficient long-context {LLM} inference.
\newblock In \emph{Forty-first International Conference on Machine Learning}, 2024.
\newblock URL \url{https://openreview.net/forum?id=KzACYw0MTV}.

\bibitem[Tiwari et~al.(2025)Tiwari, Xi, Tomar, Hooper, Kim, Horton, Najibi, Mahoney, Keutzer, and Gholami]{quantspec}
Rishabh Tiwari, Haocheng Xi, Aditya Tomar, Coleman Hooper, Sehoon Kim, Maxwell Horton, Mahyar Najibi, Michael~W. Mahoney, Kurt Keutzer, and Amir Gholami.
\newblock {QuantSpec}: Self-speculative decoding with hierarchical quantized {KV} cache.
\newblock \emph{arXiv preprint}, abs/2502.10424, 2025.
\newblock URL \url{https://arxiv.org/abs/2502.10424}.

\bibitem[Vaswani et~al.(2017)Vaswani, Shazeer, Parmar, Uszkoreit, Jones, Gomez, Kaiser, and Polosukhin]{transformer}
Ashish Vaswani, Noam Shazeer, Niki Parmar, Jakob Uszkoreit, Llion Jones, Aidan~N. Gomez, Lukasz Kaiser, and Illia Polosukhin.
\newblock Attention is all you need.
\newblock In Isabelle Guyon, Ulrike von Luxburg, Samy Bengio, Hanna~M. Wallach, Rob Fergus, S.~V.~N. Vishwanathan, and Roman Garnett (eds.), \emph{Advances in Neural Information Processing Systems 30: Annual Conference on Neural Information Processing Systems 2017, December 4-9, 2017, Long Beach, CA, {USA}}, pp.\  5998--6008, 2017.
\newblock URL \url{https://proceedings.neurips.cc/paper/2017/hash/3f5ee243547dee91fbd053c1c4a845aa-Abstract.html}.

\bibitem[Wang et~al.(2025)Wang, Ji, Liu, Xu, Xu, Zhu, and Che]{laq}
Yixuan Wang, Shiyu Ji, Yijun Liu, Yuzhuang Xu, Yang Xu, Qingfu Zhu, and Wanxiang Che.
\newblock Lookahead {Q}-cache: Achieving more consistent {KV} cache eviction via pseudo query.
\newblock In Christos Christodoulopoulos, Tanmoy Chakraborty, Carolyn Rose, and Violet Peng (eds.), \emph{Proceedings of the 2025 Conference on Empirical Methods in Natural Language Processing}, pp.\  34158--34174, Suzhou, China, November 2025. Association for Computational Linguistics.
\newblock ISBN 979-8-89176-332-6.
\newblock \doi{10.18653/v1/2025.emnlp-main.1732}.
\newblock URL \url{https://aclanthology.org/2025.emnlp-main.1732/}.

\bibitem[Wolf et~al.(2019)Wolf, Debut, Sanh, Chaumond, Delangue, Moi, Cistac, Rault, Louf, Funtowicz, Davison, Shleifer, von Platen, Ma, Jernite, Plu, Xu, Scao, Gugger, Drame, Lhoest, and Rush]{huggingface}
Thomas Wolf, Lysandre Debut, Victor Sanh, Julien Chaumond, Clement Delangue, Anthony Moi, Pierric Cistac, Tim Rault, Rémi Louf, Morgan Funtowicz, Joe Davison, Sam Shleifer, Patrick von Platen, Clara Ma, Yacine Jernite, Julien Plu, Canwen Xu, Teven~Le Scao, Sylvain Gugger, Mariama Drame, Quentin Lhoest, and Alexander~M. Rush.
\newblock {HuggingFace}'s {Transformers}: State-of-the-art natural language processing.
\newblock \emph{arXiv preprint}, abs/1910.03771, 2019.
\newblock URL \url{https://arxiv.org/abs/1910.03771}.

\bibitem[Xiao et~al.(2024)Xiao, Tian, Chen, Han, and Lewis]{streamingllm}
Guangxuan Xiao, Yuandong Tian, Beidi Chen, Song Han, and Mike Lewis.
\newblock Efficient streaming language models with attention sinks.
\newblock In \emph{The Twelfth International Conference on Learning Representations}, 2024.
\newblock URL \url{https://openreview.net/forum?id=NG7sS51zVF}.

\bibitem[Yang et~al.(2024)Yang, Yang, Zhang, Hui, Zheng, Yu, Li, Liu, Huang, Wei, et~al.]{qwen2.5}
An~Yang, Baosong Yang, Beichen Zhang, Binyuan Hui, Bo~Zheng, Bowen Yu, Chengyuan Li, Dayiheng Liu, Fei Huang, Haoran Wei, et~al.
\newblock {Qwen2.5} technical report.
\newblock \emph{arXiv preprint}, abs/2412.15115, 2024.
\newblock URL \url{https://arxiv.org/abs/2412.15115}.

\bibitem[Yang et~al.(2025{\natexlab{a}})Yang, Li, Yang, Zhang, Hui, Zheng, Yu, Gao, Huang, Lv, et~al.]{qwen3}
An~Yang, Anfeng Li, Baosong Yang, Beichen Zhang, Binyuan Hui, Bo~Zheng, Bowen Yu, Chang Gao, Chengen Huang, Chenxu Lv, et~al.
\newblock {Qwen3} technical report.
\newblock \emph{arXiv preprint arXiv:2505.09388}, 2025{\natexlab{a}}.

\bibitem[Yang et~al.(2025{\natexlab{b}})Yang, Yu, Li, Liu, Huang, Huang, Jiang, Tu, Zhang, Zhou, et~al.]{qwen2.5-1m}
An~Yang, Bowen Yu, Chengyuan Li, Dayiheng Liu, Fei Huang, Haoyan Huang, Jiandong Jiang, Jianhong Tu, Jianwei Zhang, Jingren Zhou, et~al.
\newblock {Qwen2.5-1M} technical report.
\newblock \emph{arXiv preprint}, abs/2501.15383, 2025{\natexlab{b}}.
\newblock URL \url{https://arxiv.org/abs/2501.15383}.

\bibitem[Yang et~al.(2018)Yang, Qi, Zhang, Bengio, Cohen, Salakhutdinov, and Manning]{yang2018hotpotqa}
Zhilin Yang, Peng Qi, Saizheng Zhang, Yoshua Bengio, William Cohen, Ruslan Salakhutdinov, and Christopher~D. Manning.
\newblock {HotpotQA}: A dataset for diverse, explainable multi-hop question answering.
\newblock In \emph{Proceedings of the 2018 Conference on Empirical Methods in Natural Language Processing}, pp.\  2369--2380, Brussels, Belgium, 2018. Association for Computational Linguistics.
\newblock \doi{10.18653/v1/D18-1259}.
\newblock URL \url{https://aclanthology.org/D18-1259}.

\bibitem[Zhang et~al.(2023)Zhang, Sheng, Zhou, Chen, Zheng, Cai, Song, Tian, Re, Barrett, Wang, and Chen]{h2o}
Zhenyu Zhang, Ying Sheng, Tianyi Zhou, Tianlong Chen, Lianmin Zheng, Ruisi Cai, Zhao Song, Yuandong Tian, Christopher Re, Clark Barrett, Zhangyang Wang, and Beidi Chen.
\newblock {H2O}: Heavy-hitter oracle for efficient generative inference of large language models.
\newblock In \emph{Thirty-seventh Conference on Neural Information Processing Systems}, 2023.
\newblock URL \url{https://openreview.net/forum?id=RkRrPp7GKO}.

\end{thebibliography}
\bibliographystyle{iclr2026_conference}
\clearpage
\appendix

{\LARGE \textbf{Appendix}} \par 
\startcontents[sections]
\printcontents[sections]{ }{1}{}

\clearpage

\section{Algorithm Pseudocode}
\label{app:algo}

\begin{algorithm} 
\setstretch{1.2} 
\caption{\speckv{}}
\label{alg:speckv}
\footnotesize
\begin{algorithmic}[1]
\State \textbf{Input:}
    \Statex \hspace{\algorithmicindent} Input sequence $x$ with length $\ninput$
    \Statex \hspace{\algorithmicindent} Parameters: Number of lookahead tokens $\nlookahead$, Maximum cache capacity $\cmax$, 
    \Statex \hspace{\algorithmicindent} Compression window size $\ncompwnd$, Kernel size $\poolk$, Prefill window size $\nprewnd$, 
    \Statex \hspace{\algorithmicindent}  Number of global tokens in prefill $\nglobal$ \hspace{\algorithmicindent} 

\vspace{5pt}

\State Generate a draft output $y_{\text{draft}}$ of length $\nlookahead$ using the draft model.

\State Forward $x$ and $y_{\text{draft}}$ through the target model.

\For{each attention head in target model}
    \State $X \gets \text{Target model hidden states from prompt at current layer}$
    \Comment{$X \in \mathbb R^{\ninput \times d}$}
    
    \State $Y \gets \text{Target model hidden states from draft output at current layer}$
    \Comment{$Y \in \mathbb R^{\nlookahead \times d}$}

    \State $m \gets \ninput - \ncompwnd$

    \State $A \gets \operatorname{CrossAttention}\left(Q \text{ from } \left[\begin{smallmatrix}
        X_{m:}\\ Y
    \end{smallmatrix}\right],\; K/V \text{ from } X_{:m}\right)$ \Comment{Compute attention}


    \State $s \gets \operatorname{MaxReduce}(A)$ \Comment{$\mathbb R^{n_1 \times n_2} \to \mathbb R^{n_2}$}

    \State $s \gets \operatorname{AvgPool1D}(s, \poolk)$ \Comment{Smooth attention}

    \State $i_{\text{vert}} \gets \operatorname{topk}(s, \nglobal)$ \Comment{Select global tokens (Sparse prefill)}
    \State $i_{\text{slash}} \gets \{1, 2, \ldots, \nprewnd\}$ \Comment{Sliding window (Sparse prefill)}
    \State $i_{\text{cache}} \gets \operatorname{topk}(s, \cmax - \ncompwnd)
        \cup \{m + 1, m + 2, \ldots, \ninput\}$ \Comment{Select KVs}

    \State $\text{output} \gets \operatorname{VerticalSlash}(X, i_{\text{vert}}, i_{\text{slash}})$ \Comment{Sparse prefill attention}
    
    \State $\text{cache} \gets K_{i_{\text{cache}}}, V_{i_{\text{cache}}}$ \Comment{KV cache dropping}

\EndFor

\end{algorithmic}
\end{algorithm}

\newlength\myindent
\setlength\myindent{2em}
\newcommand\bindent{%
  \begingroup
  \setlength{\itemindent}{\myindent}
  \addtolength{\algorithmicindent}{\myindent}
}

\begin{algorithm} 
\setstretch{1.2} 
\caption{\specpc{}}
\label{alg:spc}
\footnotesize
\begin{algorithmic}[1]
\State \textbf{Input:}
    \Statex \hspace{\algorithmicindent} Draft attention tensor $A \in \mathbb{R}^{n_\text{layer} \times n_\text{head} \times \left(\ninput+\nlookahead-1\right) \times \ninput}$
    \Statex \hspace{\algorithmicindent} Parameters: Window size $\ncompwnd$,  Kernel size $\poolk$, 
    Number of neighbors $\nneigh$, 
    \Statex \hspace{\algorithmicindent} Number of selected tokens $\cmax$,  Number of skipped layers $\lskip$
\vspace{5pt}

\State $m \gets \ninput - \ncompwnd$

\State $A \gets A_{l_{\text{skip}}:, :, m:, :m}$ \Comment{Skip layers and only consider window queries and non-window keys}
\For{$j \in \{1, 2, \ldots, \ncompwnd\}$}
    \State $A_{\ldots, j, :} \gets \frac{j}{\ncompwnd}A_{\ldots, j, :}$ \Comment{Assign more weight to later tokens}
\EndFor

\State $s \gets \operatorname{MaxReduce}(A)$  \Comment{$\mathbb{R}^{n_1 \times n_2 \times n_3 \times n_4} \to \mathbb R^{n_4}$}  

\State $s \gets \operatorname{AvgPool1D}(s, \poolk)$ \Comment{Smooth attention}
\State $s \gets \operatorname{MaxPool1D}(s, n_\text{neighbor})$ \Comment{Keep neighbor tokens}

\State $i_{\text{selected}} \gets \operatorname{topk}(s, C_\text{max}) \cup \{m + 1, m + 2, \ldots, \ninput\}$ \Comment{Keep most activated tokens and window tokens}
\State \Return $i_{\text{selected}}$ 
\end{algorithmic}
\end{algorithm}

\section{Mathematical Proofs}
\label{app:proofs}

\begin{lemma} \label{softmax_lemma}
    $\|\operatorname{Softmax}(x) - \operatorname{Softmax}(y)\|_2 \le \|x - y\|_\infty$.
\end{lemma}

\begin{proof}
     Let $J$ be the Jacobian matrix of $\operatorname{Softmax}$. 
     Then, $J(v) = \operatorname{diag}(p) - pp^T$, where $p = \operatorname{Softmax}(v)$.
     Note that $p$ is a probability distribution, so $p_i \ge 0$ for all $i$ and $\sum_{i}p_i = 1$.
     For any vectors $v$ and $z$,
    \begin{eqnarray*}
        \|J(v)z\|_2^2 
        &=& \|(\operatorname{diag}(p) - pp^T)z\|_2^2\\
        &=&\sum_i (p_iz_i - p_ip^Tz)^2\\
        &=&\sum_i p_i^2(z_i - p^Tz)^2\\
        &\le&\sum_i p_i(z_i - p^Tz)^2\\
        &=& \sum_i \left(p_iz_i^2 - 2p_iz_ip^Tz + p_i(p^Tz)^2\right)\\
        &=& \sum_i p_iz_i^2 - (p^Tz)^2\\
        &\le& \sum_i p_iz_i^2\\
        &\le& \sum_i p_i\|z\|_\infty^2\\
        &\le& \|z\|_\infty^2.
    \end{eqnarray*}
    Thus, $\|J(v)z\|_2 \le \|z\|_\infty$ for all $v$ and $z$.
    From the fundamental theorem of line integrals,
    \begin{equation}
        \operatorname{Softmax}(x) - \operatorname{Softmax}(y) 
        = \int_0^1 J(y + t(x - y))(x - y)dt.
    \end{equation}
    Finally,
    \begin{eqnarray*}
        \|\operatorname{Softmax}(x) - \operatorname{Softmax}(y) \|_2
        &=& \left\|\int_0^1 J(y + t(x - y))(x - y)dt\right\|_2\\
        &\le& \int_0^1 \|J(y + t(x - y))(x - y)\|_2dt\\
        &\le& \int_0^1 \|x - y\|_\infty dt\\
        &=& \|x - y\|_\infty.
    \end{eqnarray*}

\end{proof}



\begin{lemma} \label{inverse_softmax_lemma}
    Let $y = \operatorname{Softmax}(x)$ and $y' = \operatorname{Softmax}(x')$.
    If $\|y - y'\|_p \le \epsilon$, then there exists a scalar $c$ such that $\|x - x' - c\mathbf 1\|_p \le \tfrac{\epsilon}{m}$, where $m = \min_i(\min(y_i, y_i'))$ and $p \in \{1, 2, \ldots, \infty\}$.
\end{lemma}

\begin{proof}
    From the mean value theorem, there exists $\xi \in (y_i, y_i')$ such that
    \begin{equation}
        \tfrac{\log y_i - \log y_i'}{y_i - y_i'} = \tfrac{d\log t}{dt}\Bigg|_{t=\xi} = \tfrac{1}{\xi}.
    \end{equation}
    Note that $\xi > 0$. Then, 
    \begin{equation}
        |\log y_i - \log y_i'| = \tfrac{1}{\xi}|y_i - y_i'| \le \tfrac{1}{m}|y_i - y_i'|.
    \end{equation}

    Let $c = \log \sum_je^{x_j} - \log \sum_je^{x_j'}$, so
    \begin{equation}
        \left|\log \tfrac{e^{x_i}}{\sum_je^{x_j}} - \log \tfrac{e^{x_i'}}{\sum_je^{x_j'}}\right|
            = \left|x_i - x_i' - \left(\log \sum_je^{x_j} - \log \sum_je^{x_j'}\right)\right| 
            = |x_i - x_i' - c|
    \end{equation}
    for all $i$.
    Thus,
    \begin{align*}
        |x_i - x_i' - c| \le \tfrac{1}{m}|y_i - y_i'|
        &\implies |x_i - x_i' - c|^p \le \tfrac{1}{m^p}|y_i - y_i'|^p\\
        &\implies \sum_i|x_i - x_i' - c|^p \le \tfrac{1}{m^p}\sum_i|y_i - y_i'|^p\\
        &\implies \|x - x' - c\mathbf 1\|_p^p \le \tfrac{1}{m^p}\|y - y'\|_p^p\\
        &\implies \|x - x' - c\mathbf 1\|_p \le \tfrac{1}{m}\|y - y'\|_p\\
        &\implies \|x - x' - c\mathbf 1\|_p \le \tfrac{\epsilon}{m}.
    \end{align*}
\end{proof}

\subsection{Proof of \cref{theorem:speckv}}\label{proof:speckv}

We define the vector of importance scores as and its approximation as
\begin{equation}
s^T = \tfrac{1}{\noutput} \sum_{i=1}^\noutput \operatorname{Softmax}\left( \tfrac{x_i^{(o)T} W_q W_k^T X^T}{\sqrt d} \right)
\quad \text{and} \quad
\hat{s}^T = \tfrac{1}{\noutput}\sum_{i=1}^\noutput\operatorname{Softmax}\left( \tfrac{\hat{x}_i^{(o)T} W_q W_k^T X^T}{\sqrt d}\right),
\end{equation}

where $X = [x_1, \dots, x_\ninput]^T \in \mathbb{R}^{\ninput \times d}$ is the matrix of input embeddings, $x_i^{(o)} \in \mathbb{R}^d$ is the $i$th output embedding, and $\hat{x}_i^{(o)} \in \mathbb{R}^d$ is the $i$th approximate output embedding (from the draft model). $s_i$ and $\hat s_i$ denote the importance of the $i$th KV pair. In practice, \speckv{} estimates importance using queries from recent input and draft output tokens. This is omitted from the theoretical analysis for clarity.

\begin{proof}
    We assume $\|x_i^{(o)} - \hat x_i^{(o)}\|_2 \le \epsilon$ for all $i$ and $\|x_j\|_2 \le \sqrt d$ for all $j$.
    
    $\|x_i^{(o)} - \hat x_i^{(o)}\|_2 \le \epsilon$, so 
    \begin{eqnarray*}
        \left|\tfrac{x_i^{(o)T} W_q W_k^T x_j}{\sqrt d} - \tfrac{\hat x_i^{(o)T} W_q W_k^T x_j}{\sqrt d} \right|
        &=& \tfrac{1}{\sqrt d}\left|(x_i^{(o)} - \hat x_i^{(o)})^T W_q W_k^T x_j\right| \\
        &\le& \tfrac{1}{\sqrt d}\|W_q W_k^T\|_2 \epsilon\sqrt{d} \\
        &=& \epsilon \|W_q W_k^T\|_2.
    \end{eqnarray*}
    Thus, \begin{equation}
        \left\|
            \tfrac{x_i^{(o)T} W_q W_k^T X^T}{\sqrt d} - \tfrac{\hat x_i^{(o)T} W_q W_k^T X^T}{\sqrt d}
        \right\|_\infty 
        \le \epsilon \|W_q W_k^T\|_2.
    \end{equation}
    
    Applying \cref{softmax_lemma} and the triangle inequality, we get
    \begin{eqnarray*}
        \|s - \hat s\|_2
        &=& \left\|
             \tfrac{1}{\noutput} \sum_{i=1}^\noutput \operatorname{Softmax}\left(\tfrac{x_i^{(o)T} W_q W_k^T X^T}{\sqrt d}\right) - \tfrac{1}{\noutput} \sum_{i=1}^\noutput \operatorname{Softmax}\left(\tfrac{\hat x_i^{(o)T} W_q W_k^T X^T}{\sqrt d}\right)
        \right\|_2 \\
        &\le& \tfrac{1}{\noutput}\sum_{i=1}^\noutput\left\|
            \operatorname{Softmax}\left(\tfrac{x_i^{(o)T} W_q W_k^T X^T}{\sqrt d}\right) - \operatorname{Softmax}\left(\tfrac{\hat x_i^{(o)T} W_q W_k^T X^T}{\sqrt d}\right)
        \right\|_2 \\
        &\le& 
        \tfrac{1}{\noutput}\sum_{i=1}^\noutput\left\|
            \tfrac{x_i^{(o)T} W_q W_k^T X^T}{\sqrt d} - \tfrac{\hat x_i^{(o)T} W_q W_k^T X^T}{\sqrt d}
        \right\|_\infty \\
        &\le& \epsilon \|W_q W_k^T\|_2,
    \end{eqnarray*}
\end{proof}

\subsection{Proof of \cref{theorem:specpc_rip2}} \label{proof:rip}

\begin{equation}X = [x_1, \ldots, x_\ninput]^T\end{equation}
\begin{equation}Y = [y_1, \ldots, y_\ninput]^T = \operatorname{Softmax}\left(\tfrac{X W_q W_k^T X^T}{\sqrt{d}}\right) X W_v\end{equation}
\begin{equation}\hat Y = [\hat y_1, \ldots, \hat y_\ninput]^T = \operatorname{Softmax}\left(\tfrac{X \hat W_q \hat W_k^T X^T}{\sqrt{d}}\right) X \hat W_v\end{equation}
\begin{equation}A = [a_1, \ldots, a_\ninput]^T = \operatorname{Softmax}\left(\tfrac{X W_q W_k^T X^T}{\sqrt{d}}\right)\end{equation}
\begin{equation}\hat A = [\hat a_1, \ldots, \hat a_\ninput]^T = \operatorname{Softmax}\left(\tfrac{X \hat W_q \hat W_k^T X^T}{\sqrt{d}}\right)\end{equation}

\begin{proof}

We assume $\|y_i - \hat y_i\|_2 \le \epsilon\|X\|_{\infty, 2}$, where $\|X\|_{\infty, 2}$ is the maximum $\ell_2$ norm of the rows of $X$. Additionally, we assume that there exists a constant $c$ such that $cX^T$ has the Restricted Isometry Property~\citep{rip} with parameters $(2k, \delta)$, where $\delta$ is the restricted isometry constant and $k$ is the approximate sparsity of $a_i$ and $\hat a_i$.

Recall that a matrix $B$ satisfies the Restricted Isometry Property with constant $\delta \in (0,1)$ if for every $k$-sparse vector $v$, the following inequality holds:

\begin{equation}
    (1 - \delta)\|v\|_2^2 \le \|Bv\|_2^2 \le (1 + \delta)\|v\|_2^2.
\end{equation}

Let $\Delta W_v = W_v - \hat W_v$ and $\Delta a_i = a_i - \hat a_i$.

If $\ninput = 1$, then $A = \hat A \in \mathbb R^{1 \times 1}$ with $A_{1, 1} = \hat A_{1, 1} = 1$, so $AX = \hat{A}X = X = x_1^T$,
which implies 
\begin{equation}\|y_1 - \hat y_1\|_2 
= \|x_1^T W_V - x_1^T \hat W_v\|_2 
= \|x_1^T \Delta W_v\|_2 
\le \epsilon\|X\|_{\infty, 2} = \epsilon\|x_1\|_2\end{equation}
for all $x_1$. 
$\|x_1^TW_v\|_2 \le \epsilon \|x_1\|$ for all $x_1$ is the definition of the matrix $\ell_2$ norm, so $\|\Delta W_v\|_2  \le \epsilon$.

\begin{eqnarray*}
\|y_i - \hat y_i\|_2 
&=& \|a_i^T X W_v - \hat a_i^T X \hat W_v\|_2  \\
&=& \|a_i^T X W_v - (\hat a_i^T X W_v - \hat a_i^T X \Delta W_v)\|_2  \\
&=& \|a_i^T X W_v - (a_i^T X W_v - \Delta a_i^T X W_v  - \hat a_i^T X \Delta W_v)\|_2  \\
&=& \|\Delta a_i^T X W_v  + \hat a_i^T X \Delta W_v\|_2  \\
&\le& \epsilon\|X\|_{\infty, 2}
\end{eqnarray*}
Then, $\|\Delta a_i^T X W_v\|_2 \le \epsilon\|X\|_{\infty, 2} + \|\hat a_i^T X \Delta W_v\|_2$.

Since $\hat a_i^TX$ is a convex combination of the rows of $X$, $\|\hat a_i^T X\|_2 \le \|X\|_{\infty, 2}$.

Thus, $\|\Delta a_i^T X W_v\|_2 
\le \epsilon\|X\|_{\infty, 2} + \|\hat a_i^T X \Delta W_v\|_2 
\le \epsilon\|X\|_{\infty, 2} + \|\hat a_i^T X\|_2 \|\Delta W_v\|_2 
\le 2\epsilon\|X\|_{\infty, 2}$.

Attention scores are approximately sparse \citep{minference}, especially for long sequences. Therefore, we assume $a_i$ and $\hat a_i$ are $k$-sparse. Then, $\Delta a_i$ is at most $2k$-sparse. Since $cX^T$ has the Restricted Isometry Property with parameters $2k$, $\delta$,
\begin{equation}(1-\delta)\|\Delta a_i\|_2 \le \|\Delta a_i^T (cX)\|_2 \le (1+\delta)\|\Delta a_i\|_2.\end{equation}
Then, 
\begin{equation}\tfrac{1}{c}(1-\delta)\|\Delta a_i\|_2 \le \|\Delta a_i^T X\|_2 \le \tfrac{2\epsilon\|X\|_{\infty, 2}}{\sigma_{\min}(W_v)},\end{equation}
so
\begin{equation}\|\Delta a_i\|_2 \le \tfrac{2c\epsilon\|X\|_{\infty, 2}}{\sigma_{\min}(W_v)(1-\delta)}.\end{equation}
\end{proof}







\subsection{Proof of \cref{theorem:specpc}}\label{proof:specpc}

\begin{theorem}\label{theorem:specpc}
    If $\|Y - \hat Y\|_2 \le \epsilon\|X\|_2$ for all $X$ and the column space of $W_q, W_k, \hat W_q, \hat W_k$ is a subset of the column space of $W_v$, then $\|a_i - \hat a_i\|_2 \le \epsilon\delta$, where
    \begin{equation}
        \delta = 2d\tfrac{\sigma_{\max}(W_v)^2}{\sigma_{\min}(W_v)}\exp\left(2\max\left(
            \tfrac{\|W_q\|_2 \|W_k\|_2}{\sigma_{\min}(W_v)^2}, 
            \tfrac{\|\hat W_q\|_2 \|\hat W_k\|_2}{\sigma_{\min}(\hat W_v)^2}
        \right)\right)\|X\|_{\infty, 2}^2.
    \end{equation}
\end{theorem}

\begin{proof}
We assume $\|Y - \hat Y\|_2 \le \epsilon\|X\|_2$. Additionally, we assume that the column space of $W_q, W_k, \hat W_q, \hat W_k$ is a subset of the column space of $W_v$. To get a norm bound on $\Delta a_i = a_i - \hat a_i$, we will bound the norms of the error in approximate weight matrices. We will find these bounds by using specific inputs, taking advantage of the fact that $\|Y - \hat Y\|_2 \le \epsilon\|X\|_2$ for all $X$.

We will start by bounding $\Delta W_v = W_v - \hat W_v$, by choosing an input that fixes $A$ and $\hat A$. If $n = 1$, then $A = \hat A = [1]$, so $AX = \hat{A}X = X = x_1^T$,
which implies 
\begin{equation}\|Y - \hat Y\|_2 
= \|A X W_V - \hat A X \hat W_v\|_2  
= \|x_1^T W_V - x_1^T \hat W_v\|_2 
= \|x_1^T \Delta W_v\|_2 
\le \epsilon\|X\|_2 = \epsilon\|x_1\|_2\end{equation}
for all $x_1$. 
Thus, $\|\Delta W_v\|_2  \le \epsilon$.

Next, we will bound the norm of $\Delta B = B - \hat B$, where $B = W_qW_k^T$ and $\hat B = \hat W_q \hat W_k^T$. We will choose the $X$ so that the values are the identity matrix. Then $Y = A$.
Let $U\Sigma V^T$ be the singular value decomposition of $W_v$. We set 

\begin{equation}X = \Phi V\Sigma^{-1}U^T,\end{equation}

where $\Phi$ is an arbitrary orthonormal basis spanning $\mathbb R^{d \times d}$. 

Note that $\sigma_{\min}(\Phi) = \sigma_{\max}(\Phi) = 1$, so
$\|X\|_2 = \sigma_{\max}(X) = \tfrac{1}{\sigma_{\min}(W_v)}$ and
$\sigma_{\min}(X) = \tfrac{1}{\sigma_{\max}(W_v)}$.

Now, 
\begin{eqnarray*}
    \|Y - \hat Y\|_2 &=& \|A X W_V - \hat A X \hat W_v\|_2\\
    &=& \|A X W_V - (\hat A X W_v - \hat A X \Delta W_v)\|_2  \\
    &=& \|A X W_V - (A X W_v - \Delta A X W_v - \hat A X \Delta W_v)\|_2  \\
    &=& \|\Delta A X W_v + \hat A X \Delta W_v\|_2  \\
    &\le& \|\Delta A X W_v\|_2 + \|\hat A X \Delta W_v\|_2 \\
    &=& \|\Delta A \Phi V\Sigma^{-1}U^T U\Sigma V^T\|_2 + \|\hat A X \Delta W_v\|_2 \\
    &=& \|\Delta A \Phi\|_2 + \|\hat A X \Delta W_v\|_2 \\
    &=& \|\Delta A\|_2 + \|\hat A X \Delta W_v\|_2 \\
    &\le& \epsilon\|X\|_2 + \epsilon\|X\|_2\\
    &\le& \tfrac{2\epsilon}{\sigma_{\min}(W_v)}.
\end{eqnarray*}

Note that each row of $\hat A$ is a probability distribution (non-negative entries summing to 1), so left-multiplying $X$ by $\hat A$ forms a convex combination of the rows of $X$. From Jensen's inequality we get $\|\hat AX\|_2 \leq \|X\|_2$, because $x \rightarrow \|x\|_2$ is a convex function.

Let $\delta_1 = \max\left(
    \tfrac{\|W_q\|_2 \|W_k\|_2}{\sigma_{\min}(W_v)^2}, \tfrac{\|\hat W_q\|_2 \|\hat W_k\|_2}{\sigma_{\min}(\hat W_v)^2}
\right)$.
Since $\|X\|_2 = \tfrac{1}{\sigma_{\min}(W_v)}$ and $\|B\|_2 \le \|W_q\|_2\|W_k\|_2$, $\|X B X^T\|_2 \le \delta_1$.
Consequently, $|x_i^T B x_j| \le \delta_1$ for all $i, j$. This implies that each attention weight satisfies 

\begin{equation}a_{i,j} > \tfrac{e^{-\delta_1}}{\sum_{j=1}^d e^{\delta_1}} = \tfrac{1}{d}e^{-2\delta_1}.\end{equation}

The same argument applied to $\hat a$ gives

\begin{equation}\hat a_{i,j} > \tfrac{e^{-\delta_1}}{\sum_{j=1}^d e^{\delta_1}} = \tfrac{1}{d}e^{-2\delta_1}.\end{equation}

Applying \cref{inverse_softmax_lemma} to each row, there exists $c \in \mathbb{R}^d$ such that  

\begin{equation}\begin{aligned}
\left\|\tfrac{X B X^T}{\sqrt{d}} - \tfrac{X \hat{B} X^T}{\sqrt{d}} + c\mathbf{1}^T\right\|_2 &\le d e^{2\delta_1} \left\| \operatorname{Softmax}\left( \tfrac{X B X^T}{\sqrt{d}} \right) - \operatorname{Softmax}\left( \tfrac{X \hat{B} X^T}{\sqrt{d}} \right) \right\|_2 \\
\left\|X B X^T - X \hat{B} X^T + \sqrt d c\mathbf{1}^T\right\|_2 &\le d^{3/2} e^{2\delta_1} \|\Delta A\|_2\\
&\le \tfrac{2\epsilon d^{3/2} e^{2\delta_1}}{\sigma_{\min}(W_v)}.
\end{aligned}\end{equation}

Minimizing over $c$, we obtain  

\begin{equation}
\begin{aligned}
\min_c \|X B X^T - X \hat{B} X^T + \sqrt dc\mathbf{1}^T\|_2 &= \min_c \|X \Delta B X^T - \sqrt dc\mathbf{1}^T\|_2 \\
&= \left\|X \Delta B X^T - \tfrac{1}{d} X \Delta B X^T \mathbf{1} \mathbf{1}^T\right\|_2 \\
&= \left\|X \Delta B X^T \left(I - \tfrac{1}{d} \mathbf{1} \mathbf{1}^T\right)\right\|_2 \\
\end{aligned}
\end{equation}

Substituting in the definition of $X$, we get

\begin{equation}\|\Phi V\Sigma^{-1}U^T \Delta B U\Sigma^{-1}V^T\Phi^T(I - d^{-1}\mathbf 1 \mathbf 1^T)\|_2 \le \tfrac{\epsilon d^{3/2}e^{2\delta_1}}{\sigma_{\min}(W_v)}.\end{equation}

Each multiplication by $\Sigma^{-1}$ can decrease the norm by at most $\tfrac{1}{\sigma_{\max}(W_v)}$, so when removing both instances of $\Sigma^{-1}$ we scale the bound by $\sigma_{\max}(W_v)^2$, giving us

\begin{equation}\|\Phi VU^T \Delta B UV^T\Phi^T(I - d^{-1}\mathbf 1 \mathbf 1^T)\|_2 \le \tfrac{2\epsilon d^{3/2}e^{2\delta_1}}{\sigma_{\min}(W_v)}\sigma_{\max}(W_v)^2.\end{equation}

Then, since $\Phi$ and $V$ are orthonormal and preserve spectral norm under multiplication, we conclude

\begin{equation}\|U^T \Delta B UV^T\Phi^T(I -d^{-1}\mathbf 1 \mathbf 1^T)\|_2 \le \tfrac{2\epsilon d^{3/2}e^{2\delta_1}}{\sigma_{\min}(W_v)}\sigma_{\max}(W_v)^2.\end{equation}

Finally, since the column space of $W_q, W_k, \hat W_q, \hat W_k$ is a subset of the column space of $W_v$, the column space of $\Delta B$ is a subset of the column space of $U$. Thus, left multiplication by $U^T$ does not impact the spectral norm, so

\begin{equation}\|\Delta B UV^T\Phi^T(I -d^{-1}\mathbf 1 \mathbf 1^T)\|_2 \le \tfrac{2\epsilon d^{3/2}e^{2\delta_1}}{\sigma_{\min}(W_v)}\sigma_{\max}(W_v)^2.\end{equation}

Note that the matrix $I - \tfrac{1}{d} \mathbf{1}\mathbf{1}^T$ is a projection onto the subspace orthogonal to the all-ones vector. Its singular values are $[1, \ldots, 1, 0]$, so its spectral norm is
\begin{equation}\left\|I - \tfrac{1}{d} \mathbf{1}\mathbf{1}^T\right\|_2 = 1.\end{equation}

Moreover, since $\left\|I - \tfrac{1}{h} \mathbf{1}\mathbf{1}^T\right\|_2 = 1$ and $\Phi$ is an arbitrary orthonormal basis of $\mathbb R^d$, it follows that for any fixed $P \in \mathbb R^{d \times d}$, we can choose $\Phi$ such that the largest component of $P\Phi^T$ lies entirely in the subspace orthogonal to $\mathbf 1$. In this case,

\begin{equation}\left\|P\Phi^T(I - \tfrac{1}{d} \mathbf 1 \mathbf 1^T)\right\|_2 = \|P\|_2.\end{equation}

Thus, $\|\Delta B\|_2 \le \delta_2$ where $\delta_2 = 2\epsilon d^{3/2}\tfrac{\sigma_{\max}(W_v)^2}{\sigma_{\min}(W_v)}e^{2\delta_1}$.

Now that we have bounded $\|\Delta B\|_2$, we will consider any input $X$.
Then, $|x_i^T \Delta B x_j| \le \delta_2\|X\|_{\infty, 2}^2$, so $\|x_i^T \Delta B X^T\|_\infty \le \delta_2\|X\|_{\infty, 2}^2$. $\|X\|_{\infty, 2}$ is the maximum $\ell_2$ norm of the rows of $X$.

From \cref{softmax_lemma},
\begin{align*}
\|a_i - \hat a_i\|_2 &= \|\operatorname{Softmax}\left(\tfrac{x_i^T B X^T}{\sqrt d}\right) - \operatorname{Softmax}\left(\tfrac{x_i^T \hat B X^T}{\sqrt d}\right)\|_2 \\
&\le \left\|\tfrac{x_i^T B X^T}{\sqrt d} - \tfrac{x_i^T \hat B X^T}{\sqrt d}\right\|_\infty = \tfrac{1}{\sqrt d}\left\|x_i^T \Delta B X^T \right\|_\infty \le \tfrac{\delta_2\|X\|_{\infty, 2}^2}{\sqrt d}.
\end{align*}

\begin{equation}\|a_i - \hat a_i\|_2 \le \tfrac{\delta_2\|X\|_{\infty, 2}^2}{\sqrt d}
= 2\epsilon d\tfrac{\sigma_{\max}(W_v)^2}{\sigma_{\min}(W_v)}\exp\left(2\max\left(
    \tfrac{\|W_q\|_2 \|W_k\|_2}{\sigma_{\min}(W_v)^2}, \tfrac{\|\hat W_q\|_2 \|\hat W_k\|_2}{\sigma_{\min}(\hat W_v)^2}
\right)\right)\|X\|_{\infty, 2}^2.\end{equation}

\end{proof}

\section{Benchmark Dataset Details}
\label{app:datasets}

\subsection{LongBench}

\begin{table}[htbp]
\centering
\caption{LongBench tasks.}
\resizebox{1\linewidth}{!}{
\begin{tabular}{|l|l|c|c|c|c|c|}
\hline
\textbf{Task} & \textbf{Dataset} & \textbf{Source} & \textbf{Avg. Words} & \textbf{Metric} & \textbf{Language} & \textbf{Size} \\
\hline
\multicolumn{7}{|c|}{\textbf{Single-Document QA}} \\
\hline
1-1 & NarrativeQA        & Literature, Film       & 18,409 & F1         & English  & 200 \\
1-2 & Qasper             & Science               & 3,619  & F1         & English  & 200 \\
1-3 & MultiFieldQA-en    & Multi-field           & 4,559  & F1         & English  & 150 \\
1-4 & MultiFieldQA-zh    & Multi-field           & 6,701  & F1         & Chinese  & 200 \\
\hline
\multicolumn{7}{|c|}{\textbf{Multi-Document QA}} \\
\hline
2-1 & HotpotQA           & Wikipedia             & 9,151  & F1         & English  & 200 \\
2-2 & 2WikiMultihopQA    & Wikipedia             & 4,887  & F1         & English  & 200 \\
2-3 & MuSiQue            & Wikipedia             & 11,214 & F1         & English  & 200 \\
2-4 & DuReader           & Baidu Search          & 15,768 & Rouge-L    & Chinese  & 200 \\
\hline
\multicolumn{7}{|c|}{\textbf{Summarization}} \\
\hline
3-1 & GovReport          & Government report     & 8,734  & Rouge-L    & English  & 200 \\
3-2 & QMSum              & Meeting               & 10,614 & Rouge-L    & English  & 200 \\
3-3 & MultiNews          & News                  & 2,113  & Rouge-L    & English  & 200 \\
3-4 & VCSUM              & Meeting               & 15,380 & Rouge-L    & Chinese  & 200 \\
\hline
\multicolumn{7}{|c|}{\textbf{Few-shot Learning}} \\
\hline
4-1 & TREC               & Web question          & 5,177  & Accuracy (CLS) & English  & 200 \\
4-2 & TriviaQA           & Wikipedia, Web        & 8,209  & F1             & English  & 200 \\
4-3 & SAMSum             & Dialogue              & 6,258  & Rouge-L        & English  & 200 \\
4-4 & LSHT               & News                  & 22,337 & Accuracy (CLS) & Chinese  & 200 \\
\hline
\multicolumn{7}{|c|}{\textbf{Synthetic Task}} \\
\hline
5-1 & PassageCount       & Wikipedia             & 11,141 & Accuracy (EM)  & English  & 200 \\
5-2 & PassageRetrieval-en& Wikipedia             & 9,289  & Accuracy (EM)  & English  & 200 \\
5-3 & PassageRetrieval-zh& C4 Dataset            & 6,745  & Accuracy (EM)  & Chinese  & 200 \\
\hline
\multicolumn{7}{|c|}{\textbf{Code Completion}} \\
\hline
6-1 & LCC                & Github                & 1,235  & Edit Sim       & Python/C\#/Java & 500 \\
6-2 & RepoBench-P        & Github repository     & 4,206  & Edit Sim       & Python/Java     & 500 \\
\hline
\end{tabular}}
\label{tab:longbench_tasks}
\end{table}

LongBench\footnote{\url{https://huggingface.co/datasets/THUDM/LongBench} (MIT License)}~\citep{longbench} is a benchmark suite designed for long-context evaluation, comprising 14 English tasks, five Chinese tasks, and two code tasks. As Llama does not support Chinese, we excluded the corresponding tasks. Furthermore, we removed the synthetic tasks, as these are already covered by the RULER benchmark. The remaining tasks are grouped into five categories: single-document question answering, multi-document question answering, summarization, few-shot learning, and code completion. For each category, the overall score is calculated as the average of all its subtasks. The final LongBench score is computed as the average across all included tasks. \Cref{tab:longbench_tasks} provides an overview of all tasks, adapted from~\cite{longbench}.

\subsection{RULER}

RULER\footnote{\url{https://github.com/NVIDIA/RULER} (Apache License 2.0)}~\citep{ruler} is a synthetic dataset designed to evaluate the true supported context length of LLMs. It comprises 13 tasks, including eight needle-in-a-haystack (NIAH) retrieval tasks, two aggregation tasks, two question answering (QA) tasks, and one multi-hop tracing task.

The NIAH tasks involve hiding random key-value pairs within generated text and challenging the model to retrieve them. Aggregation tasks simulate summarization by asking the model to extract the most frequent or common words from a given passage. The QA tasks require the model to answer a question about a randomly selected paragraph within the context, serving as a real-world analog to NIAH tasks. In the multi-hop tracing task, the model must identify all variable names that reference the same value within a chain of assignments.

RULER is generated for a range of sequence lengths using randomly generated texts drawn from Paul Graham essays, SQuAD~\citep{rajpurkar2016squad}, and HotPotQA~\citep{yang2018hotpotqa} datasets. This approach enables a comprehensive assessment of a language model’s capability to process varying context lengths. Evaluation is conducted based on accuracy, considering a response correct if it contains the requested value associated with the specified key.

\subsection{MileBench}

\begin{table}[ht]
\centering
\caption{Overview of the MileBench datasets. 
  Average tokens are computed using Qwen2.5-VL~\citep{qwen2.5-vl}.}
\label{tab:milebench_data}
\resizebox{1\linewidth}{!}{
\begin{tabular}{llccccc}
\toprule
\textbf{Category} & \textbf{Dataset} & \textbf{Avg. Words} & \textbf{Avg. Images} & \textbf{Avg. Tokens} & \textbf{Metric}  & \textbf{Size} \\
\midrule
\multirow{3}{*}{\textbf{Temporal}} & EgocentricNavigation & 85 & 45 & 3,079 & Accuracy & 200 \\
& MovingDirection & 62 & 5 & 1,042 & Accuracy & 200 \\
& SceneTransition & 66 & 20 & 5,125 & Accuracy & 200 \\
\midrule
\multirow{3}{*}{\textbf{Semantic}} & SlideVQA & 66 & 2 & 2,053 & Accuracy & 200 \\
& TQA & 50 & 8 & 5,536 & Accuracy & 200 \\
& WebQA & 146 & 2 & 1,706 & Accuracy & 200 \\
\bottomrule
\end{tabular}}
\end{table}

MileBench\footnote{\url{https://milebench.github.io} (Apache License 2.0)}~\citep{milebench} is a long-context benchmark designed to evaluate Multimodal Large Language Models (MLLMs). It comprises 29 multi-image-text datasets, organized into 12 tasks, which are further grouped into four categories: Temporal Multi-Image, Semantic Multi-Image, and two diagnostic categories—NIAH and Image Retrieval.

For our additional experiments in \cref{app:add_res_milebench}, we selected three datasets each from the Temporal Multi-Image and Semantic Multi-Image categories: EgocentricNavigation, MovingDirection, and SceneTransition for the Temporal Multi-Image category, and SlideVQA, TQA, and WebQA for the Semantic Multi-Image category. \cref{tab:milebench_data} provides an overview of the selected datasets.

\section{Experimental Setup}
\label{app:exp_details}

\subsection{Hyperparameter Settings}

\begin{table}[ht]
\centering
\caption{Prompt compression backbones and parameter counts.}
\resizebox{0.7\linewidth}{!}{
\begin{tabular}{lll}
\toprule
\textbf{Method} & \textbf{Backbone}           & \textbf{Parameters} \\
\midrule
LLMLingua-2~\citep{pan2024llmlingua2}     & xlm-roberta-large~\citep{conneau2020roberta}           & 560M                \\
CPC~\citep{liskavets2024cpc}             & Llama-3.2-1B~\citep{grattafiori2024llama}                & 1B                  \\
R2C~\citep{choi2024r2c}             & T5-base~\citep{t5base}                     & 220M                \\
\bottomrule
\end{tabular}}
\label{tab:method_backbones}
\end{table}

\begin{table}[ht]
\centering
\caption{Summary of hyperparameters for various methods.}
\resizebox{\linewidth}{!}{
\begin{tabular}{lccccccccc}
\toprule
\textbf{Hyperparameter} & \textbf{StreamingLLM} & \textbf{H2O} & \textbf{SnapKV} & \textbf{PyramidKV} & \textbf{Ada-SnapKV} & \textbf{LAQ++} & \textbf{\speckv{}} & \textbf{SpecPrefill} & \textbf{\specpc{}} \\
\midrule
Window size $\ncompwnd$                & 32   & 32   & 32   & 32   & 32& 32& 32 & 1   & 64 \\
Pool                                         & --   & --   & Max  & Max& Max & Max  & Max   & Avg & Avg \\
Kernel size $\poolk$                                & --   & --   & 7& 7    & 7 & 7   & 7   & 13  & 64/32 \\
Reduction                                    & --   & --   & Mean & Mean& Mean & Max  & Max & Mean-Max  & Max \\
\# lookahead tokens $\nlookahead$ & --   & --   & --   & -- & -- & 8 & All  & 8 & 1 \\
Compression window size $\nprewnd$ & --   & --   & --   & -- & -- & --  & 2048 & -- & -- \\
\# global tokens in prefill $\nglobal$        & --   & -- & --  & -- & --  & --   & 2048 & -- & -- \\
\# neighbors $\nneigh$ / chunk size & --   & --   & -- & --  & --   & -- & -- & 32 & 64/32 \\
\# skipped layers $\lskip$                     & --   & --   & -- & --  & -- & --  & -- & --   & 8 \\
Initial cache size                     & --   & --   & -- & --  & --   & $\cmax$ & -- & --  & -- \\
\bottomrule
\end{tabular}}
\label{tab:hyperparameters}
\end{table}

\Cref{tab:method_backbones} lists the backbone models employed by each prompt compression method, while \cref{tab:hyperparameters} details the hyperparameters used in our experiments. Generally, we select hyperparameters for each method based on their respective codebases. We observe that using max aggregation improved performance compared to mean aggregation for \speckv{} and \specpc{}. For \speckv{}, setting $\nprewnd$ and $\nglobal$ to 2048 resulted in minimal accuracy loss but substantially reduced latency (\cref{app:sec:vert}).

For \speckv{}, we always generate tokens until the draft model produces the EOS token, which yields the best performance. For latency measurements, we set $\nlookahead=64$ tokens, reflecting the average sequence length in our benchmarks. In \specpc{}, prompt compression drops tokens uniformly across all layers and heads (unlike SpecKV, which prunes per head), so a larger $\cmax$ is needed to retain relevant information. While a larger $\nlookahead$ can boost performance, in practice, generating only one token per prompt ($\nlookahead=1$) is usually sufficient. Strong alignment between the draft and target model attentions enables \specpc{} to outperform methods like R2C and CPC. For an ablation on $\nlookahead$, see \cref{app:fig:nlookahead}.

Retaining the local context for prompt compression proved essential. This observation aligns with the design of existing prompt compression methods, which typically aim to preserve entire sentences within the prompt. Consequently, we increase both the pooling kernel size ($\poolk$) and the number of neighboring tokens ($\nneigh$) to 64. For Llama, slightly better results are achieved by reducing both $\poolk$ and $\nneigh$ to 32, though the performance difference was marginal.

For all remaining methods not explicitly mentioned, we use the default configurations provided in their respective codebases.

\subsection{Implementation Details}

For our experimental results, we employ the following large language models: 
\textbf{Llama-3.2-1B-Instruct}\footnote{\url{https://huggingface.co/meta-llama/Llama-3.2-1B-Instruct}~(Llama 3.2 license)}, 
\textbf{Llama-3.1-8B-Instruct}\footnote{\url{https://huggingface.co/meta-llama/Llama-3.1-8B-Instruct}~(Llama 3.1 license)}, 
\textbf{Llama-3.1-70B-Instruct (4-bit)}\footnote{\url{https://huggingface.co/meta-llama/Llama-3.1-70B-Instruct}~(Llama 3.1 license)}, 
\textbf{Qwen2.5-0.5B-Instruct}\footnote{\url{https://huggingface.co/Qwen/Qwen2.5-0.5B-Instruct}~(Apache License 2.0)}, 
\textbf{Qwen2.5-14B-Instruct}\footnote{\url{https://huggingface.co/Qwen/Qwen2.5-14B-Instruct}~(Apache License 2.0)}, 
\textbf{Qwen2.5-32B-Instruct}\footnote{\url{https://huggingface.co/Qwen/Qwen2.5-32B-Instruct}~(Apache License 2.0)}, and 
\textbf{Qwen2.5-72B-Instruct-GPTQ-Int4}\footnote{\url{https://huggingface.co/Qwen/Qwen2.5-72B-Instruct-GPTQ-Int4}~(Apache License 2.0)}. 
For MLLM evaluation on MileBench~\citep{milebench}, we utilize \textbf{Qwen2.5-VL-3B-Instruct-AWQ}\footnote{\url{https://huggingface.co/Qwen/Qwen2.5-VL-3B-Instruct-AWQ}~(Apache License 2.0)} and 
\textbf{Qwen2.5-VL-32B-Instruct-AWQ}\footnote{\url{https://huggingface.co/Qwen/Qwen2.5-VL-32B-Instruct-AWQ}~(Apache License 2.0)}.

Our implementation is based on \textsc{PyTorch}~\citep{pytorch} (BSD-3 License) and Huggingface's \textsc{Transformers}~\citep{huggingface} (Apache License 2.0). 
All experiments leverage \textsc{FlashAttention-2}\footnote{\url{https://github.com/Dao-AILab/flash-attention}~(BSD 3-Clause License)}~\citep{flashattn2}.
Latency measurements are performed using \textsc{\vllm{}}\footnote{\url{https://github.com/vllm-project/vllm}~(Apache License 2.0)} wherever possible 
(i.e., where attention map outputs are not required). 
For implementing the sparse prefill mechanism of \speckv{}, we use kernels from 
\textsc{MInference}\footnote{\url{https://github.com/microsoft/MInference}~(MIT License)}. 
All methods are evaluated via greedy decoding.
Experiments are conducted on NVIDIA H100 80GB GPUs, with runtimes varying by context length; for a maximum context length of 64K tokens, experiments take up to 2 hours.

For evaluating \textsc{StreamingLLM}~\citep{streamingllm}, \textsc{H2O}~\citep{h2o}, \textsc{SnapKV}~\citep{snapkv}, and \textsc{PyramidKV}~\citep{pyramidkv}, we use implementations from \textsc{KVCache-Factory}\footnote{\url{https://github.com/Zefan-Cai/KVCache-Factory}~(MIT License)}. In this library, the \textsc{StreamingLLM} and \textsc{H2O} implementations drop KV once after prefill, rather than at each decoding step, differing from their original codebases. This adjustment enables fairer comparison to \textsc{SnapKV} and others. 
We extend \textsc{KVCache-Factory} to support Grouped Query Attention~\citep{gqa} by repeating keys and values for each KV head, computing attention within the window, and averaging across KV heads. This approach avoids duplicating the KV cache.
For other baselines, we use their official implementations.
For Lookahead Q-Cache (LAQ++), we provide our own implementation since their code is not publicly available.

\section{Extended Experimental Analysis}
\label{app:add_res}

\subsection{Additional Results on RULER and LongBench}
\label{app:sec:add_models_exp}

In this section, we present results for RULER and LongBench using various $\cmax$ values of 256, 512, and 1024 for KV dropping, and 1024, 2048, and 3072 for prompt compression. Specifically, we employ Qwen2.5-32B-Instruct, Llama-3.1-70B-Instruct (4-bit quantized with bitsandbytes\footnote{\url{https://github.com/bitsandbytes-foundation/bitsandbytes} (MIT License)}), and Qwen2.5-72B-Instruct-GPTQ-Int4 as target model. 

\cref{app:fig:extra_ruler_kv,app:fig:extra_ruler_pc} show the RULER results, and
\cref{app:tab:extra_longbench_qwen25_32b,app:tab:extra_longbench_qwen25_72b,app:tab:extra_longbench_llama31_70b} present the LongBench results. Overall, our methods achieve higher accuracy than the baselines in most settings, especially with small $\cmax$. Specifically, \speckv{} significantly outperforms SnapKV and LAQ++ on RULER in most cases. Similarly, \specpc{} consistently achieves strong results, particularly at longer sequence lengths on RULER. On LongBench, both of our methods also surpass the baselines.

\subsection{Integration and Comparison with the AdaKV Baseline}
\label{app:sec:adakv}

\cref{fig:ruler_llama_qwen_full} (RULER), \cref{tab:qwen_longbench_full,tab:llama_longbench_full} (LongBench) present the performance of our proposed methods on Qwen2.5 (0.5B draft, 14B target) and Llama-3 (1B draft, 8B target). We include additional baselines, notably AdaKV~\citep{adakv}, for various values of $\cmax$. AdaKV is an extension to SnapKV and allows different KV budget per attention head. We apply AdaKV to SpecKV in a similar fashion, which further boosts performance in our experiments.



\subsection{Multi-modal Evaluation}
\label{app:add_res_milebench}

We conduct additional experiments using Qwen2.5-VL-3B-Instruct-AWQ (draft) and Qwen2.5-VL-32B-Instruct-AWQ (target) on six MileBench~\citep{milebench} datasets. We select three datasets each from the Temporal Multi-Image (EgocentricNavigation, MovingDirection, SceneTransition) and Semantic Multi-Image (SlideVQA, TQA, WebQA) categories. We focus on these datasets because, for Qwen2.5-VL, the performance gap between draft and target models is most significant; in other cases, the models perform too similarly or the draft even outperforms the target.

For KV dropping, we evaluate H2O, SnapKV, and PyramidKV from our prior experiments. We do not include AdaKV in our evaluation as it is dependent on an older Transformers~\citep{huggingface} version incompatible with Qwen2.5-VL. For prompt compression, we compare with FastV~\citep{fastv}—a method specialized for dropping image tokens inside LLMs. FastV uses a hyperparameter $k$: it runs all tokens up to layer $k$, then drops less-attended image tokens based on the attention map, processing only the top tokens thereafter. This makes FastV less efficient than \specpc{}, since all tokens must be processed up to $k$ with the full model, requiring considerable memory. Notably, FastV must compute the entire attention map at layer $k$, preventing the use of FlashAttention and leading to out-of-memory errors, even for moderate sequence lengths. As a result, many MileBench datasets exceed 80GB VRAM, so we limit our analysis to these six datasets.

Since the selected MileBench datasets have relatively short average context lengths, we conduct experiments using reduced $\cmax$ values for both KV cache dropping (64, 96, and 128) and prompt compression (512, 768, and 1024).

\cref{app:fig:milebench_speckv} presents results for various KV dropping methods. Our proposed method, \speckv{}, demonstrates performance comparable to existing approaches, while significantly outperforming the others on the WebQA task.

\cref{app:fig:milebench_specpc} compares the performance of \specpc{} and FastV under two configurations ($k=2$ and $k=5$). Our method consistently outperforms FastV in most cases.

\revised{\subsection{\speckv{}: Impact of $n_\text{lookahead}$ Parameter on Importance Score Correlation}}
\label{app:sec:speckv_corr}

\begin{figure}[h]
\centering
\includegraphics[width=0.4\textwidth]{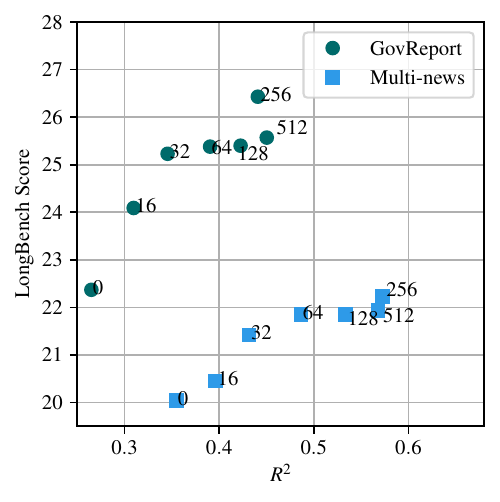}
\caption{\revised{Impact of $\nlookahead$ on SpecKV importance score accuracy ($R^2$) and LongBench downstream performance ($\nout=512$). The plot shows a strong positive correlation: as $\nlookahead$ increases (with $\nlookahead=0$ being equivalent to SnapKV), both the $R^2$ (correlation with the true target model scores) and the downstream task score improve. Notably, SpecKV improves correlation and accuracy even with a small $\nlookahead=32$, which is only 6.25\% of the output length. Experiments use a Qwen2.5-0.5B draft model and a Qwen2.5-32B target model on two LongBench tasks.}}
\label{fig:speckv_r2_vs_score}
\end{figure}


We analyze the impact of the $\nlookahead$ parameter on the correlation ($R^2$) between \speckv{}'s estimated importance scores and the ground-truth scores from the target model. For this analysis, we use 50 examples from the Multi-news and GovReport tasks from LongBench, selected for their long maximum output length (512 tokens). We employ Qwen2.5-0.5B as the draft model and Qwen2.5-32B as the target model, with $\cmax=256$.

To establish the ground-truth scores, we use the target model's generated output (up to $\nout$) as a perfect lookahead sequence and record the resulting importance scores. In \cref{fig:speckv_r2_vs_score,fig:speckv_imp_corr_multi_news,fig:speckv_imp_corr_gov_report}, we compute the $R^2$ correlation between these ground-truth scores and the scores estimated by \speckv{} using various $\nlookahead$ values.

As shown in \cref{fig:speckv_imp_corr_multi_news,fig:speckv_imp_corr_gov_report}, the $R^2$ correlation steadily increases with $\nlookahead$, confirming that a larger lookahead provides a more accurate importance estimation. This benefit is particularly pronounced for longer output sequences ($\noutput$). We also test $\nlookahead=0$ (equivalent to SnapKV), which yields a significantly inferior correlation.

Furthermore, we connect this score accuracy to downstream performance. \cref{fig:speckv_r2_vs_score} plots the $R^2$ value against the final LongBench downstream score. We observe that these two metrics are highly correlated: a higher $R^2$ (better score accuracy) generally yields better downstream performance. Consequently, increasing $\nlookahead$ improves both the $R^2$ correlation and the final task score.

\revised{\subsection{Performance of Cross-Family Models}}
\label{app:sec:cross_model}

\begin{figure}[h]
\centering
\begin{subfigure}[b]{0.495\textwidth}
\centering
\includegraphics[width=\textwidth]{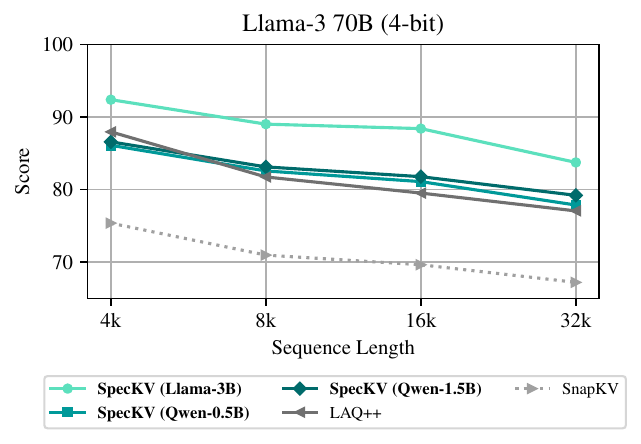}
\caption{KV cache dropping.}
\end{subfigure}
\hfill
\begin{subfigure}[b]{0.495\textwidth}
\centering
\includegraphics[width=\textwidth]{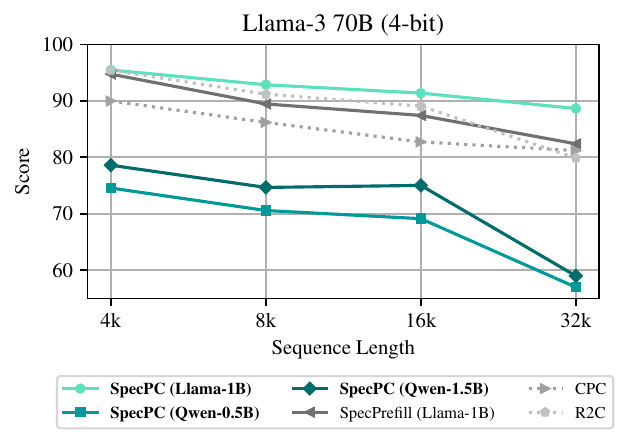}
\caption{Prompt compression.}
\end{subfigure}
\caption{\revised{Cross-family model results for \speckv{} and \specpc{} on RULER. \speckv{} demonstrates robust performance even with draft models from different families, outperforming the strongest baseline, LAQ++.}}
\label{fig:cross_ruler_llama70b_plot}
\end{figure}
\begin{table}
\centering
\caption{\revised{Cross-family model results for \speckv{} and \specpc{} on LongBench using Llama-3.1-70B (4-bit) as target model.}}
\resizebox{\linewidth}{!}{
\begin{tabular}{lllcccccc}
\toprule
  & $\cmax$ & \textbf{Method} & \makecell{\textbf{Single-}\\\textbf{doc QA}} & \makecell{\textbf{Multi-}\\\textbf{doc QA}} & \textbf{Summary} & 
\makecell{\textbf{Few-shot}\\\textbf{Learning}} & \makecell{\textbf{Code}\\\textbf{Completion}} & 
\textbf{All} \\
\midrule
\multirow{5}{*}{KV}  &  \multirow{5}{*}{256} 
&     SnapKV     & \textbf{55.88} & 45.30 & 22.49 & 62.15 & 55.49 & 47.75 \\
   &       &    LAQ++   & \underline{54.90} & 46.48 & 22.83 & \underline{64.31} & 55.10 & 48.43 \\
   &       &     \textbf{\speckv{} (Llama-3B)}     & 51.80 & \textbf{47.23} & \textbf{25.53} & 64.02 & \textbf{58.75} & \textbf{48.80} \\
     &       &     \textbf{\speckv{} (Qwen-0.5B)}  & 53.18 & 45.70 & 24.40 & 63.69 & 57.36 & 48.26 \\
     &       &     \textbf{\speckv{} (Qwen-1.5B)}  & 51.20 & \underline{47.07} & \underline{24.98} & \textbf{65.81} & \underline{57.53} & \underline{48.73} \\
\midrule
\multirow{6}{*}{PC}  &  \multirow{6}{*}{1024} 
&    CPC    & 45.14 & 39.41 & 24.86 & 61.40 & 37.58 & 41.97 \\
  &     &    R2C    & 48.93 & 42.01 & 25.38 & 58.91 & 40.19 & 43.29 \\
   &       &    SpecPrefill (Llama-1B)   & \underline{54.62} & \textbf{46.43} & 25.63 & 64.80 & 44.92 & \underline{48.37} \\
  &     &    \textbf{\specpc{} (Llama-1B)}    & \textbf{56.84} & \underline{44.48} & \textbf{25.91} & \textbf{67.37} & 47.15 & \textbf{48.44} \\
     &       &     \textbf{\specpc{} (Qwen-0.5B)}  & 36.73 & 36.42 & 24.68 & 64.51 & \textbf{50.75} & 42.04 \\
     &       &     \textbf{\specpc{} (Qwen-1.5B)}  & 51.58 & 39.53 & \underline{25.87} & \underline{66.43} & \underline{50.25} & 46.48 \\
\bottomrule
\end{tabular}}
\label{tab:cross_model_llama31_70b_longbench}
\end{table}

In this section, we evaluate the effectiveness of our methods using cross-family draft models. We set Llama-3.1-70B as the target model and employ Qwen2.5-0.5B and Qwen2.5B-1.5B as draft models. The evaluation is conducted on the RULER and LongBench benchmarks, using $\cmax=256$ for \speckv{} and $\cmax=1024$ for \specpc{}.

A key challenge in this setup is that the Qwen and Llama families use different tokenizers, necessitating a token translation step. For \speckv{}, we de-tokenize the draft model's output and then re-tokenize it with the target model's tokenizer. For \specpc{}, we aggregate attention scores from tokens to words before re-tokenizing, to avoid tokenizing partial words, and replace the draft model's chat template with the target's.

The results, shown in \cref{fig:cross_ruler_llama70b_plot} (RULER) and \cref{tab:cross_model_llama31_70b_longbench} (LongBench), indicate that both \speckv{} and \specpc{} achieve good results even with cross-family drafts. Notably, \speckv{} outperforms the strongest baseline, LAQ++. This aligns with the intuition that \speckv{}, which relies only on the draft model's output, should generalize well to cross-model scenarios. \specpc{} experiences a performance drop because it relies on similar attention patterns between the draft and target models, a condition that may not hold in cross-family setups.

\revised{\subsection{Extended Results for SpecKV-PC}}
\label{app:sec:combined_speckv_specpc}

This section presents extended results for the cascaded compression strategy, SpecKV-PC. By restricting the large target model to process only a small fraction of the original prompt, this approach achieves substantially better latency and memory efficiency than \speckv{} alone.

We evaluate the impact of varying prompt compression ratios on accuracy using Qwen2.5-1.5B/32B and Llama-3-3B/70B as draft/target pairs, with a final KV cache size of $\cmax=256$. Our combined method, denoted as SpecKV-PC-X (where the prompt is compressed to X tokens), is benchmarked against standard \speckv{}, \specpc{}, LAQ++, and SnapKV.

As illustrated in \cref{fig:combined_speckv_specpc_ruler_qwen32b_llama70b_kv_256} (RULER) and \cref{tab:combined_speckv_specpc_longbench} (LongBench), SpecKV-PC proves to be both efficient and highly accurate. Notably, we find that moderate pre-compression (e.g., SpecKV-PC-2048) yields accuracy superior to standard \speckv{} alone, particularly at the longer sequence lengths (RULER). This suggests that the initial \specpc{} stage acts as an effective pre-filter, discarding irrelevant information to help the target model focus on the most important tokens.

Conversely, extreme prompt compression (e.g., \specpc{}-256) results in significant performance degradation. This indicates that for aggressive compression targets, coarse prompt-level reduction is insufficient.

\begin{figure}[h]
\centering
\includegraphics[width=0.5\textwidth]{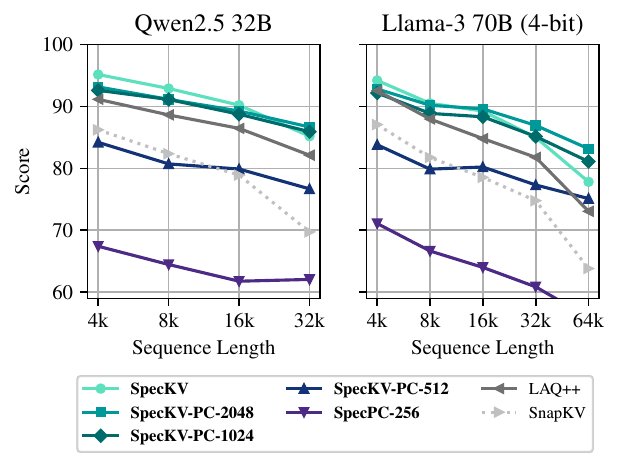}
\caption{\revised{
Performance of the combined \speckv{} and \specpc{} methods on RULER (final $\cmax=256$). 
Our cascaded approach (SpecKV-PC-X) pre-compresses the prompt to X tokens, achieving higher accuracy than standard \speckv{} (which sees the full prompt) at longer sequences. 
This suggests the initial pre-filtering allows \speckv{} to make a more effective final selection.
}}
\label{fig:combined_speckv_specpc_ruler_qwen32b_llama70b_kv_256}
\end{figure}

\begin{table}[h!]
\centering
\caption{\revised{LongBench performance of cascaded compression with \speckv{} and \specpc{} using Qwen2.5 and Llama-3. SpecKV-PC-2048 achieves superior performance.}}
\resizebox{\linewidth}{!}{
\begin{tabular}{llcccccc|cccccc}
\toprule
 & & \multicolumn{6}{c|}{\textbf{Qwen2.5 32B}} & \multicolumn{6}{c}{\textbf{Llama-3 70B (4-bit)}} \\
\cmidrule(lr){3-8}\cmidrule(lr){9-14}
\textbf{Group} & \textbf{Method} 
 & \rotatebox{45}{SingleQA} & \rotatebox{45}{MultiQA} & \rotatebox{45}{Summ.} & \rotatebox{45}{Few-shot} & \rotatebox{45}{Code} & \rotatebox{45}{All}
 & \rotatebox{45}{SingleQA} & \rotatebox{45}{MultiQA} & \rotatebox{45}{Summ.} & \rotatebox{45}{Few-shot} & \rotatebox{45}{Code} & \rotatebox{45}{All} \\
\midrule
\multirow{1}{*}{Dense}
& Target  & 56.01 & 43.99 & 25.90 & 64.06 & 44.74 & 47.78 & 55.02 & 47.06 & 28.61 & 70.47 & 48.19 & 49.99 \\
\midrule
\multirow{6}{*}{KV}
  &   SnapKV    & 52.54 & 40.21 & 19.89 & 61.18 & 40.12 & 42.98 & 55.88 & 45.30 & 22.49 & 62.15 & 55.49 & 47.75 \\
  &   LAQ++    & \underline{55.15} & \underline{44.14} & 22.24 & 63.25 & 41.19 & 45.79 & 54.90 & 46.48 & 22.83 & \underline{64.31} & 55.10 & 48.43 \\
  &   \textbf{SpecKV}    & 53.48 & 43.77 & 24.02 & \textbf{63.79} & \underline{44.80} & \underline{46.06} & 51.80 & \underline{47.23} & 25.53 & 64.02 & \textbf{58.75} & 48.80 \\
   &   \textbf{SpecKV-PC-2048}  & 52.60 & \textbf{44.52} & \underline{24.11} & \underline{63.38} & \textbf{48.45} & \textbf{46.48} & \textbf{61.42} & 47.15 & \textbf{26.51} & \textbf{66.94} & \underline{58.19} & \textbf{51.60} \\
   &   \textbf{SpecKV-PC-1024}  & \textbf{55.37} & 42.21 & \textbf{24.14} & 63.09 & 39.72 & 45.28 & \underline{58.73} & \textbf{47.33} & \underline{25.69} & 64.28 & 55.19 & \underline{49.89} \\
   &   \textbf{SpecKV-PC-512}   & 45.77 & 36.57 & 22.36 & 56.57 & 39.57 & 40.21 & 51.73 & 46.11 & 23.01 & 63.64 & 50.00 & 46.68 \\
  \midrule
\multirow{1}{*}{PC}
  &  \textbf{SpecPC-256}   & 34.75 & 27.05 & 19.15 & 45.0 & 35.07 & 32.0 & 27.97 & 25.93 & 19.09 & 51.25 & 48.12 & 33.5 \\
\bottomrule
\end{tabular}}
\label{tab:combined_speckv_specpc_longbench}
\end{table}

\revised{\subsection{\speckv{}: Extended Latency and Memory Analysis}}
\label{app:sec:ext_lat_and_mem}

We analyze the latency and memory usage of baselines and our algorithms, breaking down latency into three stages: draft generation, target prefill, and target decoding. For LAQ++, draft generation occurs after target prefill to share the same KV cache.

All experiments are run on an H200 GPU (141 GB VRAM) using the Qwen2.5 model (3B draft, 32B target). We evaluate across a range of input sizes, $\ninput \in \{4\text{k}, 8\text{k}, 16\text{k}, 32\text{k}, 64\text{k}\}$, and output sizes, $\noutput \in \{64, 128, 256, 512\}$. A global $\cmax=256$ is used for all experiments. For algorithm-specific settings, we set $\nlookahead = \nout$ for SpecKV and use the default $\nlookahead = 8$ for LAQ++. For the SpecKV-PC variant, we employ SpecKV-PC-2048, which pre-compresses the prompt to 2048 tokens.

Our results demonstrate that draft lookahead introduces negligible latency (\cref{fig:latency_breakdown_plot}) and memory overhead (\cref{fig:memory_breakdown_plot}) relative to the total cost of target model inference. Notably, cascaded compression with \speckv{} and \specpc{} (SpecKV-PC) significantly reduces prefill time and peak memory. At 64k context, SpecKV-PC is about 40\% faster and 25 GB more memory efficient than LAQ++.

\subsection{Ablation Studies}

In this section, we conduct a series of ablation experiments to further analyze the effectiveness of our two proposed methods: \speckv{} and \specpc{}. For consistency, we fix $\cmax$ to 256 for KV dropping and 1024 for prompt compression across all experiments. We utilize Qwen2.5 (Instruct), employing the 0.5B model as the draft and the 14B model as the target. We sample 100 random examples per task from the LongBench and RULER benchmarks.

\subsubsection{\speckv{}: Analysis of Enhanced Draft Models}

\cref{app:fig:better_draft} illustrates the impact of using draft models of different sizes and versions on the RULER score. As anticipated, both newer~\citep{qwen3} and larger draft models lead to improved performance of \speckv{}.

\subsubsection{\speckv{} and \specpc{}: Impact of $n_\text{lookahead}$}

\cref{app:fig:nlookahead} illustrates how varying the number of generated draft tokens, $\nlookahead$, affects the performance of \speckv{} and \specpc{}. Overall, increasing $\nlookahead$ generally results in higher final accuracy for \speckv{}, whereas it yields only marginal improvements for \specpc{}. We attribute this to the larger $\cmax$ budget of \specpc{}, which allows it to capture all important tokens without needing to generate long drafts.

\subsubsection{\speckv{}: Accuracy Analysis of Sparse Prefill}
\label{app:sec:vert}

In this section, we experimentally evaluate how the sparsity of \speckv{}'s prefill procedure affects downstream task performance (\cref{app:fig:nvert}). Specifically, we set $\nglobal$ equal to $\nprewnd$ and compare several values for these parameters. As anticipated, reducing sparsity (i.e., using a higher $\nglobal$) generally results in higher accuracy; however, accuracy improvements plateau at $\nglobal = 2048$, which we therefore adopt for our main experiments.

Interestingly, for certain LongBench categories, increased sparsity (i.e., lower $\nglobal$) can actually lead to improved performance. This counterintuitive result suggests that, for some tasks, sparser prefill may serve as a form of regularization, preventing overfitting to irrelevant context.

\section*{LLM Usage Disclosure}
We utilized Google's Gemini 2.5 Pro~\citep{gemini25pro} and OpenAI's GPT~\citep{gpt4} to assist with improving the grammar, clarity, and readability of this manuscript. The authors reviewed and edited all LLM-generated suggestions to ensure the final text accurately reflects our scientific contributions and claims. The authors retain full responsibility for the content of this paper.



\clearpage

\begin{figure}[hp]
    \centering
     \begin{subfigure}[b]{1\textwidth}
         \centering
     \includegraphics[width=\textwidth]{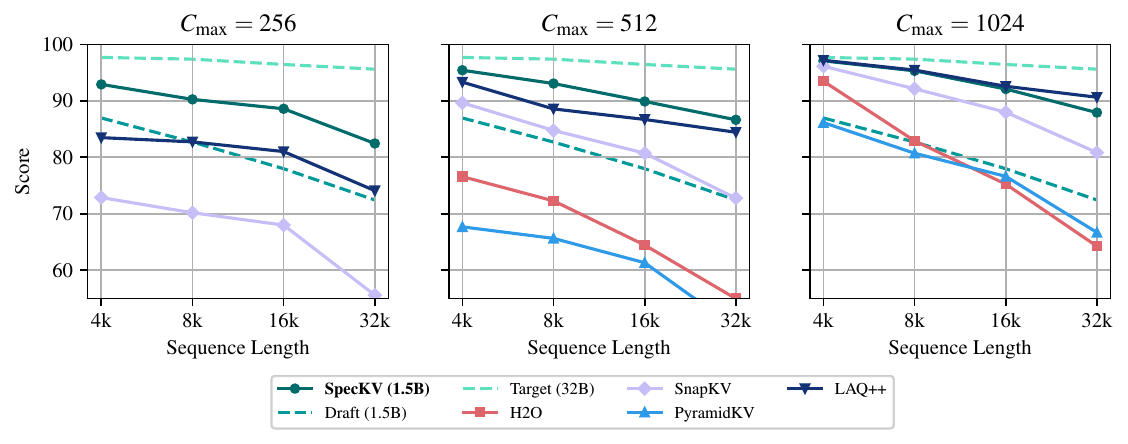}
     \caption{Qwen2.5: Draft model (1.5B), target model (32B)\footnotemark}
     \end{subfigure}
     \par\bigskip
     \begin{subfigure}[b]{1\textwidth}
      \includegraphics[width=\textwidth]{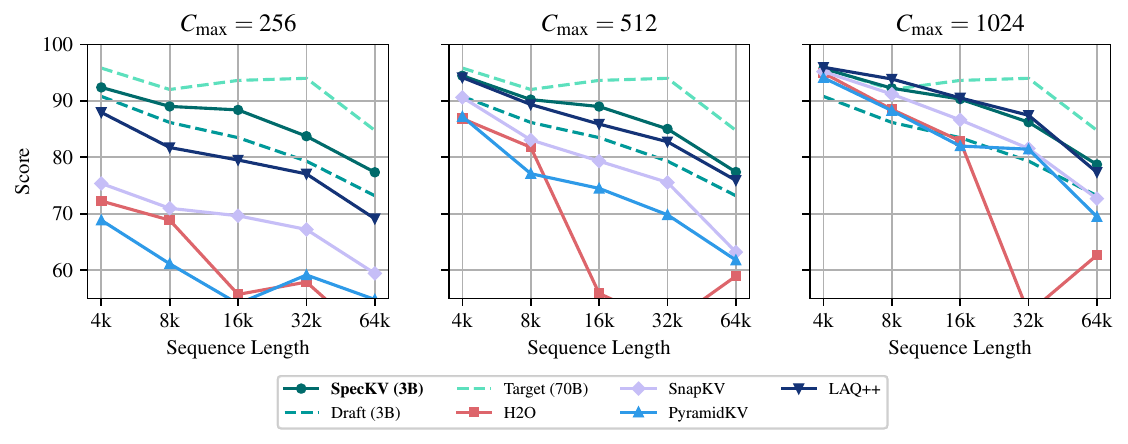}
     \caption{Llama-3: Draft model (3.2-3B), target model (3.1-70B, 4-bit)}
     \end{subfigure}
     \par\bigskip
     \begin{subfigure}[b]{1\textwidth}
      \includegraphics[width=\textwidth]{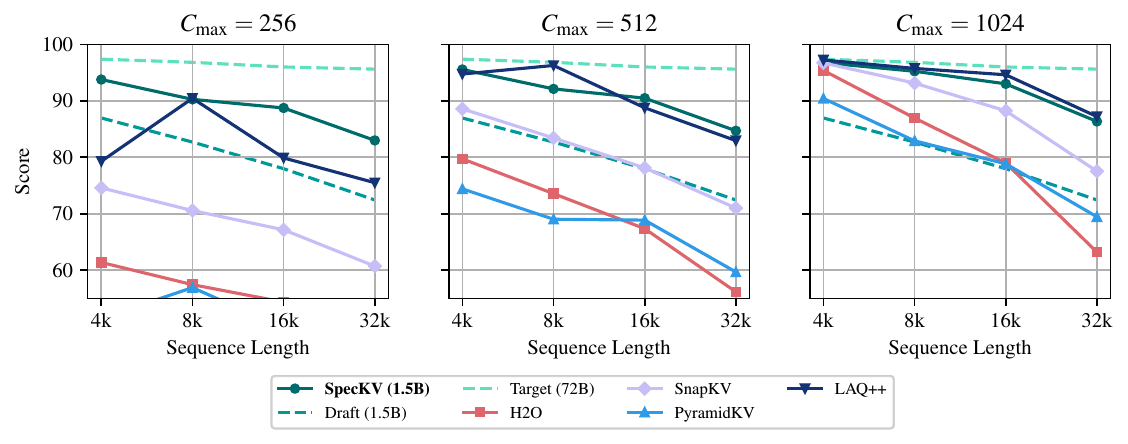}
     \caption{Qwen2.5: Draft model (1.5B), target model (72B, 4-bit)}
     \end{subfigure}
\caption{Extended RULER results for KV cache dropping. Our proposed \speckv{} method consistently outperforms SnapKV and LAQ++ across the majority of evaluated settings, often by a substantial margin.}
\label{app:fig:extra_ruler_kv}
\end{figure}
\footnotetext{H2O and PyramidKV are not plotted for Qwen2.5 at $\cmax=256$ as their performance falls outside the visible range.}
\begin{figure}[hp]
    \centering
     \begin{subfigure}[b]{1\textwidth}
         \centering
\includegraphics[width=\textwidth]{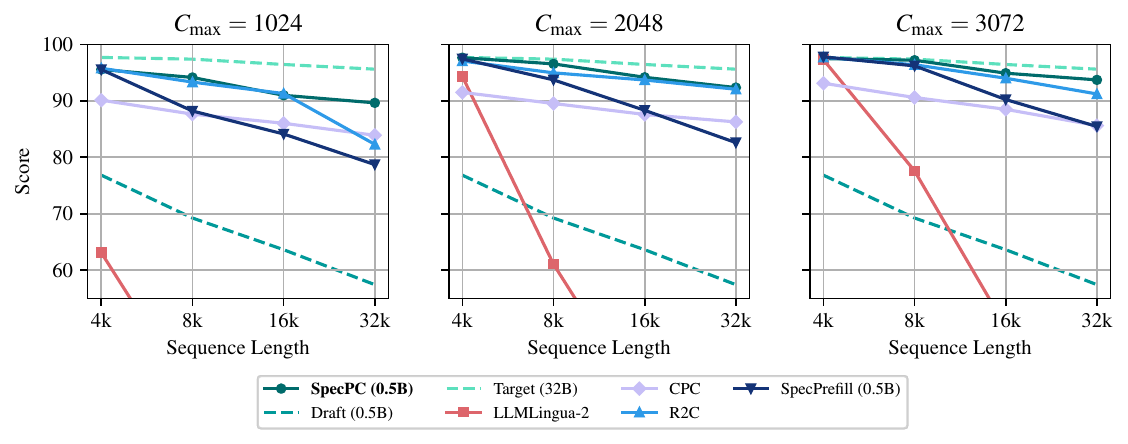}
     \caption{Qwen2.5: Draft model (0.5B), target model (32B)}
     \end{subfigure}
     \par\bigskip
     \begin{subfigure}[b]{1\textwidth}
      \includegraphics[width=\textwidth]{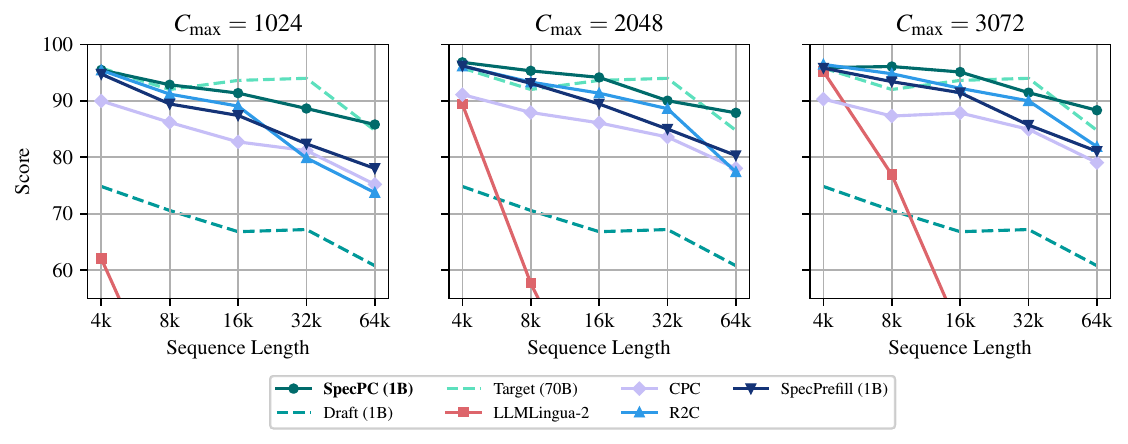}
     \caption{Llama-3: Draft model (3.2-1B), target model (3.1-70B, 4-bit)}
     \end{subfigure}
     \par\bigskip
     \begin{subfigure}[b]{1\textwidth}
      \includegraphics[width=\textwidth]{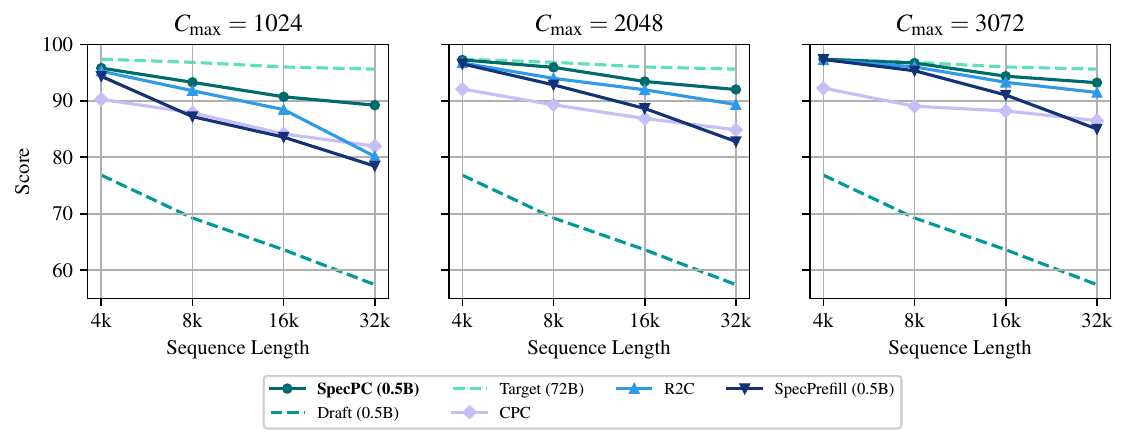}
     \caption{Qwen2.5: Draft model (0.5B), target model (72B, 4-bit)}
     \end{subfigure}
\caption{Extended RULER results on prompt compression. Our proposed \specpc{} method consistently outperforms CPC, R2C, and SpecPrefill, maintaining strong performance even on long sequences.}
\label{app:fig:extra_ruler_pc}
\end{figure}


\clearpage
\begin{table}
\centering
\caption{LongBench results for Qwen2.5, featuring 0.5B (\specpc{}) and 1.5B (\speckv{}) draft models and a 32B target model.}
\resizebox{\linewidth}{!}{
\begin{tabular}{lllcccccc}
\toprule
  & $\cmax$ & \textbf{Method} & \makecell{\textbf{Single-}\\\textbf{doc QA}} & \makecell{\textbf{Multi-}\\\textbf{doc QA}} & \textbf{Summary} & 
\makecell{\textbf{Few-shot}\\\textbf{Learning}} & \makecell{\textbf{Code}\\\textbf{Completion}} & 
\textbf{All} \\
\midrule
\multirow{3}{*}{Dense} & \multirow{3}{*}{--} & Draft (0.5B) & 19.07 & 26.58 & 20.90 & 53.51 & 32.48 & 30.37 \\
& &  Draft (1.5B)  & 36.16 & 35.01 & 22.79 & 63.92 & 36.62 & 39.06 \\
& &  Target (32B) & 56.01 & 43.99 & 25.90 & 64.06 & 44.74 & 47.78 \\
\midrule
\multirow{16}{*}{KV}  &  \multirow{5}{*}{256} 
&    H2O  & 46.63 & 30.81 & 19.88 & 56.03 & 39.27 & 39.32 \\
 &   &    SnapKV    & 52.54 & 40.21 & 19.89 & 61.18 & 40.12 & 42.98 \\
  &     &   PyramidKV  & 50.92 & 37.26 & 18.90 & 63.24 & 40.20 & 43.19 \\
  &     &   LAQ++  & \textbf{55.15} & \textbf{44.14} & \underline{22.24} & \underline{63.25} & \underline{41.19} & \underline{45.79} \\
  &      &    \textbf{\speckv{} (1.5B)}    & \underline{53.48} & \underline{43.77} & \textbf{24.02} & \textbf{63.79} & \textbf{44.80} & \textbf{46.06} \\
\cmidrule{2-9}
 &  \multirow{5}{*}{512}  
&  H2O  & 48.62 & 36.72 & 21.31 & 59.10 & 40.99 & 42.27 \\
 &   &    SnapKV    & \underline{55.24} & 42.21 & 21.47 & \underline{63.36} & 39.69 & 44.73 \\
  &     &   PyramidKV  & 54.79 & 41.87 & 20.28 & 62.00 & 40.08 & \underline{46.55} \\
  &     &   LAQ++  & \textbf{55.77} & \textbf{45.30} & \underline{23.96} & 62.37 & \underline{41.86} & 46.44 \\
  &      &    \textbf{\speckv{} (1.5B)}    & 52.78 & \underline{43.97} & \textbf{24.80} & \textbf{64.72} & \textbf{46.77} & \textbf{46.60} \\
\cmidrule{2-9}
&  \multirow{5}{*}{1024}  
&  H2O  & 51.86 & 41.30 & 23.02 & 60.18 & 42.93 & 44.84 \\
 &   &    SnapKV    & 55.32 & \underline{44.04} & 23.08 & \textbf{66.15} & \textbf{44.42} & \underline{46.76} \\
  &     &   PyramidKV  & \underline{55.91} & 42.67 & 22.28 & \underline{64.76} & 42.52 & \textbf{48.31} \\
  &     &   LAQ++  & \textbf{56.33} & \textbf{45.18} & \underline{25.17} & 62.13 & \underline{43.06} & 46.61 \\
  &      &    \textbf{\speckv{} (1.5B)}    & 55.72 & 43.95 & \textbf{25.20} & 63.31 & 42.38 & 46.38 \\
\midrule
\multirow{13}{*}{PC}  &  \multirow{4}{*}{1024} &   CPC   & 45.60 & \underline{40.62} & 23.09 & 60.08 & 32.31 & 40.91 \\
 &   &   R2C   & \underline{50.49} & 40.37 & \underline{23.26} & 53.45 & 34.11 & 39.88 \\
  &     &   SpecPrefill (0.5B)  & 45.94 & 39.32 & 23.16 & \underline{62.04} & \textbf{43.17} & \underline{42.70} \\
 &   &   \textbf{\specpc{} (0.5B)}   & \textbf{51.23} & \textbf{41.40} & \textbf{23.37} & \textbf{62.26} & \underline{38.23} & \textbf{43.66} \\
\cmidrule{2-9}
 &  \multirow{4}{*}{2048} &   CPC   & 51.01 & 42.31 & 23.74 & 60.92 & 35.83 & 43.26 \\
 &   &   R2C   & 50.32 & \textbf{42.66} & \underline{24.08} & 59.11 & 40.54 & 44.19 \\
  &     &   SpecPrefill (0.5B)  & \underline{51.60} & 40.72 & 23.71 & \textbf{64.20} & \underline{45.29} & \underline{45.09} \\
 &   &   \textbf{\specpc{} (0.5B)}   & \textbf{55.40} & \underline{42.60} & \textbf{24.12} & \underline{61.46} & \textbf{48.05} & \textbf{46.20} \\
\cmidrule{2-9}
& \multirow{4}{*}{3072} &   CPC   & \underline{56.80} & 42.05 & \textbf{24.77} & \underline{62.79} & 37.98 & 45.37 \\
 &    &   R2C   & 51.67 & \underline{42.88} & 24.09 & 62.45 & 27.01 & 42.66 \\
  &     &   SpecPrefill (0.5B)  & 53.85 & \textbf{42.92} & 24.48 & \textbf{64.03} & \underline{45.76} & \underline{46.24} \\
 &    &   \textbf{\specpc{} (0.5B)}   & \textbf{56.89} & 42.42 & \underline{24.71} & 62.66 & \textbf{47.49} & \textbf{46.79} \\
\bottomrule
\end{tabular}}
\label{app:tab:extra_longbench_qwen25_32b}
\end{table}

\begin{table}
\centering
\caption{LongBench results for Llama-3, featuring 3.2-1B (\specpc{}) and 3.2-3B (\speckv{}) draft models and a 3.1-70B (4-bit) target model.}
\resizebox{\linewidth}{!}{
\begin{tabular}{lllcccccc}
\toprule
  & $\cmax$ & \textbf{Method} & \makecell{\textbf{Single-}\\\textbf{doc QA}} & \makecell{\textbf{Multi-}\\\textbf{doc QA}} & \textbf{Summary} & 
\makecell{\textbf{Few-shot}\\\textbf{Learning}} & \makecell{\textbf{Code}\\\textbf{Completion}} & 
\textbf{All} \\
\midrule
\multirow{3}{*}{Dense} & \multirow{3}{*}{--} & Draft (1B) & 28.37 & 28.62 & 26.12 & 59.10 & 33.59 & 35.27 \\
& &  Draft (3B) & 45.53 & 40.52 & 27.46 & 64.82 & 46.73 & 44.89 \\
& &  Target (70B) & 55.02 & 47.06 & 28.61 & 70.47 & 48.19 & 49.99 \\
\midrule
\multirow{16}{*}{KV}  &  \multirow{5}{*}{256} 
 &   H2O  & 54.07 & 41.30 & 22.55 & 49.10 & 54.14 & 43.52 \\
 &    &    SnapKV    & \textbf{55.88} & 45.30 & 22.49 & 62.15 & \underline{55.49} & 47.75 \\
  &     &   PyramidKV  & \underline{55.41} & 45.59 & 22.50 & 59.06 & 49.90 & 46.25 \\
  &     &   LAQ++  & 54.90 & \underline{46.48} & \underline{22.83} & \textbf{64.31} & 55.10 & \underline{48.43} \\
  &     &    \textbf{\speckv{} (3B)}    & 51.80 & \textbf{47.23} & \textbf{25.53} & \underline{64.02} & \textbf{58.75} & \textbf{48.80} \\
\cmidrule{2-9}
 &  \multirow{5}{*}{512} 
&  H2O  & 55.01 & 43.86 & 23.85 & 58.48 & 52.00 & 46.26 \\
 &    &    SnapKV    & \textbf{56.08} & 46.94 & 24.40 & 63.34 & \underline{52.14} & 48.32 \\
  &     &   PyramidKV  & 54.73 & 47.07 & 24.20 & \underline{63.62} & 47.05 & 47.35 \\
  &     &   LAQ++  & \underline{55.06} & \underline{47.45} & \underline{24.58} & 63.39 & 51.75 & \underline{48.36} \\
  &     &    \textbf{\speckv{} (3B)}    & 53.31 & \textbf{47.47} & \textbf{27.33} & \textbf{67.49} & \textbf{54.19} & \textbf{49.66} \\
\cmidrule{2-9}
&  \multirow{5}{*}{1024} 
&  H2O  & 54.19 & 45.77 & 26.03 & 63.62 & \underline{49.58} & 47.71 \\
 &    &    SnapKV    & 54.94 & \textbf{47.40} & 25.97 & \textbf{67.56} & 48.13 & \textbf{48.85} \\
  &     &   PyramidKV  & \textbf{56.32} & 46.82 & 26.20 & 62.58 & 48.28 & 48.02 \\
  &     &   LAQ++  & \underline{55.19} & \underline{47.17} & \underline{26.43} & \underline{67.41} & 47.16 & 48.61 \\
  &     &    \textbf{\speckv{} (3B)}    & 53.32 & 46.02 & \textbf{27.41} & 66.40 & \textbf{50.66} & \underline{48.63} \\
\midrule
\multirow{13}{*}{PC}  &  \multirow{4}{*}{1024} &   CPC   & 45.14 & 39.41 & 24.86 & 61.40 & 37.58 & 41.97 \\
 &   &   R2C   & 48.93 & 42.01 & 25.38 & 58.91 & 40.19 & 43.29 \\
  &     &   SpecPrefill (1B)  & \underline{54.62} & \textbf{46.43} & \underline{25.63} & \underline{64.80} & \underline{44.92} & \underline{48.37} \\
 &   &   \textbf{\specpc{} (1B)}   & \textbf{56.84} & \underline{44.48} & \textbf{25.91} & \textbf{67.37} & \textbf{47.15} & \textbf{48.44} \\
\cmidrule{2-9}
 &  \multirow{4}{*}{2048} &   CPC   & 55.97 & 46.00 & 26.72 & 64.78 & 41.31 & 47.36 \\
 &   &   R2C   & 53.62 & \textbf{46.70} & 26.27 & 62.41 & \textbf{47.90} & 47.34 \\
  &     &   SpecPrefill (1B)  & \underline{58.39} & 45.37 & \underline{27.13} & \underline{66.60} & 45.45 & \underline{48.81} \\
 &   &   \textbf{\specpc{} (1B)}   & \textbf{59.39} & \underline{46.25} & \textbf{27.60} & \textbf{68.42} & \underline{46.70} & \textbf{49.88} \\
\cmidrule{2-9}
& \multirow{4}{*}{3072} &   CPC   & 56.64 & \underline{46.69} & 27.29 & \underline{64.92} & \underline{44.75} & \underline{48.30} \\
 &   &   R2C   & 56.11 & 45.15 & \underline{27.51} & 61.76 & \textbf{47.17} & 47.57 \\
  &     &   SpecPrefill (1B)  & \underline{57.75} & 46.15 & 27.15 & 64.70 & 42.50 & 48.02 \\
 &   &   \textbf{\specpc{} (1B)}   & \textbf{58.47} & \textbf{48.07} & \textbf{27.72} & \textbf{65.43} & 41.28 & \textbf{48.69} \\
\bottomrule
\end{tabular}}
\label{app:tab:extra_longbench_llama31_70b}
\end{table}

\begin{table}
\centering
\caption{LongBench results for Qwen2.5, featuring 0.5B (\specpc{}) and 1.5B (\speckv{}) draft models and a 72B (4-bit) target model.}
\resizebox{\linewidth}{!}{
\begin{tabular}{lllcccccc}
\toprule
  & $\cmax$ & \textbf{Method} & \makecell{\textbf{Single-}\\\textbf{doc QA}} & \makecell{\textbf{Multi-}\\\textbf{doc QA}} & \textbf{Summary} & 
\makecell{\textbf{Few-shot}\\\textbf{Learning}} & \makecell{\textbf{Code}\\\textbf{Completion}} & 
\textbf{All} \\
\midrule
\multirow{3}{*}{Dense} & \multirow{3}{*}{--} & Draft (0.5B) & 19.07 & 26.58 & 20.90 & 53.51 & 32.48 & 30.37 \\
& &  Draft (1.5B)  & 36.16 & 35.01 & 22.79 & 63.92 & 36.62 & 39.06 \\
&  & Target (72B)  & 58.70 & 45.89 & 26.24 & 64.83 & 51.08 & 49.22 \\
\midrule
\multirow{16}{*}{KV}  &  \multirow{5}{*}{256} 
 &  H2O  & 53.68 & 35.05 & 21.10 & 50.84 & 48.87 & 41.41 \\
  &   &    SnapKV    & \underline{55.85} & 40.12 & 21.65 & 55.82 & \underline{50.41} & 44.37 \\
  &     &   PyramidKV  & 55.13 & 38.89 & 20.14 & 53.02 & 49.59 & 42.91 \\
  &     &   LAQ++  & \textbf{57.26} & \textbf{43.54} & \underline{22.95} & \underline{62.53} & \textbf{55.79} & \underline{46.98} \\
  &     &    \textbf{\speckv{} (1.5B)}  & 53.94 & \underline{42.71} & \textbf{25.07} & \textbf{66.58} & 46.93 & \textbf{47.06} \\
\cmidrule{2-9}
 &  \multirow{5}{*}{512} 
&  H2O  & 55.77 & 39.52 & 22.61 & 58.01 & 50.69 & 44.94 \\
 &   &   SnapKV    & \textbf{59.08} & \textbf{44.13} & 23.04 & 61.33 & \underline{51.77} & \underline{47.59} \\
  &     &   PyramidKV  & 56.98 & 42.52 & 22.03 & 55.43 & 50.29 & 45.10 \\
  &     &   LAQ++  & \underline{58.28} & 34.94 & \underline{24.38} & \underline{63.74} & \textbf{54.79} & \textbf{49.08} \\
  &     &    \textbf{\speckv{} (1.5B)}    & 55.74 & \underline{43.39} & \textbf{25.85} & \textbf{64.56} & 47.80 & 47.44 \\
\cmidrule{2-9}
&  \multirow{5}{*}{1024} 
 &  H2O  & 55.37 & 43.50 & 24.11 & 62.09 & 50.68 & 46.90 \\
 &    &    SnapKV     & \textbf{59.56} & \underline{44.06} & 24.41 & \underline{62.85} & \underline{52.87} & \underline{48.46} \\
  &     &   PyramidKV  & 58.34 & 43.68 & 23.18 & 59.66 & 51.45 & 46.96 \\
  &     &   LAQ++  & \underline{58.66} & 35.37 & \textbf{26.35} & 60.90 & \textbf{53.56} & \textbf{49.59} \\
   &       &     \textbf{\speckv{} (1.5B)}     & 56.04 & \textbf{44.27} & \underline{25.71} & \textbf{63.46} & 49.90 & 47.73 \\
\midrule
\multirow{13}{*}{PC}  &  \multirow{4}{*}{1024} &   CPC   & 46.37 & 38.36 & 24.19 & 52.54 & 28.30 & 38.64 \\
 &   &   R2C   & \textbf{55.78} & 40.10 & \underline{24.40} & 47.62 & 33.17 & 40.72 \\
  &     &   SpecPrefill (0.5B)  & \underline{49.04} & \underline{40.32} & \textbf{24.41} & \textbf{63.53} & \textbf{50.12} & \underline{45.16} \\
 &   &   \textbf{\specpc{} (0.5B)}   & 48.52 & \textbf{45.34} & 24.14 & \underline{61.04} & \underline{48.30} & \textbf{45.26} \\
\cmidrule{2-9}
 &  \multirow{4}{*}{2048}  &   CPC   & \underline{55.14} & 43.53 & 24.71 & 49.46 & 33.21 & 41.78 \\
 &   &   R2C   & 54.61 & \underline{45.70} & 24.99 & 53.65 & 42.43 & 44.41 \\
  &     &   SpecPrefill (0.5B)  & 53.39 & 43.46 & \textbf{25.53} & \textbf{65.89} & \underline{50.18} & \underline{47.51} \\
 &   &   \textbf{\specpc{} (0.5B)}   & \textbf{58.36} & \textbf{46.38} & \underline{25.32} & \underline{64.43} & \textbf{51.11} & \textbf{48.98} \\
\cmidrule{2-9}
& \multirow{4}{*}{3072} &   CPC   & 55.78 & 44.41 & 25.47 & 42.71 & 39.16 & 41.68 \\
 &   &   R2C   & \underline{59.22} & 44.70 & 25.69 & 55.22 & 44.04 & 45.90 \\
  &     &   SpecPrefill (0.5B)  & 55.75 & \textbf{45.71} & \textbf{25.93} & \textbf{64.22} & \textbf{51.63} & \underline{48.44} \\
 &   &   \textbf{\specpc{} (0.5B)}   & \textbf{59.84} & \underline{44.95} & \underline{25.91} & \underline{64.18} & \underline{46.88} & \textbf{48.46} \\
\bottomrule
\end{tabular}}
\label{app:tab:extra_longbench_qwen25_72b}
\end{table}

\clearpage
\begin{figure}
    \centering
     \begin{subfigure}[b]{0.49\textwidth}
         \centering
         \includegraphics[width=\textwidth]{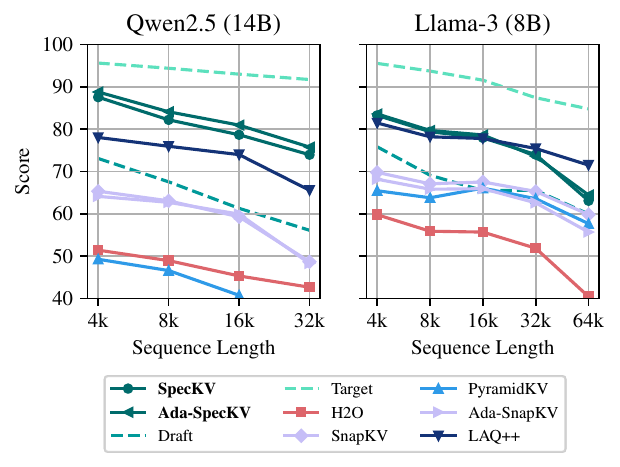}
         \caption{KV dropping ($\cmax = 256$)}
     \end{subfigure}
     \hfill
     \begin{subfigure}[b]{0.49\textwidth}
         \centering
         \includegraphics[width=\textwidth]{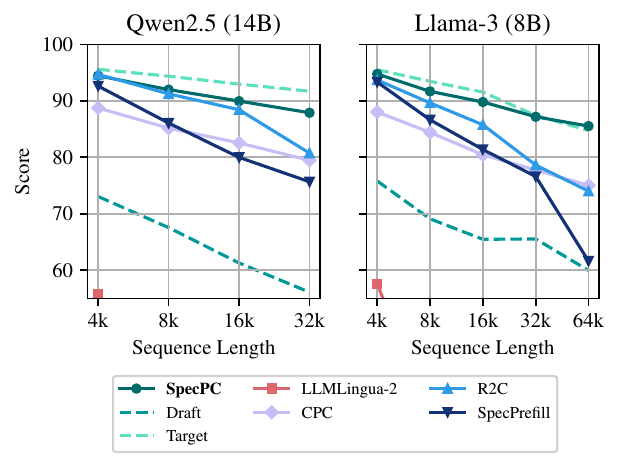}
         \caption{Prompt compression ($\cmax = 1024$)}
     \end{subfigure}
    \par\bigskip
      \begin{subfigure}[b]{0.49\textwidth}
         \centering
         \includegraphics[width=\textwidth]{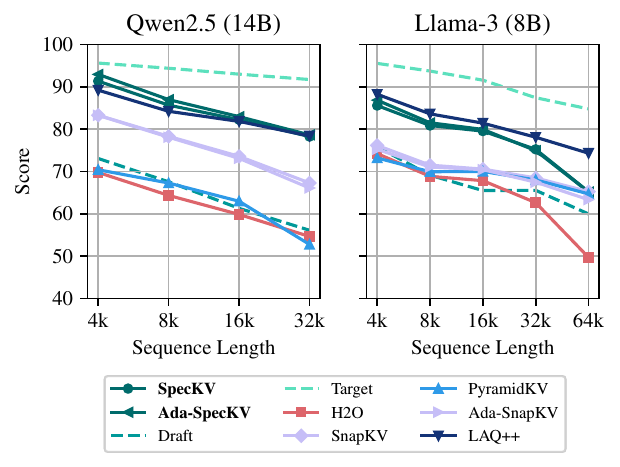}
         \caption{KV dropping ($\cmax = 512$)}
     \end{subfigure}
     \hfill
     \begin{subfigure}[b]{0.49\textwidth}
         \centering
         \includegraphics[width=\textwidth]{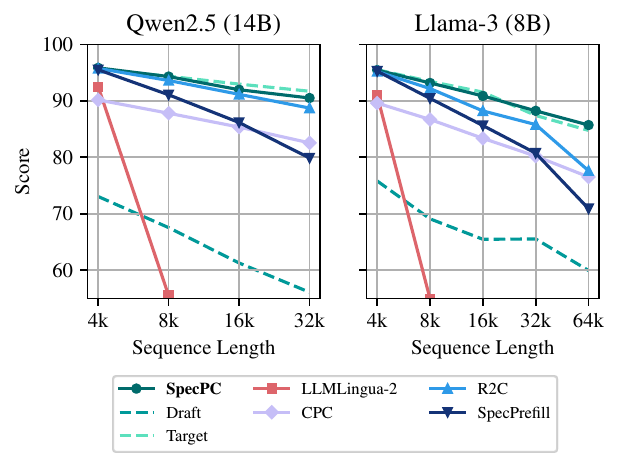}
         \caption{Prompt compression ($\cmax = 2048$)}
     \end{subfigure}
    \par\bigskip
      \begin{subfigure}[b]{0.49\textwidth}
         \centering
         \includegraphics[width=\textwidth]{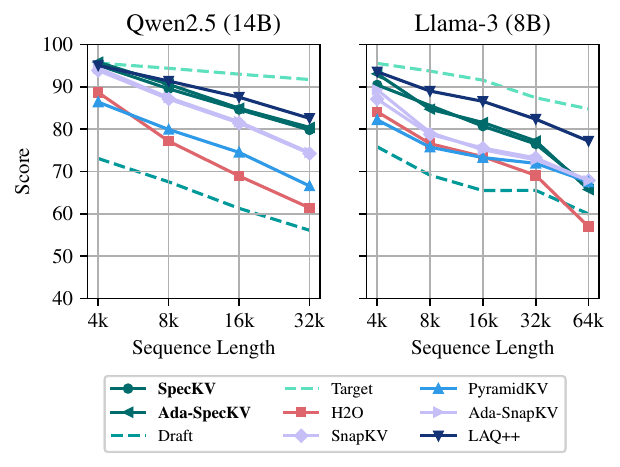}
         \caption{KV dropping ($\cmax = 1024$)}
     \end{subfigure}
     \hfill
     \begin{subfigure}[b]{0.49\textwidth}
         \centering
         \includegraphics[width=\textwidth]{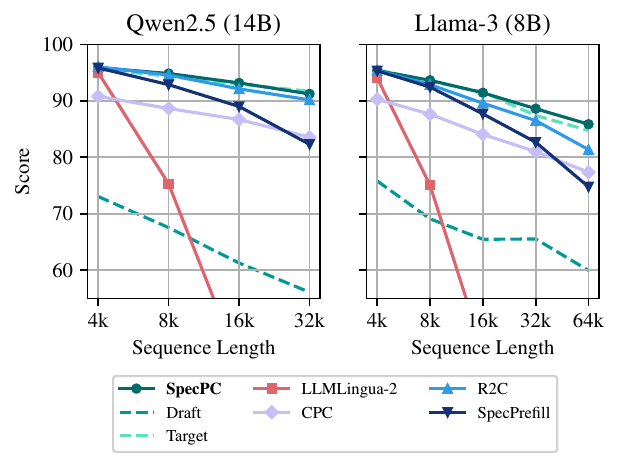}
         \caption{Prompt compression ($\cmax = 3072$)}
     \end{subfigure}
    \caption{Performance of the proposed \speckv{} and \specpc{} on RULER compared to various baselines including AdaKV for Qwen2.5-0.5B-Instruct and Llama-3.2-1B-Instruct as draft models, and Qwen2.5-14B-Instruct and Llama-3.1-8B-Instruct as target models. AdaKV can further boost the performance of \speckv{}. Notably, \specpc{} performs nearly on par with the target model. Low-performing methods are omitted.}
    \label{fig:ruler_llama_qwen_full}
\end{figure}
\clearpage
\vspace*{\fill}
\begin{table}[hp]
\centering
\caption{Results for LongBench performance with Qwen2.5-0.5B (draft) and Qwen2.5-14B (target).}
\resizebox{\linewidth}{!}{
\begin{tabular}{lllcccccc}
\toprule
  & $\cmax$ & \textbf{Method} & \makecell{\textbf{Single-}\\\textbf{doc QA}} & \makecell{\textbf{Multi-}\\\textbf{doc QA}} & \textbf{Summary} & 
\makecell{\textbf{Few-shot}\\\textbf{Learning}} & \makecell{\textbf{Code}\\\textbf{Completion}} & 
\textbf{All} \\
\midrule
\multirow{2}{*}{Dense} & \multirow{2}{*}{--} & Draft & 21.04 & 24.4 & 21.23 & 54.07 & 33.39 & 30.64 \\
&& Target & 53.19 & 42.83 & 24.99 & 65.34 & 51.75 & 47.33 \\
\midrule
\multirow{25}{*}{KV}  &  \multirow{8}{*}{256}  &   StreamingLLM   & 38.66 & 24.13 & 19.10 & 49.75 & 35.23 & 33.24 \\
  &    &   H2O   & 46.98 & 29.66 & 19.82 & 50.88 & 47.86 & 38.41 \\
  &    &   SnapKV   & 49.07 & 33.19 & 19.49 & 54.49 & 47.34 & 40.25 \\
  &    &   PyramidKV   & 47.07 & 31.84 & 18.32 & 54.81 & 43.26 & 38.76 \\
  &    &   Ada-SnapKV   & 50.92 & 34.57 & 19.63 & 55.07 & 49.58 & 41.41 \\
  &   &  LAQ++  & \textbf{52.45} & \textbf{39.99} & 22.12 & \textbf{61.46} & \underline{51.32} & \underline{45.05} \\
  &    &   \textbf{\speckv{}}   & 51.07 & 38.76 & \underline{23.20} & 59.68 & 49.58 & 44.09 \\
  &    &   \textbf{\adaspeckv{}}   & \underline{51.92} & \underline{39.38} & \textbf{24.90} & \underline{61.05} & \textbf{53.43} & \textbf{45.61} \\
\cmidrule{2-9}
 &  \multirow{8}{*}{512}  &   StreamingLLM   & 38.97 & 27.04 & 21.14 & 53.78 & 36.52 & 35.41 \\
  &    &   H2O   & 47.56 & 33.16 & 21.02 & 53.45 & 49.78 & 40.37 \\
  &    &   SnapKV   & 50.74 & 38.33 & 20.97 & 57.81 & 49.94 & 43.10 \\
  &    &   PyramidKV   & 50.19 & 36.37 & 20.25 & 58.17 & 47.88 & 42.19 \\
  &    &   Ada-SnapKV   & 51.69 & 38.61 & 21.48 & 59.84 & 51.85 & 44.18 \\
  &   &  LAQ++  & \textbf{52.31} & \underline{42.06} & 23.75 & 61.62 & 51.64 & 45.89 \\
  &    &   \textbf{\speckv{}}   & \underline{51.98} & 41.51 & \underline{24.11} & \underline{62.60} & \underline{52.33} & \underline{46.09} \\
  &    &   \textbf{\adaspeckv{}}   & 51.00 & \textbf{42.18} & \textbf{25.74} & \textbf{63.74} & \textbf{55.28} & \textbf{48.46} \\
\cmidrule{2-9}
&  \multirow{8}{*}{1024}  
&   StreamingLLM   & 40.08 & 32.01 & 22.12 & 55.42 & 38.36 & 37.54 \\
  &    &   H2O   & 49.73 & 37.27 & 22.38 & 55.94 & 51.26 & 42.75 \\
  &    &   SnapKV   & 51.56 & 40.98 & 22.65 & 63.90 & 51.48 & 45.73 \\
  &    &   PyramidKV   & 51.42 & 40.93 & 21.76 & 59.96 & 49.92 & 44.43 \\
  &    &   Ada-SnapKV   & 51.81 & 40.96 & 22.83 & \underline{64.14} & 51.97 & 45.94 \\
  &   &  LAQ++  & \underline{52.27} & \underline{42.59} & \underline{25.04} & 61.56 & 51.68 & 46.27 \\
  &    &   \textbf{\speckv{}}   & \textbf{52.92} & 41.85 & 24.30 & 63.46 & \underline{52.50} & \underline{46.61} \\
  &    &   \textbf{\adaspeckv{}}   & \underline{52.27} & \textbf{43.22} & \textbf{26.39} & \textbf{64.83} & \textbf{54.40} & \textbf{47.78} \\
\midrule
\multirow{16}{*}{PC}  &  \multirow{5}{*}{1024}   &  LLMLingua-2  & 27.10 & 23.74 & 22.57 & 35.85 & 37.62 & 28.79 \\
  &   &  CPC  & 42.76 & 36.05 & 22.58 & 48.38 & 35.57 & 37.17 \\
  &   &  R2C  & \underline{46.43} & \underline{37.42} & 22.64 & 47.21 & 29.48 & 37.15 \\
  &   &  SpecPrefill  & 46.21 & 37.21 & \underline{22.82} & \textbf{62.18} & \underline{48.34} & \underline{43.00} \\
  &   &  \textbf{\specpc{}}  & \textbf{47.70} & \textbf{38.30} & \textbf{23.30} & \underline{59.74} & \textbf{52.52} & \textbf{43.73} \\
  \cmidrule{2-9}
 &  \multirow{5}{*}{2048}  &  LLMLingua-2  & 35.10 & 32.22 & 23.51 & 39.82 & 44.80 & 34.40 \\
  &   &  CPC  & 48.74 & 39.70 & 23.39 & 55.77 & 42.07 & 41.92 \\
  &   &  R2C  & \underline{50.68} & \underline{40.65} & 23.47 & 56.72 & 45.62 & 43.27 \\
  &   &  SpecPrefill  & 49.75 & 38.58 & \textbf{23.70} & \textbf{65.05} & \underline{50.40} & \underline{45.15} \\
  &   &  \textbf{\specpc{}}  & \textbf{53.23} & \textbf{41.43} & \underline{23.68} & \underline{63.26} & \textbf{54.92} & \textbf{46.76} \\
\cmidrule{2-9}
& \multirow{5}{*}{3072}
 &  LLMLingua-2  & 43.50 & 35.67 & 24.10 & 45.66 & 47.32 & 38.67 \\
  &   &  CPC  & 51.59 & 41.08 & 23.84 & 58.14 & 44.82 & 43.83 \\
  &   &  R2C  & 51.71 & \underline{41.39} & 23.89 & 62.09 & 48.58 & 45.31 \\
  &   &  SpecPrefill  & \underline{51.81} & 41.27 & \underline{24.26} & \underline{64.28} & \underline{50.66} & \underline{46.15} \\
  &   &  \textbf{\specpc{}}  & \textbf{53.25} & \textbf{41.48} & \textbf{24.49} & \textbf{64.51} & \textbf{54.41} & \textbf{47.14} \\
\bottomrule
\end{tabular}}
\label{tab:qwen_longbench_full}
\end{table}
\vspace*{\fill}
\clearpage
\vspace*{\fill}
\begin{table}[hp]
\centering
\caption{Results for LongBench performance with Llama-3.2-1B (draft) and Llama-3.1-8B (target).}
\resizebox{\linewidth}{!}{
\begin{tabular}{lllcccccc}
\toprule
  & $\cmax$ & \textbf{Method} & \makecell{\textbf{Single-}\\\textbf{doc QA}} & \makecell{\textbf{Multi-}\\\textbf{doc QA}} & \textbf{Summary} & 
\makecell{\textbf{Few-shot}\\\textbf{Learning}} & \makecell{\textbf{Code}\\\textbf{Completion}} & 
\textbf{All} \\
\midrule
\multirow{2}{*}{Dense} & \multirow{2}{*}{--} & Draft & 28.08 & 27.27 & 25.65 & 60.16 & 31.11 & 34.69 \\
& & Target & 45.85 & 43.79 & 28.68 & 66.65 & 50.46 & 46.84 \\
\midrule
\multirow{25}{*}{KV} 
& \multirow{8}{*}{256}
&   StreamingLLM   & 38.69 & 27.12 & 21.64 & 50.75 & 34.74 & 34.58 \\
  &    &   H2O   & 43.54 & 36.81 & 22.62 & 55.64 & 47.81 & 40.82 \\
  &    &   SnapKV   & 43.79 & 37.31 & 21.96 & 56.29 & 47.34 & 40.91 \\
  &    &   PyramidKV   & 43.62 & 37.79 & 21.75 & 55.15 & 46.32 & 40.54 \\
  &    &   Ada-SnapKV   & \underline{44.04} & 37.86 & 22.34 & 59.95 & 50.39 & 42.38 \\
 &   &  LAQ++  & \textbf{45.81} & \textbf{41.72} & 22.93 & \textbf{64.18} & 47.21 & \textbf{44.17} \\
  &    &   \textbf{\speckv{}}   & 43.23 & 39.73 & \underline{24.43} & \underline{60.90} & \underline{51.09} & \underline{43.36} \\
  &    &   \textbf{\adaspeckv{}}   & 42.40 & \underline{40.73} & \textbf{25.74} & 57.94 & \textbf{52.51} & 43.25 \\
\cmidrule{2-9}
& \multirow{8}{*}{512} 
&   StreamingLLM   & 38.80 & 29.15 & 24.16 & 52.59 & 37.25 & 36.33 \\
  &    &   H2O   & 44.82 & 39.36 & 23.95 & 58.70 & 49.30 & 42.79 \\
  &    &   SnapKV   & 45.00 & 41.31 & 23.61 & 61.36 & 49.92 & 43.83 \\
  &    &   PyramidKV   & \underline{45.26} & 41.22 & 23.36 & 61.15 & 48.06 & 43.51 \\
  &    &   Ada-SnapKV   & 45.09 & 41.17 & 24.19 & 61.81 & \underline{51.74} & 44.30 \\
  &   &  LAQ++  & \textbf{46.17} & \textbf{43.44} & 24.17 & \textbf{63.13} & 48.74 & \textbf{44.87} \\
  &    &   \textbf{\speckv{}}   & 43.48 & 41.86 & \underline{26.17} & \underline{62.34} & 51.36 & \underline{44.59} \\
  &    &   \textbf{\adaspeckv{}}   & 43.80 & \underline{42.56} & \textbf{26.59} & 58.72 & \textbf{54.41} & 44.56 \\
\cmidrule{2-9}
& \multirow{8}{*}{1024} 
&   StreamingLLM   & 38.45 & 32.54 & 25.03 & 58.09 & 38.60 & 38.54 \\
  &    &   H2O   & 45.45 & 42.50 & 25.42 & 59.17 & 50.10 & 44.13 \\
  &    &   SnapKV   & 45.46 & 43.15 & 25.42 & 62.06 & \underline{52.37} & 45.22 \\
  &    &   PyramidKV   & \underline{46.10} & 42.78 & 25.17 & 63.43 & 49.55 & 45.11 \\
  &    &   Ada-SnapKV   & 46.06 & 43.34 & 25.48 & \underline{63.47} & 50.50 & 45.43 \\
  &   &  LAQ++  & \textbf{46.21} & \textbf{44.14} & 25.71 & \textbf{64.10} & 50.00 & \textbf{45.75} \\
  &    &   \textbf{\speckv{}}   & 43.73 & 43.39 & \underline{26.95} & 63.14 & 51.28 & 45.30 \\
  &    &   \textbf{\adaspeckv{}}   & 44.83 & \underline{43.72} & \textbf{27.50} & 61.04 & \textbf{54.23} & \underline{45.70} \\
\midrule
\multirow{16}{*}{PC} & \multirow{5}{*}{1024}
&  LLMLingua-2  & 29.61 & 24.83 & 23.43 & 24.18 & 40.66 & 27.68 \\
  &   &  CPC  & 35.67 & 36.61 & 25.26 & 34.01 & 43.58 & 34.42 \\
  &   &  R2C  & 38.41 & 39.07 & 25.28 & 43.26 & 43.99 & 37.58 \\
  &   &  SpecPrefill  & \underline{41.51} & \underline{39.38} & \underline{25.82} & \underline{63.51} & \underline{44.68} & \underline{42.86} \\
  &   &  \textbf{\specpc{}}  & \textbf{44.83} & \textbf{39.94} & \textbf{25.85} & \textbf{63.70} & \textbf{44.82} & \textbf{43.76} \\
\cmidrule{2-9}
& \multirow{5}{*}{2048}
&  LLMLingua-2  & 34.00 & 32.51 & 24.90 & 24.76 & \textbf{47.27} & 31.64 \\
  &   &  CPC  & 40.02 & 39.41 & 26.83 & 39.02 & 46.66 & 37.80 \\
  &   &  R2C  & \underline{44.53} & 38.97 & 26.63 & 54.62 & 46.67 & 41.97 \\
  &   &  SpecPrefill  & 42.35 & \textbf{42.11} & \underline{27.23} & \underline{63.56} & 44.46 & \underline{43.91} \\
  &   &  \textbf{\specpc{}}  & \textbf{44.92} & \underline{40.71} & \textbf{27.30} & \textbf{64.77} & \underline{46.89} & \textbf{44.78} \\
\cmidrule{2-9}
& \multirow{5}{*}{3072}
&  LLMLingua-2  & 39.13 & 35.44 & 25.98 & 29.73 & \textbf{49.92} & 35.05 \\
  &   &  CPC  & 41.73 & 39.52 & 27.27 & 42.12 & \underline{49.13} & 39.30 \\
  &   &  R2C  & \underline{44.77} & 40.97 & 27.35 & 60.85 & 48.08 & 44.14 \\
  &   &  SpecPrefill  & 43.86 & \textbf{42.45} & \underline{27.82} & \underline{63.17} & 45.25 & \underline{44.46} \\
  &   &  \textbf{\specpc{}}  & \textbf{47.12} & \underline{41.95} & \textbf{28.02} & \textbf{65.61} & 45.52 & \textbf{45.65} \\
\bottomrule
\end{tabular}}
\label{tab:llama_longbench_full}
\end{table}
\vspace*{\fill}
\clearpage
\begin{figure}
    \centering
     \begin{subfigure}[b]{1\textwidth}
         \centering
     \includegraphics[width=\textwidth]
     {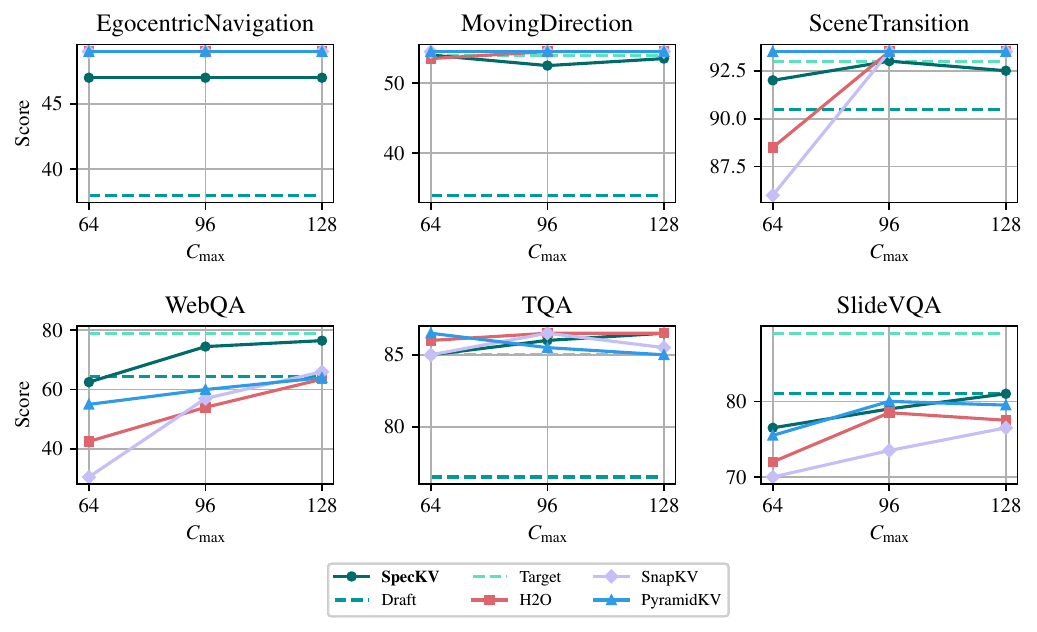}
     \caption{KV dropping}
    \label{app:fig:milebench_speckv}
     \end{subfigure}
     \par\bigskip
     \begin{subfigure}[b]{1\textwidth}
      \includegraphics[width=\textwidth]{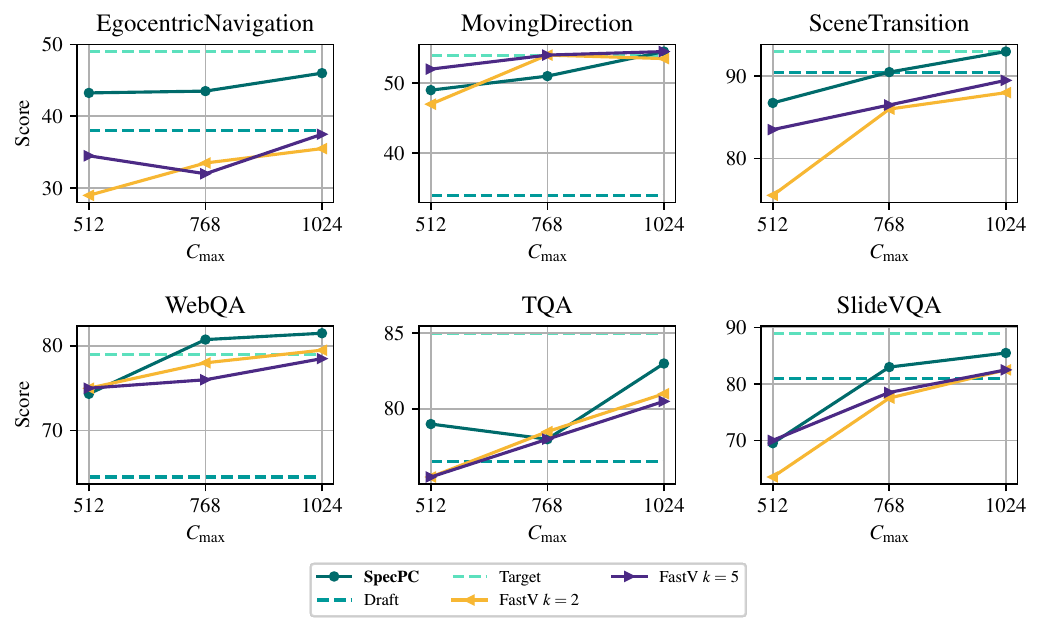}
     \caption{Prompt compression}
    \label{app:fig:milebench_specpc}
     \end{subfigure}
    \caption{MileBench multi-modal results using Qwen2.5-VL-3B-Instruct-AWQ (draft) and Qwen2.5-VL-32B-Instruct-AWQ (target). \speckv{} demonstrates competitive performance across most tasks and achieves a substantial improvement on WebQA. \specpc{} consistently outperforms both FastV configurations on the majority of datasets.}
\end{figure}

\begin{figure}[h]
\centering
\includegraphics[width=\textwidth]{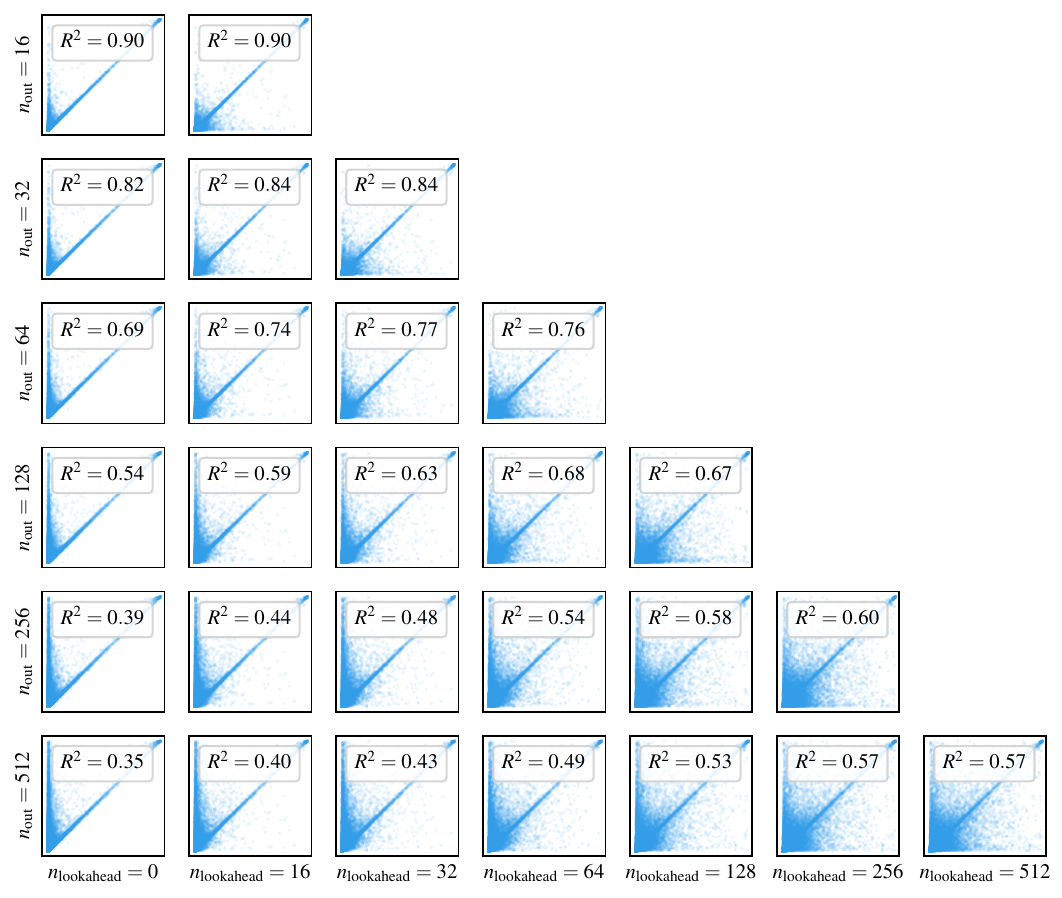}
\caption{\revised{Correlation of importance scores from \speckv{} against ground-truth scores from the dense target model for various output lengths ($\nout$) and lookaheads ($\nlookahead$). Increasing $\nlookahead$ improves correlation, especially for longer output lengths ($\nout$). SnapKV ($\nlookahead=0$) consistently shows the lowest correlation. Experiments use a Qwen2.5-0.5B draft model and a Qwen2.5-32B target model on Multi-news from LongBench.}}
\label{fig:speckv_imp_corr_multi_news}
\end{figure}
\begin{figure}[h]
\centering
\includegraphics[width=\textwidth]{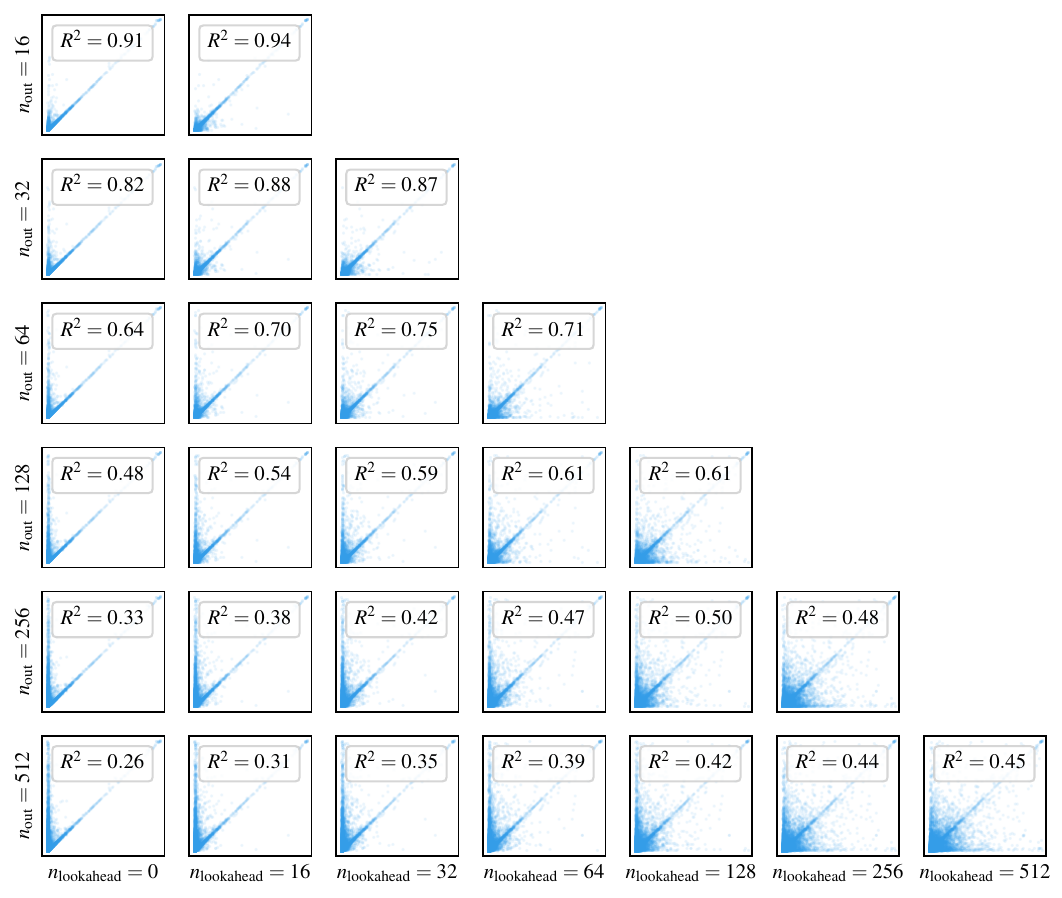}
\caption{\revised{Correlation of importance scores from \speckv{} against ground-truth scores from the dense target model for various output lengths ($\nout$) and lookaheads ($\nlookahead$). Increasing $\nlookahead$ improves correlation, especially for longer output lengths ($\nout$). SnapKV ($\nlookahead=0$) shows the lowest correlation. Experiments use a Qwen2.5-0.5B draft model and a Qwen2.5-32B target model on GovReport from LongBench.}}
\label{fig:speckv_imp_corr_gov_report}
\end{figure}

\begin{figure}[ht]
    \centering
    
    \includegraphics[
        width=\textwidth
    ]{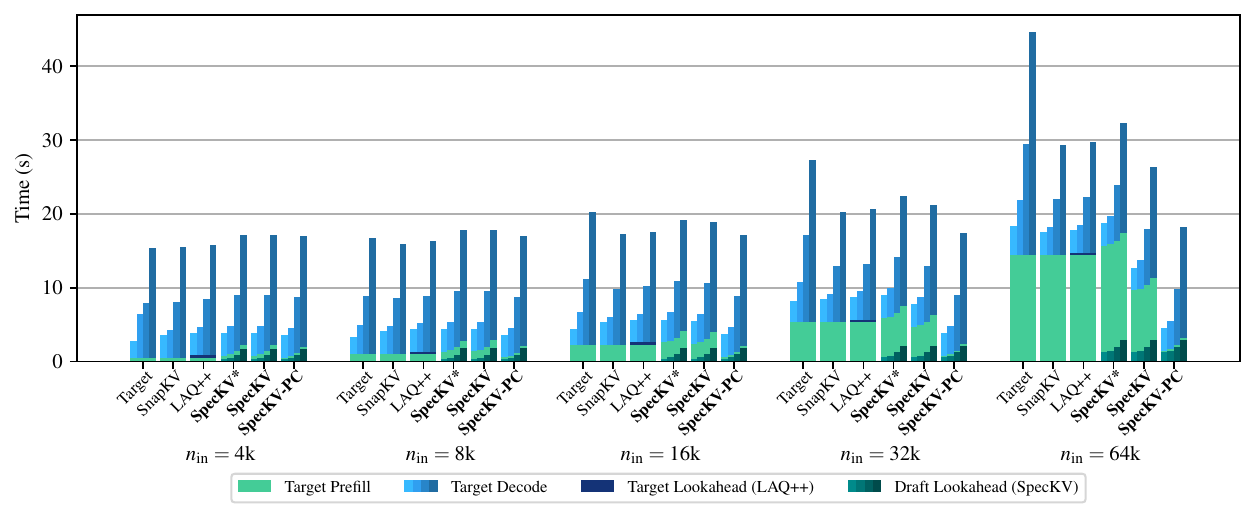} 
    
    \caption{\revised{
End-to-end latency comparison across five input sequence lengths ($\ninput$) and four output sequence lengths ($\noutput \in \{64, 128, 256, 512\}$). 
All experiments use the Qwen2.5 model (3B draft, 32B target) with $C_{\text{max}}=256$. For SpecKV algorithms, we use $\nlookahead = \nout$. 
Each bar segment breaks down the latency by processing stage, and the four bars for each method correspond to the four $\noutput$ values.
SpecKV* denotes SpecKV without sparse prefill, while SpecKV-PC combines SpecKV with SpecPC.
Notably, when $\ninput \ge 16\text{k}$, the efficiency gains from sparse prefill (SpecKV) and prompt compression (SpecKV-PC) effectively offset the lookahead overhead for $\nlookahead \le 64$ and $\le 512$, respectively.
Overall, SpecKV-PC is the most efficient method; by precompressing the prompt, it significantly cuts target prefill time, resulting in a speedup of about 40\% to 75\% (depending on output length) over LAQ++ at a 64k input context.
}}
    \label{fig:latency_breakdown_plot}
\end{figure}

\begin{figure}[hp]
    \centering
    
    \includegraphics[
        width=\textwidth
    ]{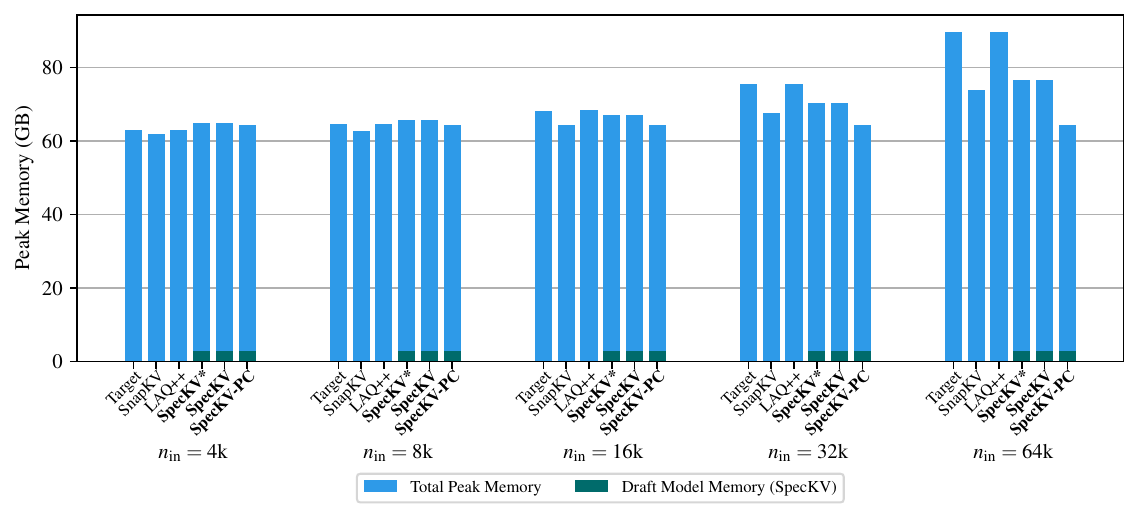} 
    
    \caption{\revised{Peak memory usage across five input sequence lengths ($\ninput$) using the Qwen2.5 model (3B draft, 32B target) with $\cmax=256$ and $\nlookahead = \nout$. 
    SpecKV* denotes SpecKV without sparse prefill, while SpecKV-PC combines SpecKV with SpecPC.
    LAQ++ offers no peak memory savings over the target model, as it must store the full cache. 
    SpecKV is slightly more memory-hungry than SnapKV because it also stores the draft model weights (indicated in green), though this constitutes a small fraction of the total memory. 
    SpecKV-PC is the most memory-efficient method; it first compresses the prompt (to 2048 tokens here) before the target model's prefill phase, significantly reducing peak usage. 
    At a 64k context length, SpecKV-PC saves approximately 9 GB over SnapKV and 25 GB over LAQ++.}}
    \label{fig:memory_breakdown_plot}
\end{figure}


\begin{figure}[hp]
\centering
\includegraphics[width=0.5\textwidth]{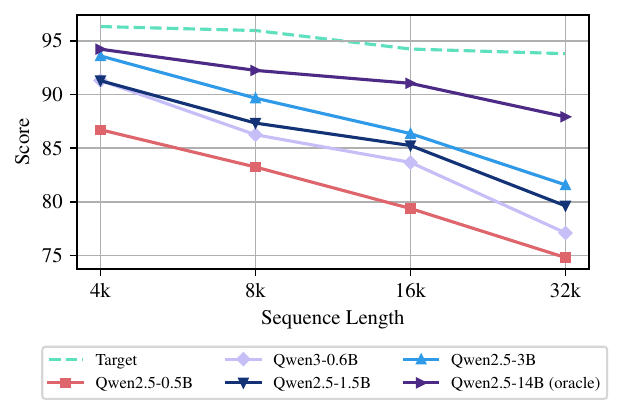}
\caption{
Effect of draft model quality on RULER score. The target model is Qwen2.5-14B (Instruct). Consistent with expectations, employing more capable draft models boosts performance. For reference, we also evaluate an oracle setting where the draft model is identical to the target (Qwen2.5-14B), representing an empirical upper bound for \speckv{}.
}
\label{app:fig:better_draft}
\end{figure}
\begin{figure}[hp]
\centering

\begin{subfigure}[b]{0.455\textwidth}
\centering
\includegraphics[width=\textwidth]{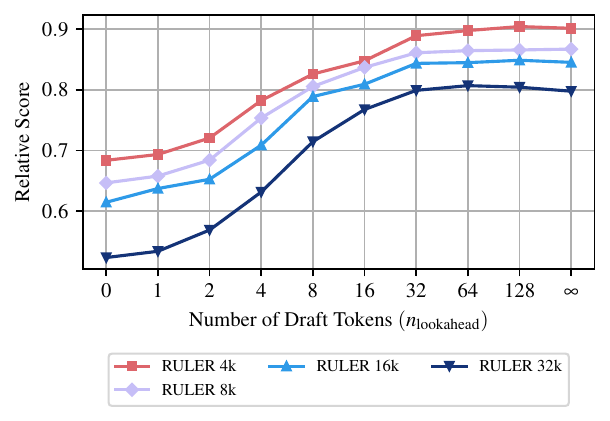}
\caption{\speckv{} on RULER}
\end{subfigure}
\hfill
\begin{subfigure}[b]{0.49\textwidth}
\centering
\includegraphics[width=\textwidth]{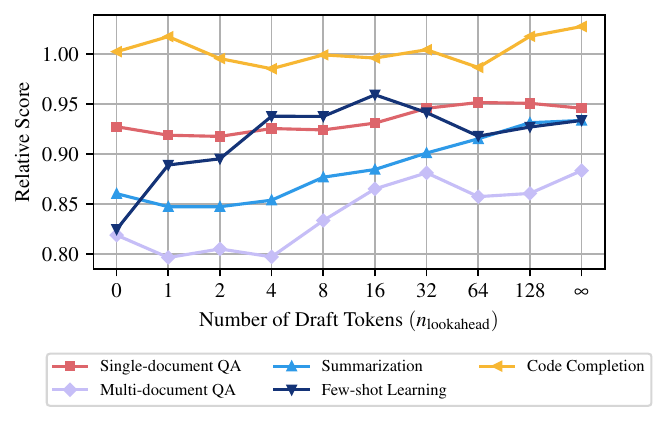}
\caption{\speckv{} on LongBench}
\end{subfigure}
\par\bigskip
\begin{subfigure}[b]{0.455\textwidth}
\centering
\includegraphics[width=\textwidth]{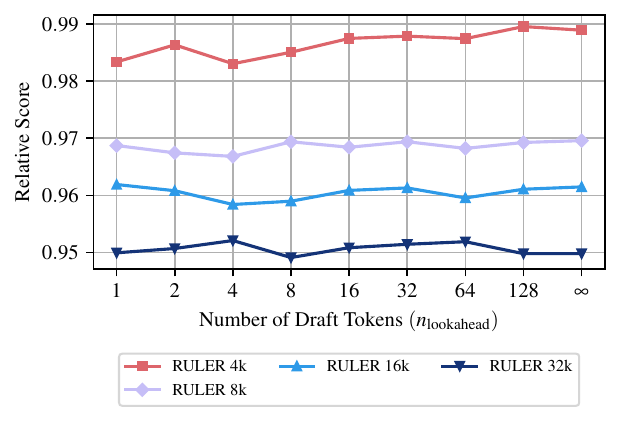}
\caption{\specpc{} on RULER}
\end{subfigure}
\hfill
\begin{subfigure}[b]{0.49\textwidth}
\centering
\includegraphics[width=\textwidth]{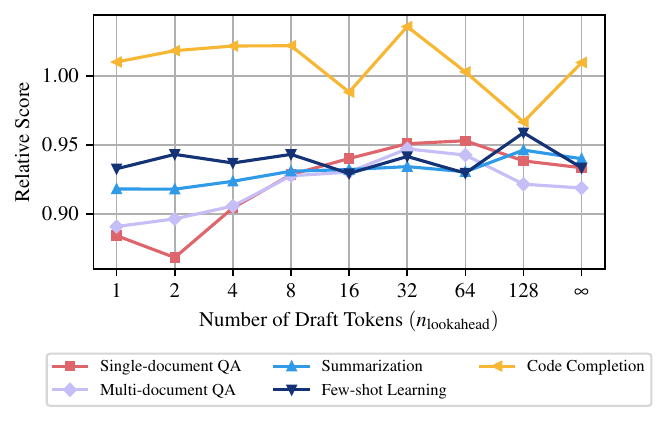}
\caption{\specpc{} on LongBench}
\end{subfigure}

\caption{\revised{
Impact of the number of generated draft tokens, $\nlookahead$, on the relative performance of \speckv{} and \specpc{}, using Qwen2.5-0.5B-Instruct as the draft model and Qwen2.5-14B-Instruct as the target model. The relative score is calculated as the score of each \speckv{} configuration divided by the score of the full dense target model. Increasing $\nlookahead$ substantially boosts \speckv{}'s score, whereas \specpc{} shows only minor improvement with higher $\nlookahead$. SnapKV ($\nlookahead=0$) has the lowest performance in most cases. $\infty$ denotes lookahead to the EOS token.
}}
\label{app:fig:nlookahead}
\end{figure}

\clearpage

\vspace*{\fill}
\begin{figure}[hp]
\centering
\begin{subfigure}[b]{0.455\textwidth}
\centering
\includegraphics[width=\textwidth]{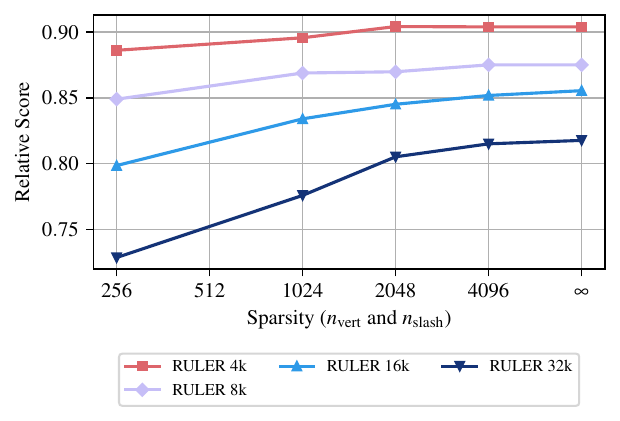}
\caption{RULER}
\end{subfigure}
\hfill
\begin{subfigure}[b]{0.49\textwidth}
\centering
\includegraphics[width=\textwidth]{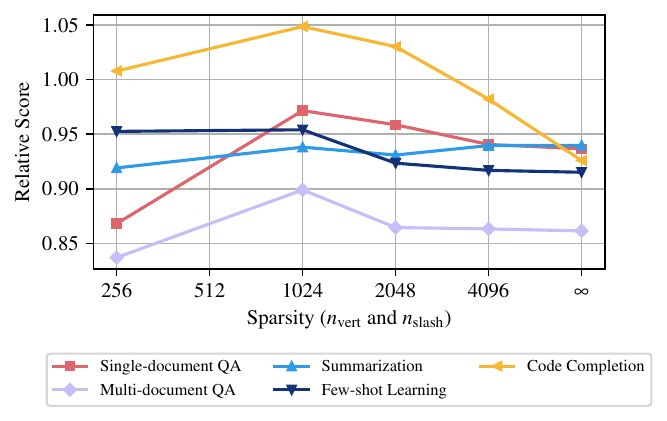}
\caption{LongBench}
\end{subfigure}
\caption{Impact of different sparsity levels in \speckv{}'s sparse prefill on relative performance, using Qwen2.5-0.5B-Instruct as the draft model and Qwen2.5-14B-Instruct as the target model. We vary $\nglobal$ (set equal to $\nprewnd$) and observe the impact on accuracy. Accuracy improves as sparsity decreases (higher $\nglobal$) up to 2048, beyond which gains saturate, hence our choice of 2048 for main results. $\infty$ corresponds to fully dense prefill. Notably, for some LongBench tasks, higher sparsity actually benefits accuracy. The relative score is calculated as the score of each \speckv{} configuration divided by the score of the full dense target model.}
\label{app:fig:nvert}
\end{figure}

\vspace*{\fill}

\end{document}